\documentclass[twoside]{article}

\usepackage[preprint]{ctrl}
\DeclareMathOperator{\Span}{span}
\DeclareMathOperator{\rng}{rng}

\usepackage{amsmath,amssymb}
\usepackage{amsthm}

\newtheorem{theorem}{Theorem}[section]
\newtheorem{lemma}[theorem]{Lemma}

\newtheorem{assumption}[theorem]{Assumption}
\newtheorem{proposition}[theorem]{Proposition}
\theoremstyle{definition}
\newtheorem{definition}[theorem]{Definition}
\theoremstyle{remark}
\newtheorem{remark}[theorem]{Remark}

\usepackage[round]{natbib}

\usepackage{hyperref}

\usepackage{booktabs}
\usepackage{graphicx}
\usepackage{algorithm}
\usepackage{algpseudocode}
\usepackage{array}

\usepackage{xcolor}
\definecolor{sattarcolor}{HTML}{A65E2E}

\usepackage{soul}

\begin{document}



\twocolumn[

\aistatstitle{Linear Bandits under Exact Sliding-Window Constraints}

% \aistatsauthor{ Author 1 \And Author 2 \And  Author 3 }

% \aistatsaddress{ Institution 1 \And  Institution 2 \And Institution 3 } ]

\aistatsauthor{
Seyed Mohammad Hadi Hosseini
\And
Yasin Abbasi-Yadkori
\And
Sattar Vakili
}

\aistatsaddress{
\texttt{hadi.hosseini0171@gmail.com}
\And
\texttt{yasin.abbasi@gmail.com}
\And
\texttt{s.vakili@ucl.ac.uk}
}

]

\begin{abstract}

We study linear bandits under \emph{exact sliding-window constraints}, where every consecutive block of actions must belong to a prescribed feasible set. In the offline setting, where the reward function is known, we show that convexity and cyclic-shift invariance make a stationary solution optimal when \(w\mid T\) and within an additive \(\mathcal{O}(w)\) gap otherwise. In the online setting, 
we show that geometric structure alone is insufficient for learning, and sublinear regret can be impossible. We introduce a \emph{transition diameter} \(\tau\) that quantifies feasible reachability and develop a rare-switching OFUL algorithm with regret \(\widetilde{\mathcal O}(d\sqrt{T}+\tau d+w)\) against the offline-optimal feasible trajectory. Finally, we remove cyclic invariance and consider general sliding-window constraints, where optimal behavior may be non-stationary. We represent recent action history as the state of a finite-memory control problem and introduce a \emph{history-state diameter} \(D\) that measures feasible communication between viable histories. Combining optimistic remaining-horizon planning with rare policy updates, we obtain a regret bound of \(\widetilde{\mathcal O}(d\sqrt{T}+dD+w)\). We evaluate our approach on real-world and synthetic benchmarks, showing that it maintains exact feasibility while achieving reward and regret comparable to baselines with substantially fewer policy updates.

\end{abstract}

\section{Introduction}
\label{sec:introduction}

Many sequential decision-making problems impose constraints not only on the current action, but also on how actions evolve over time. This occurs, for example, in online resource allocation~\citep{zhang2026online},
energy systems with ramping and rolling-horizon operation~\citep{mathias2024model},
constrained adaptive control~\citep{chattopadhyay2025model},
and safety-critical control~\citep{ames2017control,capone2025learning}. In these settings, a decision that looks good now may restrict future choices because its consequences persist over time. We model this dependence through \emph{sliding-window constraints}. At each round \(t\), the learner chooses an action \(x_t\), and every consecutive block of \(w\) actions must satisfy
\[
(x_{t-w+1},\ldots,x_t)\in\mathcal C,
\qquad
t=w,\ldots,T,
\]
for a prescribed feasible set \(\mathcal C\).

Unlike pointwise constraints~\citep{amani2019linear,gangrade2024safe}, which enforce feasibility separately at each round, and long-term constraints~\citep{mahdavi2012trading,jenatton2016adaptive,yi2021regret,bernasconi2025noregret}, which control feasibility only in aggregate over time, our setting requires every realized sliding window to satisfy the constraint exactly. Thus, exploration itself must respect the temporal constraint. The problem can be viewed as a finite-memory control problem in which the most recent \(w-1\) actions form the state, connecting our formulation to constrained MDPs~\citep{altman1999constrained,stradi2024online}. Unlike a generic constrained MDP, where rewards and transition dynamics may depend on the state, in our online linear-bandit setting the history state determines which actions are feasible next, while rewards are governed by a single \(d\)-dimensional linear parameter shared across all histories. A detailed comparison with linear, tabular, and constrained reinforcement learning is provided in Appendix~\ref{app:relation_rl}.

Our analysis is organized around three structural ideas: \emph{geometry}, which shapes the structure of feasible solutions; \emph{reachability}, which determines whether desirable feasible behaviors can be connected; and \emph{planning}, which becomes necessary when optimal behavior is non-stationary.

We first consider the offline problem, where the reward function is known. When \(\mathcal C\) is convex and invariant under cyclic shifts, cyclic symmetrization shows that stationary behavior is exactly optimal when \(w\mid T\). For arbitrary horizons, the best stationary solution is within an additive \(\mathcal{O}(w)\) 
% \sattar{Should be $\mathcal{O}$, please check for notation consistency throughout,}
of the offline optimum (Theorem~\ref{thm:stationary_near_optimality}).
% \sattar{For presenting the results in general it is better to have a subjective tone, e.g., When \(\mathcal C\) is convex and invariant under cyclic shifts, we show that stationary behavior is exactly optimal when \(w\mid T\) (See Lemma/Theorem~xx). }

In the online setting,
where the reward function is unknown, the geometric structure alone is not enough: even if good stationary feasible actions exist, moving between them may be impossible without violating the constraint, making sublinear regret impossible in general (Proposition~\ref{prop:linear-regret-without-transitions}). We thus introduce a \emph{transition diameter} \(\tau\), which quantifies feasible reachability between stationary histories, and develop a rare-switching algorithm based on OFUL~\citep{abbasi2011improved} that preserves exact sliding-window feasibility and achieves expected pseudo-regret
\(\widetilde{\mathcal O}(d\sqrt{T}+\tau d+w)\)
(Theorem~\ref{thm:expected-offline-regret}). 
% \sattar{This is the reference for OFUL, not for rare switching one.}

Finally, we move beyond cyclic invariance and consider general sliding-window constraints, where optimal behavior may be non-stationary. We use the most recent \(w-1\) actions as the state of a finite-memory control problem and introduce a \emph{history-state diameter} \(D\), which measures feasible communication between viable histories. We then combine optimistic remaining-horizon planning with rare policy updates to obtain expected pseudo-regret \(\widetilde{\mathcal O}(d\sqrt{T}+dD+w)\)
(Theorem~\ref{thm:noncyclic_regret}).

Furthermore, we study two real-world sequential decision-making applications: online recommendation with diversity constraints and online LLM routing with cost constraints. In the recommendation setting, the learner aims to maximize user reward while preventing overly repetitive recommendations of the same categories through sliding-window diversity constraints. In the LLM routing setting, the learner selects among language models with different response qualities and inference costs while satisfying a sliding-window resource budget. We complement these experiments with controlled synthetic studies that illustrate the roles of reachability and planning identified by our theoretical analysis.

Our contributions are threefold.
\textbf{(i)} For convex, cyclically invariant sliding-window constraints, we show that stationary behavior is optimal when \(w \mid T\) and within \(\mathcal{O}(w)\) of the offline optimum otherwise. We further introduce the transition diameter \(\tau\) and a feasibility-preserving rare-switching OFUL algorithm 
with expected pseudo-regret \(\widetilde{\mathcal O}(d\sqrt{T}+\tau d+w)\).
\textbf{(ii)} For general sliding-window constraints, we formulate the problem as finite-memory control, introduce the history-state diameter \(D\), and obtain expected pseudo-regret \(\widetilde{\mathcal O}(d\sqrt{T}+dD+w)\) via optimistic remaining-horizon planning with rare updates.
\textbf{(iii)} Experiments on real-world recommendation, LLM routing and synthetic tasks show that our approach maintains exact feasibility, reduces regret in non-cyclic settings, and requires substantially fewer policy updates.
\section{Related Work}
\label{sec:related_works}

\begin{table*}[t]
\centering
\caption{Comparison with related sequential-learning frameworks.}
\label{tab:related_comparison}
\small
\setlength{\tabcolsep}{8pt}
\renewcommand{\arraystretch}{1.15}
\begin{tabular}{lcccc}
\toprule
Setting
& \shortstack{Action-wise\\constraints}
& \shortstack{Horizon-wide\\constraints}
& \shortstack{History\\dependence}
& \shortstack{Exact-window\\constraints}
\\
\midrule
Safe/constrained bandits
& \checkmark & & & \\
Long-term constrained learning
& & \checkmark & & \\

Online learning with memory
& & & \checkmark & \\

Switching-constrained bandits
& & \checkmark & \checkmark & \\

\textbf{This work}
& \checkmark$^{\dagger}$
& \checkmark$^{\ddagger}$
& \checkmark
& \checkmark
\\
\bottomrule
\end{tabular}

\vspace{1mm}

\begin{minipage}{0.92\textwidth}
\footnotesize
\(\dagger\) Pointwise feasibility is recovered when \(w=1\).
\(\ddagger\) Horizon-wide constraints can be represented when \(w=T\).
\end{minipage}
\end{table*}

\paragraph{Constrained online learning.}
Online decision-making under feasibility and resource constraints has been studied extensively. Long-term constrained methods~\citep{mannor2009online,mahdavi2012trading,jenatton2016adaptive,yi2021regret,bernasconi2025noregret} typically allow per-round violations while controlling cumulative constraint violation. Bandits with knapsacks and related models~\citep{badanidiyuru2018bandits,agrawal2019bandits,castiglioni2024online} impose global budget or convex constraints over the horizon, while safe and constrained bandits~\citep{amani2019linear,gangrade2024safe} study pointwise feasibility at each round. Exact sliding-window constraints include pointwise feasibility as a special case when \(w=1\), while setting \(w=T\) allows horizon-wide constraints to be represented within the same framework.

\paragraph{Temporal constraints and learning with memory.}
Several related models~\citep{liakopoulos2019cautious,agrawal2018bandits,anava2015online,abdullah2026constrained} consider temporal structure induced by recent decisions. Cautious regret~\citep{liakopoulos2019cautious} imposes window-based resource constraints on the benchmark while allowing residual constraint violations, whereas temporal stochastic bandits~\citep{agrawal2018bandits} allow rewards to depend on recent action history. Online learning with memory~\citep{anava2015online,hebbar2024online} similarly models losses that depend on past decisions. More recently, online convex optimization with finite-memory losses and constraints~\citep{abdullah2026constrained} considers cumulative constraint violation rather than requiring every sliding window to satisfy the constraint. In contrast, our model couples decisions through exact sliding-window feasibility: every window must remain feasible while maximizing reward, so exploration must respect these temporal constraints throughout learning.

\paragraph{Constrained control and finite-memory planning.}
Constrained Markov decision processes~\citep{altman1999constrained,stradi2024online} provide a general framework for sequential decision-making under state and resource constraints, while safe and learning-based control~\citep{ames2017control,capone2025learning} seeks to maintain safety throughout interaction. Our setting is specified directly through overlapping constraints on recent actions. Without cyclic symmetry, the most recent \(w-1\) actions naturally form a finite-memory state, yielding a deterministic planning problem with state-dependent feasible actions. Unlike a generic CMDP, the history dynamics are deterministic and known, while reward uncertainty is shared across history states through a single low-dimensional linear parameter.

\paragraph{Limited switching and switching costs.}
Our online algorithms build on OFUL~\citep{abbasi2011improved} and relate to work on limited adaptivity and switching. Limited-switching methods~\citep{altschuler2018online} constrain the number of action changes. Switching-cost models~\citep{koren2017bandits,rouyer2021algorithm} instead penalize changes between actions. In contrast, switching in our setting is constrained by feasibility itself: unlike standard OFUL, where the optimistic action can be selected directly, an immediate change may be infeasible under the sliding-window constraint, requiring a sequence of intermediate feasible actions.

Table~\ref{tab:related_comparison} compares these settings along several key dimensions.
\section{Problem Formulation}
\label{sec:problem_formulation}

Let \(\mathcal X\subseteq\mathbb R^d\) denote the action space, let \(T\) be the horizon, let \(1\leq w\leq T\) be the window length, and let \(\mathcal C\subseteq\mathcal X^w\) denote the set of feasible action windows. An action sequence \(x_{1:T}=(x_1,\ldots,x_T)\) is \emph{feasible} if
\[
(x_{t-w+1},\ldots,x_t)\in\mathcal C,
\qquad t=w,\ldots,T.
\]
Thus, every realized block of \(w\) consecutive actions must satisfy the constraint exactly.
% \paragraph{Offline optimization.}
We first consider the setting in which the reward function \(f:\mathcal X\to\mathbb R\) is known. The offline optimal value is
\[
V_T(f)
:=
\max_{x_1,\ldots,x_T\in\mathcal X}
\sum_{t=1}^T f(x_t)
\]
subject to
\[
(x_{t-w+1},\ldots,x_t)\in\mathcal C,
\qquad t=w,\ldots,T.
\]

% \paragraph{Online linear-bandit feedback.}
We next consider the online setting with an unknown linear reward
\[
f_{\theta^\star}(x)
=
\langle\theta^\star,x\rangle,
\qquad
\theta^\star\in\mathbb R^d.
\]
At each round \(t\), the learner selects an action \(x_t\in\mathcal X\) and observes
\[
Y_t
=
\langle\theta^\star,x_t\rangle+\eta_t.
\]
Let
\(
\mathcal F_t
=
\sigma(x_1,Y_1,\ldots,x_t,Y_t)
\) denote the history after round \(t\). We require \(x_t\) to be \(\mathcal F_{t-1}\)-measurable. The noise process is conditionally mean-zero and \(R\)-sub-Gaussian:
\[
\begin{aligned}
\mathbb E[\eta_t\mid\mathcal F_{t-1}] &= 0,\\
\mathbb E\!\left[
\exp(u\eta_t)\mid\mathcal F_{t-1}
\right]
&\leq
\exp\!\left(\frac{u^2R^2}{2}\right),
\qquad u\in\mathbb R.
\end{aligned}
\]
A policy \(\pi\) is \emph{admissible} if each action \(x_t\) is \(\mathcal F_{t-1}\)-measurable and the resulting trajectory satisfies
\[
(x_{t-w+1},\ldots,x_t)\in\mathcal C,
\qquad t=w,\ldots,T,
\]
almost surely. For a fixed parameter \(\theta^\star\), define the offline optimal benchmark
\[
V_T^{\mathrm{off}}(\theta^\star)
:=
V_T(f_{\theta^\star}).
\]
The expected pseudo-regret of an admissible policy \(\pi\) relative to the offline-optimal feasible trajectory is
\[
\operatorname{Reg}^{\mathrm{off}}_\pi(T;\theta^\star)
:=
V_T^{\mathrm{off}}(\theta^\star)
-
\mathbb E_\pi\!\left[
\sum_{t=1}^T
\langle\theta^\star,x_t\rangle
\right].
\]
\section{Geometry: Offline Optimization under Cyclic Symmetry}
\label{sec:geometry}

In this section, we assume that the reward function is known and show how geometric structure in the feasible set can simplify the offline optimization problem.

\begin{assumption}
\label{ass:cyclic_invariance}
The window-feasibility set \(\mathcal C\subseteq\mathcal X^w\) is nonempty, closed, and convex. Define the cyclic-shift operator
\[
\mathcal S(x_1,\ldots,x_w)
=
(x_2,\ldots,x_w,x_1).
\]
We assume that \(\mathcal C\) is invariant under cyclic shifts: for every \((x_1,\ldots,x_w)\in\mathcal C\) and every \(k\in\{0,\ldots,w-1\}\),
\[
\mathcal S^k(x_1,\ldots,x_w)\in\mathcal C.
\]
\end{assumption}

\begin{assumption}
\label{ass:concave_reward}
The action space \(\mathcal X\subseteq\mathbb R^d\) is nonempty, compact, and convex. The reward function \(f:\mathcal X\to\mathbb R\) is concave and upper semicontinuous.
\end{assumption}

\begin{remark}
\label{rem:cyclic_constraint_generality}
Assumption~\ref{ass:cyclic_invariance} is satisfied by a broad class of sliding-window constraints. Appendix~\ref{app:demonstrating_functions} provides representative examples.
\end{remark}

\begin{theorem}
\label{thm:stationary_near_optimality}
Suppose Assumptions~\ref{ass:cyclic_invariance} and~\ref{ass:concave_reward} hold, and let
\[
\mathcal Z:=\{z\in\mathcal X:(z,\ldots,z)\in\mathcal C\}.
\]
Then \(\mathcal Z\) is nonempty and compact, and there exists \(z^\star\in\arg\max_{z\in\mathcal Z}f(z)\). For \(T=qw+r\), where \(q\in\mathbb N_0\) and \(0\le r<w\), the stationary trajectory \(x_t=z^\star\) is feasible and
\[
\begin{aligned}
0\leq V_T(f)-Tf(z^\star)
&\leq r\bigl(f_{\max}-f(z^\star)\bigr)\\
&\leq (w-1)\bigl(f_{\max}-f(z^\star)\bigr).
\end{aligned}
\]
where \(f_{\max}:=\max_{x\in\mathcal X}f(x)\). In particular, if \(w\mid T\), then \(x_t=z^\star\) is optimal and \(V_T(f)=Tf(z^\star)\).
\end{theorem}

The proof is provided in Appendix~\ref{app:proof-thm-stationary-near-optimality}. The dependence on \(w\) is unavoidable: Appendix~\ref{app:tightness-offline-gap} gives an example with a \(\Theta(w)\) gap.

% \paragraph{Proof sketch.}
% By cyclic invariance and convexity, averaging the cyclic shifts of any feasible length-\(w\) window produces a feasible stationary window. By concavity of \(f\), this averaging does not decrease the average reward. Thus, every feasible length-\(w\) block is no better than the optimal stationary window. Partitioning a horizon-\(T\) trajectory into \(q\) such blocks, with at most \(w-1\) remaining actions, yields the stated \(O(w)\) boundary gap.
\section{Reachability: Online Learning under Cyclic Symmetry}
\label{sec:linear-bandit}

We now turn to the online setting, where learning requires feasible movement between near-optimal stationary actions, motivating the notion of \emph{reachability}. We assume that, for known constants \(L,S>0\),
\[
\|x\|_2\leq L
\qquad\text{for all }x\in\mathcal X,
\qquad
\|\theta^\star\|_2\leq S.
\]

\begin{definition}
\label{def:stationary_regret}
The expected pseudo-regret relative to the best stationary feasible action is
\[
\operatorname{Reg}^{\mathcal Z}_{\pi}(T;\theta^\star)
:=
T\langle\theta^\star,z^\star\rangle
-
\mathbb E_\pi\!\left[
\sum_{t=1}^T
\langle\theta^\star,x_t\rangle
\right].
\]
\end{definition}

Writing \(T=qw+r\), with \(0\leq r<w\), Theorem~\ref{thm:stationary_near_optimality} applied to the linear reward \(f_{\theta^\star}(x)=\langle\theta^\star,x\rangle\) gives
\[
\begin{aligned}
0\leq \Delta_T
&:=
V_T^{\mathrm{off}}(\theta^\star)
-
T\langle\theta^\star,z^\star\rangle\\
&\leq
r\left(
\max_{x\in\mathcal X}
\langle\theta^\star,x\rangle
-
\langle\theta^\star,z^\star\rangle
\right).
\end{aligned}
\]
Consequently,
\[
\operatorname{Reg}^{\mathrm{off}}_{\pi}(T;\theta^\star)
=
\operatorname{Reg}^{\mathcal Z}_{\pi}(T;\theta^\star)
+
\Delta_T.
\]

\subsection{Why Geometry Alone Is Insufficient}
\label{sec:linear-regret-counterexample}

The offline assumptions do not guarantee sublinear online regret because they do not ensure that the histories induced by these actions are feasibly connected.

\begin{proposition}
\label{prop:linear-regret-without-transitions}
There exists a one-dimensional linear-bandit instance satisfying Assumption~\ref{ass:cyclic_invariance} such that, for every admissible policy \(\pi\),
\[
\sup_{\theta^\star\in\{-1,1\}}
\operatorname{Reg}^{\mathcal Z}_{\pi}(T;\theta^\star)
\geq T.
\]
Hence, sublinear stationary pseudo-regret cannot be guaranteed under Assumption~\ref{ass:cyclic_invariance} alone.
\end{proposition}

The proof is deferred to Appendix~\ref{proof_prop:linear-regret-without-transitions}. For \(w=1\), set \(\tau=0\). For \(w\geq2\), let
\[
\begin{aligned}
s &= (x_1,\ldots,x_{w-1})\in\mathcal X^{w-1},\\
F(s,x) &= (x_2,\ldots,x_{w-1},x).
\end{aligned}
\]
We define the stationary history
\[
h(z)=(z,\ldots,z)\in\mathcal X^{w-1},
\qquad z\in\mathcal Z.
\]
For complete histories \(s,s'\in\mathcal X^{w-1}\), let
\(d_{\mathcal C}(s,s')\) be the minimum number of feasible updates
needed to move from \(s\) to \(s'\), with
\(d_{\mathcal C}(s,s')=\infty\) if no such transition exists.
The \emph{transition diameter} is
\[
\tau
:=
\sup_{z,z'\in\mathcal Z}
d_{\mathcal C}\bigl(h(z),h(z')\bigr).
\]

\begin{assumption}
\label{ass:bounded-transitions}
The transition diameter is finite.
\end{assumption}

Appendix~\ref{app:transition_diameter} provides explicit bounds on \(\tau\) for several common sliding-window constraints. We assume access to a transition oracle \(\Gamma\) that, for any \(z,z'\in\mathcal Z\), returns a feasible transition
\[
\Gamma(z,z')
=
(a_1,\ldots,a_m),
\qquad
m\leq\tau.
\]

\subsection{Feasible Rare-Switching OFUL}
\label{sec:feasible-rare-switching-oful}

We combine OFUL with determinant-doubling episodes to limit switches between stationary actions. We assume access to an optimistic-action oracle \(\operatorname{Opt}\) and a feasible-transition oracle \(\Gamma\); details are given in Appendix~\ref{app:computational-oracles}. Let \(I_t\in\{0,1\}\) indicate whether round \(t\) is retained for estimation, and define
\[
\begin{aligned}
\mathcal O_t&:=\{s<t:I_s=1\},\\
V_t&=\lambda I_d+\sum_{s\in\mathcal O_t}x_sx_s^\top,\\
\widehat\theta_t
&=V_t^{-1}\sum_{s\in\mathcal O_t}x_sY_s.
\end{aligned}
\]
where \(\lambda\geq L^2\). The confidence radius is
\[
\beta_t(\delta)
=
R\sqrt{
2\log\!\left(
\frac{\det(V_t)^{1/2}}
{\det(\lambda I_d)^{1/2}\delta}
\right)}
+
\sqrt{\lambda}S.
\]

At the start \(t_k\) of episode \(k\), choose
\[
\begin{aligned}
z_k
&\in
\arg\max_{z\in\mathcal Z}
\left\{
\langle\widehat\theta_{t_k},z\rangle
+
\beta_{t_k}(\delta)\|z\|_{V_{t_k}^{-1}}
\right\}\\
&=
\operatorname{Opt}
\bigl(
\widehat\theta_{t_k},
V_{t_k},
\beta_{t_k}(\delta)
\bigr).
\end{aligned}
\]
The initial action is repeated for at least \(w-1\) rounds. Thereafter, \(z_k\) is played until
\(
\det(V_t)>2\det(V_{t_k}).
\)
If the newly computed action \(z_{k+1}\neq z_k\), the learner follows \(\Gamma(z_k,z_{k+1})\), which reaches \(h(z_{k+1})\) in at most \(\tau\) rounds. Transition observations are excluded from estimation. Algorithm~\ref{alg:rare-switching-oful} gives the complete procedure.

\begin{theorem}
\label{thm:expected-offline-regret}
Suppose the online linear-bandit assumptions of Section~\ref{sec:problem_formulation}, Assumptions~\ref{ass:cyclic_invariance}--\ref{ass:concave_reward}, and Assumption~\ref{ass:bounded-transitions} hold, and that the learner has exact access to \(\operatorname{Opt}\) and \(\Gamma\). Then the rare-switching OFUL policy is feasible and satisfies
\[
\operatorname{Reg}^{\mathrm{off}}_{\pi}(T;\theta^\star)
=
\widetilde{\mathcal O}
\left(
d\sqrt T+\tau d+w
\right).
\]
\end{theorem}

The proof is deferred to Appendix~\ref{proof_thm:expected-offline-regret}. Appendix~\ref{app:tau-lower-bound} shows that the linear dependence on \(\tau\) is unavoidable by constructing a one-dimensional instance with \(\Omega(\tau)\) regret.
\section{Planning: Learning beyond Cyclic Symmetry}
\label{sec:online_learning_without_cyclic}

We now move beyond cyclic symmetry and consider general sliding-window constraints, where the optimal feasible trajectory may be inherently non-stationary.

\subsection{Finite-Memory Representation}
\label{sec:finite_memory_planning}

Suppose \(w\geq2\). A complete history is
\[
s=(x_1,\ldots,x_{w-1})\in\mathcal X^{w-1}.
\]
Define
\[
\begin{aligned}
A(s)
&:= \{x\in\mathcal X:(s,x)\in\mathcal C\},\\
F(s,x)
&:= (x_2,\ldots,x_{w-1},x).
\end{aligned}
\]
The \(n\)-step viable histories are defined recursively by
\[
\begin{aligned}
\mathcal S_0
&:= \mathcal X^{w-1},\\
\mathcal S_n
&:=
\left\{
s:\exists x\in A(s),\;
F(s,x)\in\mathcal S_{n-1}
\right\},
\qquad n\geq 1.
\end{aligned}
\]
For \(s \in \mathcal S_n\), define
\[
\begin{aligned}
J_0^\star(s)
&:= 0,\\
J_n^\star(s)
&:=
\max_{\substack{x\in A(s)\\
F(s,x)\in\mathcal S_{n-1}}}
\left\{
\langle\theta^\star,x\rangle
+
J_{n-1}^\star(F(s,x))
\right\}.
\end{aligned}
\]

\begin{theorem}
\label{thm:bellman_optimality}
For every \(n\geq0\) and \(s\in\mathcal S_n\), \(J_n^\star(s)\) equals the optimal reward over all feasible \(n\)-step continuations from \(s\). Hence, for \(T\geq w-1\),
\[
V_T^{\mathrm{off}}(\theta^\star)
=
\max_{s\in\mathcal S_{T-w+1}}
\left\{
\sum_{i=1}^{w-1}\langle\theta^\star,x_i\rangle
+
J_{T-w+1}^\star(s)
\right\},
\]
where \(s=(x_1,\ldots,x_{w-1})\).
\end{theorem}

The proof is given in Appendix~\ref{app:thm_bellman_optimality}. This theorem shows that, without cyclic symmetry, the problem can be formulated as a deterministic finite-memory Markov decision process with recent action history as the state, and the Bellman recursion describes the offline optimal trajectory.

\subsection{History-State Reachability}
\label{sec:history_state_reachability}

For complete histories \(s,s'\), let \(d_{\mathcal C}(s,s')\) be the minimum number of feasible updates needed to move from \(s\) to \(s'\), with \(d_{\mathcal C}(s,s')=\infty\) if no such transition exists. Since \(\mathcal S_n\subseteq\mathcal S_1\) for all \(n\geq1\), it suffices to measure communication over \(\mathcal S_1\).

\begin{assumption}
\label{assump:history_diameter}
The \emph{history-state diameter}
\(
D
:=
\sup_{s,s'\in\mathcal S_1}
d_{\mathcal C}(s,s')
\)
is finite and independent of \(T\).
\end{assumption}

This extends the transition diameter from stationary histories to arbitrary viable histories. For many standard constraints, \(D=\mathcal{O}(w)\); see Appendix~\ref{app:history_diameter_examples}. For a bounded reward function \(g:\mathcal X\to\mathbb R\), let \(J_n^g(s)\) denote the optimal cumulative \(g\)-reward over \(n\) feasible steps starting from history state \(s\). Define
\[
\begin{aligned}
\rng(g)
&:=
\sup_{x\in\mathcal X}g(x)
-
\inf_{x\in\mathcal X}g(x),\\
\Span_{\mathcal S_n}(J_n^g)
&:=
\sup_{s\in\mathcal S_n}J_n^g(s)
-
\inf_{s\in\mathcal S_n}J_n^g(s).
\end{aligned}
\]

\begin{lemma}
\label{lem:value_span_diameter}
Under Assumption~\ref{assump:history_diameter}, for every bounded \(g\) and \(n\geq1\),
\[
\Span_{\mathcal{S}_n}(J_n^g)
\le D\,\rng(g).
\]
\end{lemma}

The proof is given in Appendix~\ref{app:lem_value_span_diameter}. For \(g(x)=\langle\theta^\star,x\rangle\), assuming \(\|x\|_2\leq L\) and \(\|\theta^\star\|_2\leq S\),
\[
\Span_{\mathcal S_n}(J_n^\star)
\leq
2LSD.
\]

\subsection{Rare-Update Optimistic Planning}
\label{sec:optimistic_noncyclic_planning}

We combine finite-memory Bellman planning with optimistic linear-bandit estimation. The learner uses an optimistic reward model based on the current confidence set and updates the Bellman policy only at determinant-doubling times.
\[
V_t
=
\lambda I_d+\sum_{q=1}^{t-1}x_qx_q^\top,
\qquad
\widehat\theta_t
=
V_t^{-1}\sum_{q=1}^{t-1}x_qY_q,
\]
with \(\lambda\geq L^2\) and
\[
\beta_t(\delta)
=
R\sqrt{
2\log\!\left(
\frac{\det(V_t)^{1/2}}
{\det(\lambda I_d)^{1/2}\delta}
\right)}
+
\sqrt{\lambda}S.
\]

For \(w>1\), the learner uses the feasible initialization oracle to obtain the first \(w-1\) actions of a valid trajectory, which define the initial history state, and planning begins at \(t_1=w\). At planning time \(t_k\), set
\[
W_k:=V_{t_k},
\qquad
\widehat\theta_k:=\widehat\theta_{t_k},
\qquad
\beta_k:=\beta_{t_k}(\delta),
\]
and keep these quantities fixed until the next update. Define the clipped optimistic reward
\[
\overline U_k(x)
:=
\operatorname{clip}_{[-LS,\,LS]}
\left(
\langle\widehat\theta_k,x\rangle
+
\beta_k\|x\|_{W_k^{-1}}
\right).
\]
Its Bellman value satisfies
\[
\begin{aligned}
J_{k,0}(s)&:=0,\\
J_{k,n}(s)
&=
\max_{\substack{x\in A(s)\\F(s,x)\in\mathcal S_{n-1}}}
\left\{
\overline U_k(x)
+
J_{k,n-1}(F(s,x))
\right\}.
\end{aligned}
\]
Thus, at round \(t\) of episode \(k\),
\[
x_t
\in
\arg\max_{\substack{x\in A(s_t)\\F(s_t,x)\in\mathcal S_{T-t}}}
\left\{
\overline U_k(x)
+
J_{k,T-t}(F(s_t,x))
\right\}.
\]
A new plan is computed when the determinant-doubling condition
\( \det(V_t)>2\det(V_{t_k}) \) is triggered.

\begin{assumption}
\label{assump:exact_optimistic_planner}
At each planning time, the learner can compute an exact Bellman-optimal remaining-horizon policy under \(\overline U_k\), as well as a feasible initialization when \(w>1\).
\end{assumption}

Assumption~\ref{assump:exact_optimistic_planner} separates the statistical guarantee from the computational problem of planning. For finite action spaces, exact planning can be implemented by backward dynamic programming, while general continuous constraints may require additional structure.

\begin{theorem}
\label{thm:noncyclic_regret}
Suppose the online linear-bandit assumptions of Section~\ref{sec:problem_formulation} hold, with \(\|x\|_2\leq L\), \(\|\theta^\star\|_2\leq S\), and \(\lambda\geq L^2\). Under Assumptions~\ref{assump:history_diameter} and~\ref{assump:exact_optimistic_planner}, the rare-update optimistic planning policy satisfies
\[
\operatorname{Reg}^{\mathrm{off}}_{\pi}(T;\theta^\star)
=
\widetilde{\mathcal O}
\left(
d\sqrt T+dD+w
\right).
\]
\end{theorem}

The proof is provided in Appendix~\ref{app:proof_thm_noncyclic_regret}. In Appendix~\ref{app:approximate_planning}, we relax Assumption~\ref{assump:exact_optimistic_planner} by allowing additive approximate remaining-horizon planning. Algorithm~\ref{alg:noncyclic-rare-switching} summarizes the complete procedure.
\section{Experiments}
\label{sec:experiments}

In this section, we evaluate our framework in two real-world sequential decision-making settings: a recommendation system with diversity constraints and LLM routing with cost constraints. In the recommendation setting, the learner sequentially selects content categories based on user preferences while enforcing sliding-window diversity constraints to prevent excessive repetition and promote diversified recommendations. In the LLM routing setting, the learner selects among multiple language models with different response qualities and inference costs. We impose sliding-window cost constraints to ensure that the total inference cost within each decision window remains below a predefined budget. The synthetic experiments are conducted in controlled settings, with the corresponding ablation studies and sensitivity analyses reported in Appendix~\ref{app:synthetic_experiments} and Appendix~\ref{app:ablation}, respectively.

\begin{table*}[t]
\centering
\caption{Recommendation results on KuaiRand-1K over 100 users, with 10 independent simulations per user and horizon. Average reward and normalized pseudo-regret are reported as mean $\pm$ 95\% confidence interval.}
\label{tab:kuairand-results}
\begin{tabular}{c l c c c c}
\toprule
$T$ & Method & Avg. Reward & Norm. Pseudo-Regret & Violations & Updates \\
\midrule

500
& Rare-Update Planning
& $0.4618 \pm 0.0343$
& $0.0540 \pm 0.0033$
& $0.0$
& $31.8$ \\

& Myopic Feasible OFUL
& $0.4638 \pm 0.0344$
& $0.0520 \pm 0.0031$
& $0.0$
& $500.0$ \\

& Primal-Dual OFUL
& $0.4639 \pm 0.0344$
& $0.0518 \pm 0.0031$
& $160.6$
& $500.0$ \\

& Unconstrained OFUL
& $0.4643 \pm 0.0344$
& $0.0515 \pm 0.0031$
& $177.6$
& $500.0$ \\

\midrule

2500
& Rare-Update Planning
& $0.4772 \pm 0.0350$
& $0.0386 \pm 0.0019$
& $0.0$
& $49.5$ \\

& Myopic Feasible OFUL
& $0.4801 \pm 0.0352$
& $0.0356 \pm 0.0017$
& $0.0$
& $2500.0$ \\

& Primal-Dual OFUL
& $0.4803 \pm 0.0352$
& $0.0355 \pm 0.0017$
& $1803.4$
& $2500.0$ \\

& Unconstrained OFUL
& $0.4854 \pm 0.0351$
& $0.0304 \pm 0.0021$
& $1864.2$
& $2500.0$ \\

\midrule

5000
& Rare-Update Planning
& $0.4844 \pm 0.0353$
& $0.0313 \pm 0.0014$
& $0.0$
& $57.0$ \\

& Myopic Feasible OFUL
& $0.4876 \pm 0.0354$
& $0.0282 \pm 0.0012$
& $0.0$
& $5000.0$ \\

& Primal-Dual OFUL
& $0.4876 \pm 0.0354$
& $0.0281 \pm 0.0013$
& $4163.3$
& $5000.0$ \\

& Unconstrained OFUL
& $0.4980 \pm 0.0354$
& $0.0177 \pm 0.0029$
& $4240.8$
& $5000.0$ \\

\bottomrule
\end{tabular}
\end{table*}

\paragraph{Recommendation with Diversity Constraints}

We consider an online recommendation problem based on the KuaiRand-1K dataset~\citep{gao2022kuairand}. The items are grouped into \(K\) content categories, with action set \(\mathcal{A}=\{1,\ldots,K\}\). At each round \(t\), the learner recommends a category \(a_t\in\mathcal{A}\) and observes a binary reward
\(
Y_t\sim\operatorname{Bernoulli}(\theta_{a_t}^\star),
\)
where \(\theta_j^\star\in[0,1]\) denotes the unknown click probability associated with category \(j\). The learner aims to identify the categories preferred by the user and maximize cumulative reward over time. A policy that focuses only on immediate reward may repeatedly recommend a small number of highly preferred categories, leading to overly repetitive suggestions. To maintain recommendation diversity, we impose exact sliding-window exposure constraints that limit how frequently each category can be recommended within any recent window. Further details of this setting are provided in Appendix~\ref{app:real_world_experiments}.

\paragraph{LLM Routing with Cost Constraints}

We consider an online LLM routing problem based on the LLMRouterBench dataset~\citep{li2026llmrouterbench}. The available language models form the action set \(\mathcal{A}=\{1,\ldots,K\}\), where each action corresponds to selecting a different LLM for answering a user query. At each round \(t\), the learner observes an incoming query \(q_t\), selects a feasible model \(a_t\in\mathcal{A}\), and receives a reward \(Y_t\) given by the corresponding LLMRouterBench quality score; if no model is feasible, a zero-cost, zero-reward no-op fallback is used. The learner uses query--model features \(\phi(q_t,a_t)\) and a linear predictor to estimate model quality and adaptively route queries across models. The learner aims to maximize response quality while controlling inference cost, since repeatedly selecting high-quality but expensive models can consume the available sliding-window budget and restrict subsequent routing decisions. To enforce resource efficiency, we impose exact sliding-window cost constraints that limit the total inference cost incurred within every recent window. Since future queries are unknown at decision time, the LLM-routing experiment uses myopic feasible selection after observing the current query rather than full remaining-horizon viability planning. Further details are provided in Appendix~\ref{app:real_world_experiments}.

\paragraph{Baselines.}
We compare our \emph{Rare-Update Optimistic Planning} method with three baselines. \emph{Myopic Feasible OFUL} selects the currently feasible action with the largest optimistic reward and recomputes its optimistic decision at every round, providing a fully updated feasible baseline against which to evaluate the rare-update mechanism. \emph{Primal-Dual OFUL} incorporates the sliding-window constraints through nonnegative Lagrangian dual variables and balances optimistic reward maximization with constraint violation penalties, inspired by primal-dual methods for constrained stochastic linear bandits~\citep{liu2021efficient}. \emph{Unconstrained OFUL} uses the same optimistic reward model but ignores the sliding-window constraints, providing a reference for the cost of enforcing exact feasibility. Further details are provided in Appendix~\ref{app:implementation_details}.

\paragraph{Real-world datasets.}
For the recommendation problem, we use the KuaiRand-1K dataset~\citep{gao2022kuairand}, collected from the Kuaishou short-video platform. The dataset contains \(1{,}000\) users, approximately \(4.37\) million videos, and \(11.7\) million user--video interactions with sequential histories and category information, enabling content-category recommendation under diversity constraints. For the LLM routing problem, we use LLMRouterBench~\citep{li2026llmrouterbench}, a large-scale benchmark containing over \(400\)K query instances from \(21\) datasets evaluated across \(33\) language models. It provides response quality scores and inference cost information, enabling LLM selection as an online decision problem under resource constraints.

\begin{table*}[t]
\centering
\caption{LLM routing results with sliding-window length \(w=5\) and budget \(B=0.40\) over 10 independent simulations for each horizon. Average reward and normalized pseudo-regret are reported as mean \(\pm\) 95\% confidence interval.}
\label{tab:llmrouter-results}
\begin{tabular}{c l c c c c}
\toprule
$T$ & Method & Avg. Reward & Norm. Pseudo-Regret & Violations & Updates \\
\midrule

500
& Rare-Update Optimistic Planning
& $0.7655 \pm 0.0145$
& $0.1645 \pm 0.0134$
& $0.0$
& $250.2$ \\

& Myopic Feasible OFUL
& $0.7701 \pm 0.0138$
& $0.1599 \pm 0.0130$
& $0.0$
& $500.0$ \\

& Unconstrained OFUL
& $0.7713 \pm 0.0167$
& $0.1587 \pm 0.0160$
& $4.7$
& $500.0$ \\

& Primal-Dual OFUL
& $0.7713 \pm 0.0167$
& $0.1587 \pm 0.0160$
& $4.7$
& $500.0$ \\

\midrule

2500
& Rare-Update Optimistic Planning
& $0.7775 \pm 0.0049$
& $0.1549 \pm 0.0040$
& $0.0$
& $789.8$ \\

& Myopic Feasible OFUL
& $0.7768 \pm 0.0048$
& $0.1556 \pm 0.0048$
& $0.0$
& $2500.0$ \\

& Unconstrained OFUL
& $0.7780 \pm 0.0059$
& $0.1543 \pm 0.0060$
& $38.6$
& $2500.0$ \\

& Primal-Dual OFUL
& $0.7782 \pm 0.0061$
& $0.1541 \pm 0.0062$
& $38.1$
& $2500.0$ \\

\midrule

5000
& Rare-Update Optimistic Planning
& $0.7797 \pm 0.0037$
& $0.1531 \pm 0.0038$
& $0.0$
& $1104.8$ \\

& Myopic Feasible OFUL
& $0.7793 \pm 0.0047$
& $0.1535 \pm 0.0047$
& $0.0$
& $5000.0$ \\

& Unconstrained OFUL
& $0.7769 \pm 0.0036$
& $0.1559 \pm 0.0036$
& $67.6$
& $5000.0$ \\

& Primal-Dual OFUL
& $0.7769 \pm 0.0036$
& $0.1559 \pm 0.0036$
& $67.6$
& $5000.0$ \\

\bottomrule
\end{tabular}
\end{table*}

\paragraph{Evaluation metrics.}
We evaluate performance using four metrics: average expected reward, normalized pseudo-regret, the number of sliding-window constraint violations, and the number of policy updates. The average expected reward measures the utility obtained by the selected actions, while the normalized pseudo-regret evaluates performance relative to the offline-optimal feasible policy. Constraint violations quantify the number of times a policy violates the corresponding sliding-window constraint, and the number of updates measures the frequency of policy recomputation. Detailed definitions of these metrics and their application-specific formulations are provided in Appendix~\ref{app:implementation_details}.

\paragraph{Recommendation Performance under Diversity Constraints.}
Table~\ref{tab:kuairand-results} summarizes the recommendation results at representative horizons. Unconstrained OFUL attains the lowest normalized pseudo-regret, but does so while incurring a large number of sliding-window constraint violations, particularly at longer horizons. In contrast, Rare-Update Optimistic Planning and Myopic Feasible OFUL satisfy the hard diversity constraints with zero violations across all reported horizons. Rare-Update Optimistic Planning achieves reward and normalized regret close to those of Myopic Feasible OFUL while requiring dramatically fewer policy updates: approximately $32$, $49$, and $57$ updates at $T=500$, $2500$, and $5000$, respectively, compared with one update per round for Myopic Feasible OFUL. These results highlight the effectiveness of the rare-update mechanism in maintaining competitive constrained performance while substantially reducing the frequency of policy recomputation.

\paragraph{LLM Routing Performance under Cost Constraints.}
Table~\ref{tab:llmrouter-results} reports LLM routing performance under the budget constraint \(B=0.40\) across different horizons \(T\). Since future query contexts are unavailable at decision time, we use the \(H=1\) specialization of Rare-Update Optimistic Planning. The method maintains zero constraint violations across all horizons and achieves reward and normalized pseudo-regret comparable to the other approaches, particularly at the longer horizons. Importantly, it substantially reduces the number of policy updates, requiring only \(250.2\), \(789.8\), and \(1104.8\) updates at \(T=500\), \(2500\), and \(5000\), respectively, compared with one update per round for Myopic Feasible OFUL. Myopic Feasible OFUL also satisfies the constraint exactly, whereas Unconstrained OFUL and Primal-Dual OFUL incur sliding-window cost violations whose frequency increases with the horizon. Additional results across budget levels and horizons are provided in Appendix~\ref{app:ablation}.
\section{Conclusion}
\label{sec:conclusion}
We study learning under exact sliding-window constraints, where each action affects both current reward and future feasibility. Our results show that the main difficulty can be understood through three structural ideas: geometry, reachability, and planning. Under convex and cyclically invariant constraints, cyclic symmetrization reduces the offline problem to stationary optimization: stationary behavior is optimal when \(w\mid T\) and within an additive \(\mathcal{O}(w)\) otherwise. In the online setting, geometry alone is not enough, since feasible behaviors may be disconnected. The transition diameter \(\tau\) captures this reachability cost and enters the regret through the price of feasibility-preserving transitions. Beyond cyclic symmetry, optimal behavior may be non-stationary. We therefore model the recent \(w-1\) actions as a finite-memory state and introduce the history-state diameter \(D\), which controls communication and continuation-value variation across viable histories. This leads to exact optimistic remaining-horizon planning with regret \(\widetilde{\mathcal O}(d\sqrt{T}+dD+w)\). Overall, the geometry--reachability--planning viewpoint separates the main sources of difficulty created by exact temporal feasibility and provides a basis for studying broader history-dependent constraint models.
\subsection*{AI use statement}

We used generative AI tools to improve the language and presentation of the manuscript and to review proof drafts for potentially nontrivial or inadvertently omitted steps. We did not use generative AI tools to formulate mathematical claims, develop the theoretical framework, design experiments, implement methods, clean or reformat datasets, or interpret experimental results. All AI-assisted feedback related to the proofs was independently checked and verified by the authors. The authors take full responsibility for the final content of the paper.

\bibliographystyle{apalike}
\bibliography{ctrl}

\clearpage
\appendix
\thispagestyle{empty}

% Supplementary material: To improve readability, you must use a single-column format for the supplementary material.
\onecolumn
\aistatstitle{Supplementary Materials}
\section{Limitations and Extensions}
\label{app:limitations}

Our results have several limitations. In the cyclic setting, the stationary reduction relies on convexity and cyclic-shift invariance, and may fail for discrete or nonconvex action spaces. In the general non-cyclic setting, our regret guarantee assumes a finite history-state diameter and access to exact remaining-horizon planning. For finite action spaces, the history-state space can grow as \(|\mathcal X|^{w-1}\), making exact planning computationally expensive for large \(w\).

These limitations suggest several directions for future work. Although the stationary reduction may fail without convexity, the finite-memory formulation remains applicable, motivating the study of periodic or recurrent feasible policies for discrete and nonconvex action spaces. Developing scalable planning methods that exploit the structure of feasible transitions, rather than enumerating the full history-state space, is another important direction. A further open question concerns the dependence on the transition diameter. Our rare-switching regret bound contains a transition term of order \(\widetilde{\mathcal O}(\tau d)\), whereas Proposition~\ref{prop:tau-lower-bound} establishes only an \(\Omega(\tau)\) lower bound. It remains open whether the additional factor of \(d\) is intrinsic or can be removed through a sharper algorithm or analysis.
\section{Computational Oracles}
\label{app:computational-oracles}

Our regret guarantees separate statistical learning from the computational tasks required to implement the algorithms. In the cyclic setting, each episode requires two operations: selecting an optimistic stationary action and constructing a feasible transition between stationary histories. The general non-cyclic setting additionally requires exact remaining-horizon planning. We formalize these operations below.

\paragraph{Optimistic-action oracle.}
Given an estimate \(\widehat\theta\in\mathbb R^d\), a positive-definite matrix \(V\succ0\), and a confidence radius \(\beta\geq0\), the optimistic-action oracle returns
\[
\operatorname{Opt}(\widehat\theta,V,\beta)
\in
\arg\max_{z\in\mathcal Z}
\left\{
\langle\widehat\theta,z\rangle
+
\beta\|z\|_{V^{-1}}
\right\}.
\]
Since \(\mathcal Z\) is compact and the objective is continuous, a maximizer exists. We assume that the oracle returns an exact maximizer. The objective
\[
z\mapsto
\langle\widehat\theta,z\rangle
+
\beta\|z\|_{V^{-1}}
\]
is convex in \(z\). Thus, even when \(\mathcal Z\) is convex, optimistic action selection is generally not a convex optimization problem because a convex function is being maximized. Efficient implementation therefore requires additional structure in \(\mathcal Z\).

\paragraph{Transition oracle.}
For \(w=1\), no history transition is required, and we take \(\tau=0\). For \(w\geq2\) and \(z,z'\in\mathcal Z\), the transition oracle returns
\[
\Gamma(z,z')
=
(a_1,\ldots,a_m),
\qquad
m\leq\tau,
\]
such that, starting from
\(
h_0=h(z),
\)
the histories
\[
h_j
=
F(h_{j-1},a_j),
\qquad
j=1,\ldots,m,
\]
satisfy
\[
(h_{j-1},a_j)\in\mathcal C,
\qquad
j=1,\ldots,m,
\]
and terminate at
\(
h_m=h(z').
\)
Hence every window encountered during the transition is feasible, and the target stationary history is reached in at most \(\tau\) rounds. When \(z=z'\), the oracle may return the empty transition.

The regret guarantee in Section~\ref{sec:linear-bandit} assumes exact access to both \(\operatorname{Opt}\) and \(\Gamma\). This isolates the statistical difficulty of learning the reward from the computational difficulty of optimizing over and navigating the feasible set.

\paragraph{Remaining-horizon planning oracle.}
In the general non-cyclic setting of Section~\ref{sec:online_learning_without_cyclic}, the learner must instead optimize over viable history states. Given a viable history \(s\in\mathcal S_n\) and an optimistic reward function \(U:\mathcal X\to\mathbb R\), the planning oracle computes an exact Bellman-optimal \(n\)-step continuation, with value
\[
J_0^U(s):=0,
\]
and, for \(n\geq1\),
\[
J_n^U(s)
=
\max_{\substack{x\in A(s)\\
F(s,x)\in\mathcal S_{n-1}}}
\left\{
U(x)+J_{n-1}^U(F(s,x))
\right\}.
\]
The regret guarantee for the non-cyclic setting is conditional on exact access to this planner. For finite action spaces, straightforward backward dynamic programming over viable histories takes \(\mathcal{O}(T|\mathcal X|^w)\) time in the direct history representation, up to viability preprocessing. For general continuous constraints, efficient planning requires additional problem structure.

\section{Relation to Linear MDPs and Constrained Reinforcement Learning}
\label{app:relation_rl}

In this section, we clarify the connection between our formulation and existing reinforcement-learning models. We first show that a direct reduction to linear MDPs~\citep{jin2020provably} can lead to much weaker dependence on the window length \(w\). The reason is that the linear-MDP framework must encode the history-dependent transition dynamics through a linear feature representation, whose dimension can grow exponentially with \(w\). In contrast, our setting has known deterministic history dynamics and only \(d\)-dimensional reward uncertainty, allowing our analysis to avoid this exponential dependence. We then compare our formulation with tabular and constrained reinforcement learning, highlighting why generic state-augmentation approaches can introduce exponential dependence on the window length and why existing constrained-RL guarantees do not directly capture our exact sliding-window setting.

\paragraph{Relation to linear MDPs.}
\citet{jin2020provably} consider a $p$-dimensional linear MDP with a known feature map $\phi:\mathcal S\times\mathcal A\to\mathbb R^p$ such that both the expected reward and the transition kernel are linear in the same features:
\[
r_h(s,a)=\langle\phi(s,a),\theta_h\rangle,\qquad
P_h(\cdot\mid s,a)=\langle\phi(s,a),\mu_h(\cdot)\rangle.
\]
Thus, the feature dimension $p$ must be rich enough to represent both rewards and transitions. Unconstrained linear bandits~\citep{dani2008stochastic,abbasi2011improved} arise as the special case $H=1$. Our sliding-window problem has a different structure. In Section~\ref{sec:online_learning_without_cyclic}, the state is the recent action history
\( s=(x_1,\ldots,x_{w-1}), \) and choosing action $x$ deterministically updates the state according to
\( F(s,x)=(x_2,\ldots,x_{w-1},x). \) Hence the transition kernel is known and deterministic: \(P(\cdot\mid s,x)=\delta_{F(s,x)}(\cdot)\).
In contrast, the unknown reward remains $d$-dimensional:
\[
r^\star(s,x)=\langle\theta^\star,x\rangle,\qquad \theta^\star\in\mathbb R^d,
\]
where the same parameter $\theta^\star$ is shared across all history states. Thus, the statistical uncertainty in our problem is only $d$-dimensional, even though the history-state representation can be much larger.

The following proposition shows that representing the natural history-shift dynamics as a linear MDP can require a feature dimension exponential in the window length.

\begin{proposition}
\label{prop:history_linear_mdp_dimension}
Let \(w\geq2\) and let \(\mathcal X=\{a_1,\ldots,a_m\}\subset\mathbb R^d\), with \(m\geq2\). Consider the history-state space \(\mathcal S=\mathcal X^{w-1}\) with deterministic transition \(F(s,x)=(x_2,\ldots,x_{w-1},x)\). Any exact linear-MDP representation of this transition kernel on the full history-state space requires feature dimension \( p\geq m^{w-1}. \) In particular, the required feature dimension is exponential in \(w\) for fixed \(m\geq2\).
\end{proposition}

\begin{proof}
Let \(P\) be the transition matrix with rows indexed by \((s,x)\in\mathcal S\times\mathcal X\) and columns indexed by \(s'\in\mathcal S\), where
\[
P_{(s,x),s'}=\mathbf 1\{F(s,x)=s'\}.
\]
Since \(|\mathcal S|=m^{w-1}\), the matrix \(P\) has \(m^{w-1}\) columns. Fix any \(s'=(b_1,\ldots,b_{w-1})\in\mathcal S\). For any \(a\in\mathcal X\), choose \(s=(a,b_1,\ldots,b_{w-2})\) and \(x=b_{w-1}\). Then \(F(s,x)=s'\), so the corresponding row of \(P\) is the standard basis vector associated with \(s'\). Hence the rows of \(P\) contain all \(m^{w-1}\) standard basis vectors, which are linearly independent. Therefore,
\( \operatorname{rank}(P)=m^{w-1}.\) If the transition kernel admits an exact \(p\)-dimensional linear-MDP representation, then in the finite-state setting \( P=\Phi M, \) where \(\Phi\) has \(p\) columns. Thus,
\[
m^{w-1}=\operatorname{rank}(P)\leq\operatorname{rank}(\Phi)\leq p,
\]
which proves \(p\geq m^{w-1}\).
\end{proof}

Proposition~\ref{prop:history_linear_mdp_dimension} clarifies the distinction between the reward dimension \(d\) and the feature dimension \(p\) required by a generic linear-MDP representation. Representing the known history-shift dynamics in the standard linear-MDP form may require \(p\geq m^{w-1}\), which is exponential in \(w\). Thus, although the history-augmented problem is an MDP, a low-dimensional linear-MDP representation need not preserve the structure that is useful in our setting. Our analysis instead learns only the \(d\)-dimensional reward parameter and handles the known history dynamics through planning and the history-state diameter \(D\). This yields the regret bound \(\widetilde{\mathcal O}(d\sqrt{T}+dD+w)\), whose statistical dependence is on \(d\) rather than on the potentially exponential representation dimension \(p\).

\paragraph{Comparison with a tabular reduction.}
One could instead treat the history-augmented process as a generic tabular MDP. For a finite action space with \(|\mathcal X|=m\), the full history representation has \(|\mathcal S|=m^{w-1}\) and \(|\mathcal A|=m\), and therefore \(|\mathcal S||\mathcal A|=m^w\). Thus, the state-action space grows exponentially with the window length \(w\). This growth can also appear in generic tabular RL regret guarantees. For example, \citet{azar2017minimax} establish finite-horizon tabular RL bounds with leading dependence of the form
\( \widetilde{\mathcal O}\!\left(\sqrt{H|\mathcal S||\mathcal A|N}\right), \)
where \(H\) is the episodic horizon and \(N\) is the total number of interaction steps. Applied to the full history-state representation, this becomes
\(
\widetilde{\mathcal O}\!\left(\sqrt{Hm^wN}\right)
=
\widetilde{\mathcal O}\!\left(m^{w/2}\sqrt{HN}\right),
\)
which has exponential dependence on \(w\).

Our analysis avoids this state-space dependence by exploiting the structure of the problem: the history dynamics are known and deterministic, while the reward uncertainty is governed by a single \(d\)-dimensional parameter shared across all histories. Rather than depending on the number of history states, our regret bound depends on their communication structure through the history-state diameter \(D\), yielding \(\widetilde{\mathcal O}(d\sqrt{T}+dD+w)\). For many constraint classes considered in this work, \(D=\mathcal{O}(w)\) (see Appendix~\ref{app:history_diameter_examples}), so the bound simplifies to \(\widetilde{\mathcal O}(d\sqrt{T}+dw)\). Thus, for these classes, the dependence on temporal memory is linear in \(w\), up to logarithmic and dimension factors, even though the full history-state space is exponential in \(w\). This highlights the advantage of exploiting the low-dimensional reward structure and known finite-memory dynamics rather than treating the problem as a generic tabular MDP.

\paragraph{Relation to constrained and safe reinforcement learning.}
Constrained Markov decision processes (CMDPs)~\citep{altman1999constrained} optimize expected cumulative reward subject to constraints on cumulative or expected costs, while safe reinforcement learning~\citep{amani2021safe} can impose stronger pathwise requirements that restrict the actions or states that may be visited during the trajectory. Our setting is closer to the latter notion: exact sliding-window feasibility requires every realized overlapping window to satisfy \((x_{t-w+1},\ldots,x_t)\in\mathcal C\) throughout the trajectory, rather than allowing violations that are compensated for later.

After augmenting the state with the previous \(w-1\) actions, the feasible-action set \(\mathcal A(s)=\{x:(s,x)\in\mathcal C\}\) can be interpreted as a state-dependent hard safety constraint. Thus, our problem can be viewed abstractly as a safe MDP. The key distinction lies in the structure of the problem: the history dynamics and constraint set are known, feasibility must hold exactly at every round, and all reward uncertainty is governed by a single \(d\)-dimensional parameter shared across history states. Generic safe-RL formulations do not typically exploit this combination of known finite-memory dynamics, exact pathwise feasibility, and low-dimensional shared reward uncertainty.

Moreover, applying a generic RL formulation after history augmentation can obscure this structure. As shown above, an exact linear-MDP representation of the full history dynamics may require feature dimension \(p\geq m^{w-1}\), while a tabular representation contains \(|\mathcal X|^{w-1}\) history states. To the best of our knowledge, existing constrained- and safe-RL guarantees do not directly yield this dependence for exact overlapping-window feasibility with low-dimensional reward uncertainty and known deterministic finite-memory dynamics.
\section{Examples of Convex Cyclically Invariant Constraints}
\label{app:demonstrating_functions}

The following examples illustrate useful classes of convex cyclically invariant window-feasibility sets. Throughout this appendix, for a window $x=(x_1,\ldots,x_w)\in\mathcal{X}^w$,
we use $x_i$ to denote its $i$th component.

\paragraph{Aggregate budget constraints.}
A cumulative budget constraint is given by
\[
\mathcal C
=
\left\{
x\in\mathcal X^w:
\sum_{i=1}^{w} c^\top x_i \leq B
\right\}.
\]

\paragraph{Maximum constraints.}
Worst-case constraints within a window can be expressed as
\[
\mathcal C
=
\left\{
x\in\mathcal X^w:
\max_{1\leq i\leq w}\phi(x_i)\leq \rho
\right\},
\]
where \(\phi\) is convex. A common example is
\[
\max_{1\leq i\leq w}\|Ax_i-b\|\leq \rho,
\]
which limits the maximum deviation of any action in the window.

\paragraph{Box constraints.}
Given \(\ell,u\in\mathbb R^d\) with \(\ell\leq u\),
\[
\mathcal C
=
\left\{
x\in\mathcal X^w:
\ell\leq x_i\leq u,\quad i=1,\ldots,w
\right\}.
\]

\paragraph{Ramping constraints.}
Using \(x_{w+1}=x_1\), cyclic variation can be controlled by
\[
\mathcal C
=
\left\{
x\in\mathcal X^w:
\max_{1\leq i\leq w}\|x_{i+1}-x_i\|\leq \rho
\right\}.
\]

\paragraph{Exponential constraints.}
Convex exponential penalties yield
\[
\mathcal C
=
\left\{
x\in\mathcal X^w:
\sum_{i=1}^{w}\exp(a^\top x_i-b)\leq \rho
\right\},
\]
which can represent cumulative risk, energy consumption, or congestion.

\paragraph{\(\ell_p\)-norm constraints.}
For \(p\geq1\),
\[
\mathcal C
=
\left\{
x\in\mathcal X^w:
\left(
\sum_{i=1}^{w}\|Ax_i-b\|_p^p
\right)^{1/p}
\leq \rho
\right\}
\]
captures aggregate deviation and robustness requirements.

\paragraph{Convex risk constraints.}
For a convex loss function \(\ell\),
\[
\mathcal C
=
\left\{
x\in\mathcal X^w:
\frac{1}{w}\sum_{i=1}^{w}\ell(x_i)\leq \rho
\right\}
\]
provides a general risk-based window constraint.

\section{Transition Diameter for Common Sliding-Window Constraints}
\label{app:transition_diameter}

We verify Assumption~\ref{ass:bounded-transitions} for the constraint classes introduced in Appendix~\ref{app:demonstrating_functions}. Table~\ref{tab:transition-diameters} summarizes the resulting transition-diameter bounds. Throughout, assume that the corresponding stationary feasible set \(\mathcal Z\) is nonempty. We consider \(w\geq2\); when \(w=1\), no transition history is required and one may take \(\tau=0\). For \(z,z'\in\mathcal X\) and \(j\in\{0,\ldots,w\}\), define the mixed window
\[
M_j(z,z')
:=
(\underbrace{z,\ldots,z}_{w-j},
 \underbrace{z',\ldots,z'}_{j}).
\]

\begin{lemma}
\label{lem:overwrite-transition}
Suppose that, for every \(z,z'\in\mathcal Z\),
\[
M_j(z,z')\in\mathcal C,
\qquad
j=1,\ldots,w-1.
\]
Then Assumption~\ref{ass:bounded-transitions} holds with
\( \tau\leq w-1. \)
\end{lemma}

\begin{proof}
Fix \(z,z'\in\mathcal Z\) and select
\[
a_j=z',
\qquad
j=1,\ldots,w-1.
\]
Starting from \(h(z)\), the complete window generated by the \(j\)-th transition action is \(M_j(z,z')\), which is feasible by assumption. After \(w-1\) actions, all entries of the original history have been overwritten, and the resulting history is
\(
h_{w-1}=h(z').
\)
Thus the transition is feasible and has length \(w-1\).
\end{proof}

%%%%%%%%%%%%%%%%%%%%%%%%%%%%%%%%%%%%%%%%%%%%%%%%%%%%%%%%%%%%%%%%%%%%%%%%%%%%%%%

\begin{proposition}[Aggregate budget constraints]
\label{prop:budget-transition}
Let \(c\in\mathbb R^d\), \(B\in\mathbb R\), and define
\[
\mathcal C_{\mathrm{budget}}
:=
\left\{
(x_1,\ldots,x_w)\in\mathcal X^w:
\sum_{i=1}^w c^\top x_i\leq B
\right\}.
\]
Then \(\mathcal C_{\mathrm{budget}}\) is convex and cyclically invariant, and
its transition diameter satisfies
\( \tau\leq w-1. \)
\end{proposition}

\begin{proof}
The stationary feasible set is
\[
\mathcal Z
=
\left\{
z\in\mathcal X:
c^\top z\leq\frac{B}{w}
\right\}.
\]
For \(z,z'\in\mathcal Z\) and \(j=1,\ldots,w-1\),
\[
\begin{aligned}
\sum_{i=1}^w c^\top M_j(z,z')_i
&=
(w-j)c^\top z+jc^\top z'
\leq
(w-j)\frac{B}{w}
+j\frac{B}{w}
=B.
\end{aligned}
\]
Thus every mixed window is feasible, and Lemma~\ref{lem:overwrite-transition} applies.
\end{proof}

%%%%%%%%%%%%%%%%%%%%%%%%%%%%%%%%%%%%%%%%%%%%%%%%%%%%%%%%%%%%%%%%%%%%%%%%%%%%%%%

\begin{proposition}[Maximum constraints]
\label{prop:maximum-transition}
Let \(\phi:\mathcal X\to\mathbb R\) be convex and define
\[
\mathcal C_{\max}
:=
\left\{
(x_1,\ldots,x_w)\in\mathcal X^w:
\max_{1\leq i\leq w}\phi(x_i)\leq\rho
\right\}.
\]
Then \(\mathcal C_{\max}\) is convex and cyclically invariant, and its
transition diameter satisfies
\( \tau\leq w-1. \)
\end{proposition}

\begin{proof}
The stationary feasible set is
\[
\mathcal Z
=
\{z\in\mathcal X:\phi(z)\leq\rho\}.
\]
For \(z,z'\in\mathcal Z\),
\[
\max_{1\leq i\leq w}
\phi\!\left(M_j(z,z')_i\right)
=
\max\{\phi(z),\phi(z')\}
\leq\rho.
\]
Therefore every mixed window is feasible, and
Lemma~\ref{lem:overwrite-transition} applies.
\end{proof}

%%%%%%%%%%%%%%%%%%%%%%%%%%%%%%%%%%%%%%%%%%%%%%%%%%%%%%%%%%%%%%%%%%%%%%%%%%%%%%%

\begin{proposition}[Box constraints]
\label{prop:box-transition}
Let \(\ell,u\in\mathbb R^d\) satisfy \(\ell\leq u\), and define
\[
\mathcal C_{\mathrm{box}}
:=
\left\{
(x_1,\ldots,x_w)\in\mathcal X^w:
\ell\leq x_i\leq u,\quad i=1,\ldots,w
\right\}.
\]
Then \(\mathcal C_{\mathrm{box}}\) is convex and cyclically invariant, and its transition diameter satisfies \( \tau\leq w-1. \)
\end{proposition}

\begin{proof}
If \(z,z'\in\mathcal Z\), then
\[
\ell\leq z\leq u,
\qquad
\ell\leq z'\leq u.
\]
Every entry of \(M_j(z,z')\) is either \(z\) or \(z'\), and therefore satisfies the box constraint. Hence every mixed window is feasible, and Lemma~\ref{lem:overwrite-transition} applies.
\end{proof}

%%%%%%%%%%%%%%%%%%%%%%%%%%%%%%%%%%%%%%%%%%%%%%%%%%%%%%%%%%%%%%%%%%%%%%%%%%%%%%%

\begin{proposition}[Exponential constraints]
\label{prop:exponential-transition}
Let \(a\in\mathbb R^d\), \(b\in\mathbb R\), and \(\rho>0\). Define
\[
\mathcal C_{\exp}
:=
\left\{
(x_1,\ldots,x_w)\in\mathcal X^w:
\sum_{i=1}^w\exp(a^\top x_i-b)\leq\rho
\right\}.
\]
Then \(\mathcal C_{\exp}\) is convex and cyclically invariant, and its transition diameter satisfies
\(
\tau\leq w-1.
\)
\end{proposition}

\begin{proof}
Convexity follows because \(x\mapsto\exp(a^\top x-b)\) is convex, while cyclic invariance follows from invariance of the sum under permutations. The stationary feasible set is
\(
\mathcal Z
=
\left\{
z\in\mathcal X:
w\exp(a^\top z-b)\leq\rho
\right\}.
\)
For \(z,z'\in\mathcal Z\),
\[
\begin{aligned}
\sum_{i=1}^w
\exp\!\left(a^\top M_j(z,z')_i-b\right)
&=
(w-j)\exp(a^\top z-b)
+j\exp(a^\top z'-b)
\leq
(w-j)\frac{\rho}{w}
+j\frac{\rho}{w}
=\rho.
\end{aligned}
\]
Thus every mixed window is feasible, and
Lemma~\ref{lem:overwrite-transition} applies.
\end{proof}

%%%%%%%%%%%%%%%%%%%%%%%%%%%%%%%%%%%%%%%%%%%%%%%%%%%%%%%%%%%%%%%%%%%%%%%%%%%%%%%

\begin{proposition}[\(\ell_p\)-norm constraints]
\label{prop:lp-transition}
Let \(A\in\mathbb R^{r\times d}\), \(b\in\mathbb R^r\), and \(p\in[1,\infty)\). Define
\[
\mathcal C_p
:=
\left\{
(x_1,\ldots,x_w)\in\mathcal X^w:
\left(
\sum_{i=1}^w\|Ax_i-b\|_p^p
\right)^{1/p}
\leq\rho
\right\}.
\]
Then \(\mathcal C_p\) is convex and cyclically invariant, and its transition diameter satisfies \(  \tau\leq w-1.\)
\end{proposition}

\begin{proof}
The left-hand side is the \(\ell_p\)-norm of the stacked affine residual
vector
\(
\bigl(
(Ax_1-b)^\top,\ldots,(Ax_w-b)^\top
\bigr)^\top,
\)
and is therefore convex. It is unchanged by cyclic permutations of the actions. For a stationary action \(z\), feasibility is equivalent to
\(
w^{1/p}\|Az-b\|_p\leq\rho.
\)
Hence
\(
\mathcal Z
=
\left\{
z\in\mathcal X:
\|Az-b\|_p\leq\frac{\rho}{w^{1/p}}
\right\}.
\)
For \(z,z'\in\mathcal Z\),
\[
\begin{aligned}
\sum_{i=1}^w
\left\|A M_j(z,z')_i-b\right\|_p^p
&=
(w-j)\|Az-b\|_p^p
+j\|Az'-b\|_p^p
\leq
(w-j)\frac{\rho^p}{w}
+j\frac{\rho^p}{w}
=\rho^p.
\end{aligned}
\]
Taking the \(p\)-th root shows that \(M_j(z,z')\in\mathcal C_p\). The result
follows from Lemma~\ref{lem:overwrite-transition}.
\end{proof}

\begin{remark}
The case \(p=\infty\) becomes
\(
\max_{1\leq i\leq w}\|Ax_i-b\|_\infty\leq\rho
\)
and is covered by Proposition~\ref{prop:maximum-transition}.
\end{remark}

%%%%%%%%%%%%%%%%%%%%%%%%%%%%%%%%%%%%%%%%%%%%%%%%%%%%%%%%%%%%%%%%%%%%%%%%%%%%%%%

\begin{proposition}[Convex risk constraints]
\label{prop:risk-transition}
Let \(\ell:\mathcal X\to\mathbb R\) be convex and define
\[
\mathcal C_{\mathrm{risk}}
:=
\left\{
(x_1,\ldots,x_w)\in\mathcal X^w:
\frac{1}{w}\sum_{i=1}^w\ell(x_i)\leq\rho
\right\}.
\]
Then \(\mathcal C_{\mathrm{risk}}\) is convex and cyclically invariant, and its transition diameter satisfies \( \tau\leq w-1. \)
\end{proposition}

\begin{proof}
The stationary feasible set is
\(
\mathcal Z
=
\{z\in\mathcal X:\ell(z)\leq\rho\}.
\)
For \(z,z'\in\mathcal Z\),
\[
\begin{aligned}
\frac{1}{w}
\sum_{i=1}^w
\ell\!\left(M_j(z,z')_i\right)
&=
\frac{w-j}{w}\ell(z)
+
\frac{j}{w}\ell(z')
\leq\rho.
\end{aligned}
\]
Thus every mixed window is feasible, and
Lemma~\ref{lem:overwrite-transition} applies.
\end{proof}

%%%%%%%%%%%%%%%%%%%%%%%%%%%%%%%%%%%%%%%%%%%%%%%%%%%%%%%%%%%%%%%%%%%%%%%%%%%%%%%

\begin{proposition}[Cyclic ramping constraints]
\label{prop:ramping-transition}
Suppose that \(\mathcal X\) is compact and convex and let
\(
D_{\mathcal X}
:=
\sup_{x,y\in\mathcal X}\|x-y\|<\infty.
\)
For \(\rho>0\), define
\[
\mathcal C_{\mathrm{ramp}}
:=
\left\{
(x_1,\ldots,x_w)\in\mathcal X^w:
\max_{1\leq i\leq w}\|x_{i+1}-x_i\|\leq\rho
\right\},
\qquad x_{w+1}=x_1.
\]
Then \(\mathcal C_{\mathrm{ramp}}\) is convex and cyclically invariant, and
Assumption~\ref{ass:bounded-transitions} holds with
\[
\tau
\leq
\left\lceil
\frac{(w-1)D_{\mathcal X}}{\rho}
\right\rceil
+w-2.
\]
\end{proposition}

\begin{proof}
The function
\(
(x_1,\ldots,x_w)
\mapsto
\max_{1\leq i\leq w}\|x_{i+1}-x_i\|
\)
is convex and cyclically invariant. Moreover, every stationary window has zero cyclic variation, so
\(
\mathcal Z=\mathcal X.
\) Fix distinct \(z,z'\in\mathcal Z\) and set
\[
N
:=
\left\lceil
\frac{(w-1)\|z'-z\|}{\rho}
\right\rceil.
\]
Define
\[
y_r
:=
z+\frac{r}{N}(z'-z),
\qquad
r=0,\ldots,N.
\]
Convexity of \(\mathcal X\) implies \(y_r\in\mathcal X\), and
\[
\|y_r-y_{r-1}\|
=
\frac{\|z'-z\|}{N}
\leq
\frac{\rho}{w-1}.
\]
Starting from \(h(z)\), play
\(
y_1,\ldots,y_N
\)
and then play \(z'=y_N\) for an additional \(w-2\) rounds. In every complete window, consecutive actions differ by at most \(\rho/(w-1)\). The cyclic difference between the first and last actions in the window is therefore at most
\[
(w-1)\frac{\rho}{w-1}
=
\rho.
\]
Hence every generated window belongs to \(\mathcal C_{\mathrm{ramp}}\). After \(N+w-2\) actions, the history equals \(h(z')\). Finally,
\[
N+w-2
\leq
\left\lceil
\frac{(w-1)D_{\mathcal X}}{\rho}
\right\rceil+w-2.
\]
When \(z=z'\), the empty transition is feasible.
\end{proof}

%%%%%%%%%%%%%%%%%%%%%%%%%%%%%%%%%%%%%%%%%%%%%%%%%%%%%%%%%%%%%%%%%%%%%%%%%%%%%%%

\begin{remark}
\label{rem:transition-diameter-sharpness}
Suppose that \(\mathcal Z\) contains distinct actions \(z\neq z'\). Any transition from \(h(z)\) to \(h(z')\) requires at least \(w-1\) actions because each new action replaces at most one entry of the current history. Therefore, for the aggregate-budget, maximum, box, exponential, \(\ell_p\)-norm, and convex-risk classes, \( \tau=w-1. \) The ramping class may require more than \(w-1\) actions because transitions must proceed gradually.
\end{remark}

%%%%%%%%%%%%%%%%%%%%%%%%%%%%%%%%%%%%%%%%%%%%%%%%%%%%%%%%%%%%%%%%%%%%%%%%%%%%%%%

\begin{table}[t]
\centering
\small
\caption{Transition-diameter for convex and cyclically invariant sliding-window constraints.}
\label{tab:transition-diameters}
\renewcommand{\arraystretch}{1.35}
\setlength{\tabcolsep}{5pt}
\begin{tabular}{
@{}
>{\centering\arraybackslash}p{0.18\linewidth}
>{\centering\arraybackslash}p{0.40\linewidth}
>{\centering\arraybackslash}p{0.27\linewidth}
@{}
}
\hline
\textbf{Constraint class}
&
\textbf{Window condition}
&
\textbf{Transition diameter}
\\
\hline

Aggregate budget
&
\(\displaystyle
\sum_{i=1}^w c^\top x_i\leq B
\)
&
\(\tau\leq w-1\)
\\

Maximum
&
\(\displaystyle
\max_{1\leq i\leq w}\phi(x_i)\leq\rho
\)
&
\(\tau\leq w-1\)
\\

Box
&
\(\displaystyle
\ell\leq x_i\leq u,\quad i=1,\ldots,w
\)
&
\(\tau\leq w-1\)
\\

Exponential
&
\(\displaystyle
\sum_{i=1}^w\exp(a^\top x_i-b)\leq\rho
\)
&
\(\tau\leq w-1\)
\\

\(\ell_p\)-norm
&
\(\displaystyle
\left(
\sum_{i=1}^w\|Ax_i-b\|_p^p
\right)^{1/p}
\leq\rho
\)
&
\(\tau\leq w-1\)
\\

Convex risk
&
\(\displaystyle
\frac{1}{w}\sum_{i=1}^w\ell(x_i)\leq\rho
\)
&
\(\tau\leq w-1\)
\\

Cyclic ramping
&
\(\displaystyle
\max_i\|x_{i+1}-x_i\|\leq\rho,
\quad x_{w+1}=x_1
\)
&
\(\displaystyle
\tau\leq
\left\lceil
\frac{(w-1)D_{\mathcal X}}{\rho}
\right\rceil
+w-2
\)
\\
\hline
\end{tabular}

\medskip

{\footnotesize
\emph{Note.}
For all classes except cyclic ramping, the transition is obtained by playing \(z'\) for \(w-1\) rounds. If the corresponding stationary feasible set contains at least two distinct actions, these bounds are sharp: \(\tau=w-1\).}
\end{table}
\section{History Diameter for Common Constraint Classes}
\label{app:history_diameter_examples}

The history diameter can be computed or bounded explicitly for several common sliding-window constraints.

\paragraph{State-independent constraints.}
Suppose \( \mathcal C=\mathcal A^w \) for some nonempty \(\mathcal A\subseteq\mathcal X\).

\begin{proposition}[State-Independent Constraints]
\label{prop:state_independent_D}
If \(\mathcal A\) contains at least two distinct actions, then
\(
D=w-1.
\)
If \(\mathcal A\) is a singleton, then \(D=0\).
\end{proposition}

\begin{proof}
For every \(n\geq1\),
\(
\mathcal S_n=\mathcal A^{w-1}.
\)
Fix
\[
s=(x_1,\ldots,x_{w-1}),
\qquad
s'=(y_1,\ldots,y_{w-1})
\]
in \(\mathcal S_n\). Starting from \(s\), append \(y_1,\ldots,y_{w-1}\). Every resulting window lies in \(\mathcal A^w\), and after \(w-1\) actions the history equals \(s'\). Hence
\(
d_{\mathcal C}(s,s')\leq w-1,
\)
so \(D\leq w-1\).
For the reverse inequality, choose distinct \(a,b\in\mathcal A\) and consider
\[
s=(a,\ldots,a),
\qquad
s'=(b,\ldots,b).
\]
Each appended action replaces only one coordinate of the current history, so at least \(w-1\) actions are required to transform \(s\) into \(s'\). Therefore,
\(
d_{\mathcal C}(s,s')\geq w-1.
\)
Thus \(D=w-1\). The singleton case is immediate.
\end{proof}

\paragraph{Box constraints.}
Consider
\[
\mathcal C_{\mathrm{box}}
=
\left\{
(x_1,\ldots,x_w)\in\mathcal X^w:
\ell\leq x_i\leq u,\quad i=1,\ldots,w
\right\},
\]
and define
\[
\mathcal A_{\mathrm{box}}
:=
\{x\in\mathcal X:\ell\leq x\leq u\}.
\]

\begin{proposition}[Box Constraints]
\label{prop:box_D}
If \(\mathcal A_{\mathrm{box}}\) contains at least two distinct actions, then
\(
D=w-1.
\)
If \(\mathcal A_{\mathrm{box}}\) is a singleton, then \(D=0\).
\end{proposition}

\begin{proof}
The constraint is state-independent:
\(
\mathcal C_{\mathrm{box}}=\mathcal A_{\mathrm{box}}^w.
\)
The result follows directly from Proposition~\ref{prop:state_independent_D}.
\end{proof}

\paragraph{Maximum constraints.}
Consider
\[
\mathcal C_{\max}
=
\left\{
(x_1,\ldots,x_w)\in\mathcal X^w:
\max_{1\leq i\leq w}\phi(x_i)\leq\rho
\right\},
\]
and define
\[
\mathcal A_{\max}
:=
\{x\in\mathcal X:\phi(x)\leq\rho\}.
\]

\begin{proposition}[Maximum Constraints]
\label{prop:maximum_D}
If \(\mathcal A_{\max}\) contains at least two distinct actions, then
\(
D=w-1.
\)
If \(\mathcal A_{\max}\) is a singleton, then \(D=0\).
\end{proposition}

\begin{proof}
The maximum constraint holds if and only if every action belongs to \(\mathcal A_{\max}\). Hence
\(
\mathcal C_{\max}=\mathcal A_{\max}^w,
\)
and Proposition~\ref{prop:state_independent_D} gives the result.
\end{proof}

\paragraph{A general reset principle for additive constraints.}
\label{app:reset_history_diameter}

Consider a separable additive constraint
\[
\mathcal C_g
=
\left\{
(x_1,\ldots,x_w)\in\mathcal X^w:
\sum_{i=1}^w g(x_i)\leq B
\right\},
\]
where \(g:\mathcal X\to\mathbb R\). Suppose \(g\) attains its minimum over \(\mathcal X\), and let
\[
x_0\in\arg\min_{x\in\mathcal X}g(x),
\qquad
g_0:=g(x_0).
\]

\begin{proposition}[Reset-Based History-Diameter Bound]
\label{prop:reset_diameter}
Suppose \(\mathcal C_g\neq\varnothing\). Then
\(
D\leq2(w-1).
\)
\end{proposition}

\begin{proof}
Fix \(n\geq1\) and
\[
s=(x_1,\ldots,x_{w-1}),
\qquad
s'=(y_1,\ldots,y_{w-1})
\]
in \(\mathcal S_n\).  Since \(\mathcal S_n\subseteq\mathcal S_1\), there exists \(a\in\mathcal X\) such that
\[
\sum_{i=1}^{w-1}g(x_i)+g(a)\leq B.
\]
Since \(g_0\leq g(a)\),
\[
\sum_{i=1}^{w-1}g(x_i)+g_0\leq B,
\]
so \(x_0\) can be appended feasibly. After one reset action, the history becomes
\(
(x_2,\ldots,x_{w-1},x_0).
\)
Since \(g_0\leq g(x_1)\), replacing \(x_1\) by \(x_0\) cannot increase the window cost. Repeating this argument shows that \(x_0\) can be played \(w-1\) times while preserving feasibility. Hence
\(
s\longrightarrow h(x_0):=(x_0,\ldots,x_0)
\)
in at most \(w-1\) actions. Similarly, since \(s'\in\mathcal S_1\), there exists \(a'\in\mathcal X\) such that
\[
\sum_{i=1}^{w-1}g(y_i)+g(a')\leq B.
\]
Because \(g_0\leq g(a')\),
\[
g_0+\sum_{i=1}^{w-1}g(y_i)\leq B.
\]
Starting from \(h(x_0)\), append \(y_1,\ldots,y_{w-1}\). After \(j\in\{1,\ldots,w-1\}\) target actions have been appended, the current window has cost
\[
(w-j)g_0+\sum_{i=1}^j g(y_i).
\]
Since \(g_0\leq g(y_i)\) for every \(i\),
\[
(w-j)g_0+\sum_{i=1}^j g(y_i)
\leq
g_0+\sum_{i=1}^{w-1}g(y_i)
\leq B.
\]
Thus every intermediate window is feasible, and \(s'\) is reached after at most \(w-1\) additional actions. Therefore,
\(
d_{\mathcal C}(s,s')\leq2(w-1).
\)
Since \(n\) and \(s,s'\in\mathcal S_n\) were arbitrary,
\(
D\leq2(w-1).
\)
\end{proof}

\paragraph{Aggregate budget constraints.}
Consider
\[
\mathcal C_{\mathrm{budget}}
=
\left\{
(x_1,\ldots,x_w)\in\mathcal X^w:
\sum_{i=1}^w c^\top x_i\leq B
\right\}.
\]

\begin{proposition}[Aggregate Budget Constraints]
\label{prop:budget_D}
Suppose \(\mathcal C_{\mathrm{budget}}\neq\varnothing\) and \(c^\top x\) attains its minimum over \(\mathcal X\). Then
\(
D\leq2(w-1).
\)
\end{proposition}

\begin{proof}
Set
\(
g(x):=c^\top x.
\)
Then
\[
\mathcal C_{\mathrm{budget}}
=
\left\{
(x_1,\ldots,x_w):
\sum_{i=1}^w g(x_i)\leq B
\right\}.
\]
The claim follows from Proposition~\ref{prop:reset_diameter}.
\end{proof}

\paragraph{Exponential constraints.}
Consider
\[
\mathcal C_{\exp}
=
\left\{
(x_1,\ldots,x_w)\in\mathcal X^w:
\sum_{i=1}^w e^{a^\top x_i-b}\leq\rho
\right\}.
\]

\begin{proposition}[Exponential Constraints]
\label{prop:exponential_D}
Suppose \(\mathcal C_{\exp}\neq\varnothing\) and \(x\mapsto e^{a^\top x-b}\) attains its minimum over \(\mathcal X\). Then
\(
D\leq2(w-1).
\)
\end{proposition}

\begin{proof}
Set
\(
g(x):=e^{a^\top x-b}.
\)
Then
\[
\mathcal C_{\exp}
=
\left\{
(x_1,\ldots,x_w):
\sum_{i=1}^w g(x_i)\leq\rho
\right\}.
\]
The result follows from Proposition~\ref{prop:reset_diameter}.
\end{proof}

\paragraph{Aggregate \(\ell_p\)-norm constraints.}
Consider
\[
\mathcal C_p
=
\left\{
(x_1,\ldots,x_w)\in\mathcal X^w:
\sum_{i=1}^w\|Ax_i-b\|_p^p\leq\rho^p
\right\}.
\]

\begin{proposition}[Aggregate \(\ell_p\)-Norm Constraints]
\label{prop:lp_D}
Suppose \(\mathcal C_p\neq\varnothing\) and \(x\mapsto\|Ax-b\|_p^p\) attains its minimum over \(\mathcal X\). Then
\(
D\leq2(w-1).
\)
\end{proposition}

\begin{proof}
Set
\(
g(x):=\|Ax-b\|_p^p.
\)
Then
\[
\mathcal C_p
=
\left\{
(x_1,\ldots,x_w):
\sum_{i=1}^w g(x_i)\leq\rho^p
\right\}.
\]
The result follows from Proposition~\ref{prop:reset_diameter}.
\end{proof}

\paragraph{Convex-risk constraints.}
Consider
\[
\mathcal C_{\mathrm{risk}}
=
\left\{
(x_1,\ldots,x_w)\in\mathcal X^w:
\frac1w\sum_{i=1}^w\ell(x_i)\leq\rho
\right\}.
\]

\begin{proposition}[Convex-Risk Constraints]
\label{prop:risk_D}
Suppose \(\mathcal C_{\mathrm{risk}}\neq\varnothing\) and \(\ell\) attains its minimum over \(\mathcal X\). Then
\(
D\leq2(w-1).
\)
\end{proposition}

\begin{proof}
The constraint is equivalent to
\(
\sum_{i=1}^w\ell(x_i)\leq w\rho.
\)
Thus it has the additive form of Proposition~\ref{prop:reset_diameter} with
\(
g(x):=\ell(x),
\)
which gives
\(
D\leq2(w-1).
\)
\end{proof}

\paragraph{Non-cyclic ramping constraints.}
Consider
\[
\mathcal C_{\mathrm{ramp}}^{\mathrm{nc}}
=
\left\{
(x_1,\ldots,x_w)\in\mathcal X^w:
\|x_{i+1}-x_i\|\leq\rho,\quad i=1,\ldots,w-1
\right\}.
\]
Assume that \(\mathcal X\) is compact and convex, and define
\(
D_{\mathcal X}:=\max_{x,y\in\mathcal X}\|x-y\|.
\)

\begin{proposition}[Exact Diameter for Non-Cyclic Ramping]
\label{prop:noncyclic_ramping_D}
Suppose \(w\geq2\), \(\rho>0\), and \(D_{\mathcal X}>0\). Then
\[
D=
\left\lceil\frac{D_{\mathcal X}}{\rho}\right\rceil+w-2.
\]
\end{proposition}

\begin{proof}
Fix \(n\geq1\) and
\[
s=(x_1,\ldots,x_{w-1}),
\qquad
s'=(y_1,\ldots,y_{w-1})
\]
in \(\mathcal S_n\). Since \(\mathcal S_n\subseteq\mathcal S_1\), both histories satisfy
\[
\|x_{i+1}-x_i\|\leq\rho,
\qquad
\|y_{i+1}-y_i\|\leq\rho,
\qquad
i=1,\ldots,w-2.
\]

Set
\[
N:=
\left\lceil
\frac{\|x_{w-1}-y_1\|}{\rho}
\right\rceil.
\]
If \(N=0\), then \(x_{w-1}=y_1\). Otherwise define
\[
z_j
:=
x_{w-1}+\frac{j}{N}(y_1-x_{w-1}),
\qquad
j=1,\ldots,N.
\]
By convexity, \(z_j\in\mathcal X\), and
\[
\|z_j-z_{j-1}\|
=
\frac{\|y_1-x_{w-1}\|}{N}
\leq\rho.
\]
Hence \(z_1,\ldots,z_N=y_1\) can be appended feasibly. Appending \(y_2,\ldots,y_{w-1}\) then reaches \(s'\). Therefore,
\[
d_{\mathcal C}(s,s')
\leq
N+w-2
\leq
\left\lceil\frac{D_{\mathcal X}}{\rho}\right\rceil+w-2.
\]
Taking the supremum gives
\[
D\leq
\left\lceil\frac{D_{\mathcal X}}{\rho}\right\rceil+w-2.
\]
For the reverse inequality, compactness gives \(a,b\in\mathcal X\) such that
\(
\|a-b\|=D_{\mathcal X}.
\)
Consider the stationary histories
\[
s_a=(a,\ldots,a),
\qquad
s_b=(b,\ldots,b).
\]
Both belong to every \(\mathcal S_n\). Along any feasible transition from \(s_a\) to \(s_b\), consecutive actions can move by at most \(\rho\). By the triangle inequality, the first occurrence of \(b\) requires at least
\(
\left\lceil\frac{D_{\mathcal X}}{\rho}\right\rceil
\)
appended actions. After the first \(b\), another \(w-2\) copies of \(b\) are required to obtain the complete history \(s_b\). Hence
\[
d_{\mathcal C}(s_a,s_b)
\geq
\left\lceil\frac{D_{\mathcal X}}{\rho}\right\rceil+w-2.
\]
Combining the upper and lower bounds proves the claim.
\end{proof} 

Thus, the ramping diameter separates into movement and memory-replacement costs:
\[
D
=
\left\lceil\frac{D_{\mathcal X}}{\rho}\right\rceil
+
(w-2).
\]

Table~\ref{tab:history_state_diameters} summarizes the resulting history-state diameter bounds.

\begin{table}[t]
\centering
\small
\caption{History-state diameter for representative sliding-window constraints.}
\label{tab:history_state_diameters}
\renewcommand{\arraystretch}{1.35}
\setlength{\tabcolsep}{5pt}
\begin{tabular}{
@{}
>{\centering\arraybackslash}p{0.24\linewidth}
>{\centering\arraybackslash}p{0.48\linewidth}
>{\centering\arraybackslash}p{0.22\linewidth}
@{}
}
\hline
\textbf{Constraint class}
&
\textbf{Window condition}
&
\textbf{History diameter}
\\
\hline

State-independent
&
\(\displaystyle
\mathcal C=\mathcal A^w
\)
&
\(\displaystyle
D=w-1
\)
\\

Box
&
\(\displaystyle
\ell\leq x_i\leq u
\)
&
\(\displaystyle
D=w-1
\)
\\

Maximum
&
\(\displaystyle
\max_i\phi(x_i)\leq\rho
\)
&
\(\displaystyle
D=w-1
\)
\\

Aggregate budget
&
\(\displaystyle
\sum_{i=1}^w c^\top x_i\leq B
\)
&
\(\displaystyle
D\leq2(w-1)
\)
\\

Exponential
&
\(\displaystyle
\sum_{i=1}^w e^{a^\top x_i-b}\leq\rho
\)
&
\(\displaystyle
D\leq2(w-1)
\)
\\

Aggregate \(\ell_p\)
&
\(\displaystyle
\sum_{i=1}^w\|Ax_i-b\|_p^p\leq\rho^p
\)
&
\(\displaystyle
D\leq2(w-1)
\)
\\

Convex risk
&
\(\displaystyle
\sum_{i=1}^w\ell(x_i)\leq w\rho
\)
&
\(\displaystyle
D\leq2(w-1)
\)
\\

Non-cyclic ramping
&
\(\displaystyle
\|x_{i+1}-x_i\|\leq\rho
\)
&
\(\displaystyle
D=
\left\lceil
\frac{D_{\mathcal X}}{\rho}
\right\rceil
+w-2
\)
\\

\hline
\end{tabular}
\end{table}
\section{Proofs for Section~\ref{sec:geometry}: Geometry}
\label{app:geometry_proofs}

\subsection{Proof of Theorem~\ref{thm:stationary_near_optimality}}
\label{app:proof-thm-stationary-near-optimality}

Define
\[
M_w
:=
\max_{(x_1,\ldots,x_w)\in\mathcal C}
\sum_{i=1}^w f(x_i).
\]
Since \(\mathcal X\) is compact, \(\mathcal X^w\) is compact. Because \(\mathcal C\subseteq\mathcal X^w\) is closed and nonempty, it is also compact. Moreover, \((x_1,\ldots,x_w)\mapsto\sum_{i=1}^w f(x_i)\) is upper semicontinuous. Hence the maximum defining \(M_w\) is attained.

\begin{lemma}
\label{lem:cyclic_symmetrization}
For any \(x=(x_1,\ldots,x_w)\in\mathcal C\),
\[
\frac{1}{w}\sum_{k=0}^{w-1}\mathcal S^k x
=
(\bar z,\ldots,\bar z)\in\mathcal C,
\qquad
\bar z:=\frac{1}{w}\sum_{i=1}^w x_i.
\]
\end{lemma}

\begin{proof}
By cyclic-shift invariance,
\[
\mathcal S^k x\in\mathcal C,
\qquad
k=0,\ldots,w-1.
\]
Since \(\mathcal C\) is convex,
\[
\bar x
:=
\frac{1}{w}\sum_{k=0}^{w-1}\mathcal S^k x
\in\mathcal C.
\]
For each coordinate \(j\in\{1,\ldots,w\}\), the \(j\)-th coordinates of
\(
x,\mathcal Sx,\ldots,\mathcal S^{w-1}x
\)
are \(x_1,\ldots,x_w\), each appearing exactly once. Therefore,
\[
\bar x_j
=
\frac{1}{w}\sum_{i=1}^w x_i
=
\bar z.
\]
Thus every coordinate of \(\bar x\) equals \(\bar z\), and hence
\(
\bar x=(\bar z,\ldots,\bar z)\in\mathcal C.
\)
\end{proof}

\begin{lemma}
\label{lem:stationary_window}
There exists \(z^\star\in\mathcal X\) such that
\[
(z^\star,\ldots,z^\star)\in\mathcal C
\qquad\text{and}\qquad
wf(z^\star)=M_w.
\]
\end{lemma}

\begin{proof}
Let
\[
x^\star=(x_1^\star,\ldots,x_w^\star)
\in
\arg\max_{(x_1,\ldots,x_w)\in\mathcal C}
\sum_{i=1}^w f(x_i),
\]
and define
\[
z^\star
:=
\frac{1}{w}\sum_{i=1}^w x_i^\star.
\]
By Lemma~\ref{lem:cyclic_symmetrization},
\(
(z^\star,\ldots,z^\star)\in\mathcal C.
\)
By concavity of \(f\) and Jensen's inequality,
\[
f(z^\star)
=
f\left(\frac{1}{w}\sum_{i=1}^w x_i^\star\right)
\geq
\frac{1}{w}\sum_{i=1}^w f(x_i^\star).
\]
Therefore,
\[
wf(z^\star)
\geq
\sum_{i=1}^w f(x_i^\star)
=
M_w.
\]
On the other hand, since
\((z^\star,\ldots,z^\star)\in\mathcal C\), this stationary window is
feasible for the optimization defining \(M_w\). Hence
\[
wf(z^\star)\leq M_w.
\]
Thus,
\[
wf(z^\star)=M_w.
\]
\end{proof}

\begin{proof}[Proof of Theorem~\ref{thm:stationary_near_optimality}]
Define
\[
\Delta:\mathcal X\to\mathcal X^w,
\qquad
\Delta(z):=(z,\ldots,z).
\]
Then
\(
\mathcal Z=\Delta^{-1}(\mathcal C).
\)
Since \(\mathcal C\) is nonempty, choose any
\(x=(x_1,\ldots,x_w)\in\mathcal C\). By
Lemma~\ref{lem:cyclic_symmetrization}, its cyclic average has the form \((\bar z,\ldots,\bar z)\in\mathcal C\), so \(\bar z\in\mathcal Z\). Hence \(\mathcal Z\) is nonempty.
Moreover, \(\Delta\) is continuous and \(\mathcal C\) is closed, so
\(\mathcal Z=\Delta^{-1}(\mathcal C)\) is closed in the compact set
\(\mathcal X\). Therefore, \(\mathcal Z\) is compact. By Lemma~\ref{lem:stationary_window}, there exists
\(z^\star\in\mathcal Z\) such that
\(
wf(z^\star)=M_w.
\)
For every \(z\in\mathcal Z\), we have
\((z,\ldots,z)\in\mathcal C\), and therefore
\[
wf(z)\leq M_w=wf(z^\star).
\]
Thus,
\[
z^\star\in\arg\max_{z\in\mathcal Z}f(z).
\]
The stationary trajectory \(x_t=z^\star\) is feasible because every
length-\(w\) window is equal to
\((z^\star,\ldots,z^\star)\in\mathcal C\). Hence
\[
V_T(f)\geq Tf(z^\star).
\]
Now write
\[
T=qw+r,
\qquad
q\in\mathbb N_0,
\qquad
0\leq r<w.
\]
Let \(y_{1:T}=(y_1,\ldots,y_T)\) be any feasible trajectory. For each
\(j=1,\ldots,q\), the block
\[
(y_{(j-1)w+1},\ldots,y_{jw})
\]
is one of the realized length-\(w\) windows and therefore belongs to
\(\mathcal C\). By the definition of \(M_w\),
\[
\sum_{i=(j-1)w+1}^{jw}f(y_i)\leq M_w.
\]
Since \(\mathcal X\) is compact and \(f\) is upper semicontinuous,
\(
f_{\max}:=\max_{x\in\mathcal X}f(x)
\)
is well defined and attained. The remaining \(r\) actions therefore
contribute at most \(rf_{\max}\). Consequently,
\[
\sum_{t=1}^T f(y_t)
\leq
qM_w+rf_{\max}.
\]
Taking the maximum over all feasible trajectories and using
\(M_w=wf(z^\star)\), we obtain
\[
V_T(f)
\leq
qwf(z^\star)+rf_{\max}
=
Tf(z^\star)
+
r\bigl(f_{\max}-f(z^\star)\bigr).
\]
Combining this with \(V_T(f)\geq Tf(z^\star)\) gives
\[
0
\leq
V_T(f)-Tf(z^\star)
\leq
r\bigl(f_{\max}-f(z^\star)\bigr).
\]
Since \(z^\star\in\mathcal X\), we have
\(f_{\max}-f(z^\star)\geq 0\). Using \(r\leq w-1\),
\[
0
\leq
V_T(f)-Tf(z^\star)
\leq
r\bigl(f_{\max}-f(z^\star)\bigr)
\leq
(w-1)\bigl(f_{\max}-f(z^\star)\bigr).
\]
Finally, if \(w\mid T\), then \(r=0\), and hence
\[
V_T(f)=Tf(z^\star).
\]
Therefore, the stationary trajectory \(x_t=z^\star\) is optimal.
\end{proof}

\subsection{Tightness of the Stationary Approximation Gap}
\label{app:tightness-offline-gap}

We next show that the $\mathcal{O}(w)$ finite-horizon gap in
Theorem~\ref{thm:stationary_near_optimality} is unavoidable in general. A simple aggregate-budget example shows that, for a suitable choice of the horizon, this gap can scale linearly with $w$. In fact, the upper bound in Theorem~\ref{thm:stationary_near_optimality} is attained exactly.

\begin{proposition}
\label{prop:tightness-offline-gap}
Let $w$ be even, $\mathcal{X}=[0,1]$, and
\[
\mathcal{C}
=
\left\{
(x_1,\ldots,x_w)\in[0,1]^w :
\sum_{i=1}^w x_i \le \frac{w}{2}
\right\}.
\]
For the linear reward $f(x)=x$, the best stationary feasible action is
$z^\star=1/2$. Moreover, for any $q\ge 0$ and \( T=qw+\frac{w}{2}, \) the offline gap is exactly
\[
V_T-Tf(z^\star)=\frac{w}{4}.
\]
Thus, the $\mathcal{O}(w)$ dependence in Theorem~\ref{thm:stationary_near_optimality} is tight in general.
\end{proposition}

\begin{proof}
A stationary action $z\in[0,1]$ is feasible if and only if
\(
wz\le \frac{w}{2},
\)
and hence the stationary feasible set is
\(
\mathcal{Z}=[0,1/2].
\)
Since $f(x)=x$, the optimal stationary feasible action is therefore
\(
z^\star=\frac{1}{2}.
\)
Thus the stationary trajectory has value
\[
Tf(z^\star)
=
\frac{T}{2}
=
\frac{qw}{2}+\frac{w}{4}.
\]
We next construct a feasible nonstationary trajectory. Consider the periodic sequence of period $w$ given by
\[
\underbrace{1,\ldots,1}_{w/2},
\underbrace{0,\ldots,0}_{w/2},
\underbrace{1,\ldots,1}_{w/2},
\underbrace{0,\ldots,0}_{w/2},
\ldots
\]
and truncate it after $T=qw+w/2$ rounds. Every block of $w$ consecutive actions contains exactly $w/2$ entries equal to $1$ and $w/2$ entries equal to $0$. Therefore, for every $t=w,\ldots,T$,
\[
\sum_{i=t-w+1}^{t} x_i=\frac{w}{2},
\]
so every sliding window belongs to $\mathcal{C}$. Each complete period contributes reward $w/2$, and the final $w/2$ actions are all equal to $1$. Hence this feasible trajectory achieves
\[
\sum_{t=1}^T f(x_t)
=
q\frac{w}{2}+\frac{w}{2}.
\]
It follows that
\[
V_T
\ge
q\frac{w}{2}+\frac{w}{2},
\]
and therefore
\[
V_T-Tf(z^\star)
\ge
\left(
q\frac{w}{2}+\frac{w}{2}
\right)
-
\left(
q\frac{w}{2}+\frac{w}{4}
\right)
=
\frac{w}{4}.
\]
To obtain the matching upper bound, apply Theorem~\ref{thm:stationary_near_optimality} with
$r=w/2$. Since
\[
f_{\max}=\max_{x\in[0,1]}f(x)=1
\qquad\text{and}\qquad
f(z^\star)=\frac12,
\]
Theorem~\ref{thm:stationary_near_optimality} gives
\[
V_T-Tf(z^\star)
\le
r\bigl(f_{\max}-f(z^\star)\bigr)
=
\frac{w}{2}
\left(1-\frac12\right)
=
\frac{w}{4}.
\]
Combining the lower and upper bounds yields
\[
V_T-Tf(z^\star)=\frac{w}{4}.
\]
Thus the additive $\mathcal{O}(w)$ finite-horizon term in Theorem~\ref{thm:stationary_near_optimality} is tight.
\end{proof}

\begin{remark}
The gap comes from the horizon ending partway through a period. Over each complete block of length $w$, both the stationary and nonstationary trajectories achieve the same average reward of $1/2$. The difference appears in the final incomplete block: the nonstationary trajectory can use its high-reward actions first, gaining an additional $w/4$ before the horizon ends.
\end{remark}

\section{Proofs for Section~\ref{sec:linear-bandit}: Reachability and Learning}
\label{app:reachability_proofs}

\subsection{Proof of Proposition~\ref{prop:linear-regret-without-transitions}}
\label{proof_prop:linear-regret-without-transitions}

\begin{proof}
Consider the one-dimensional action space \(\mathcal X=[-1,1]\) and let \(w=2\). Define
\[
\mathcal C
=
\left\{
(x_1,x_2)\in[-1,1]^2:x_1=x_2
\right\}.
\]
The set \(\mathcal C\) is nonempty, closed, and convex. It is also invariant under cyclic shifts, since
\[
(x_1,x_2)\in\mathcal C
\quad\Longrightarrow\quad
x_1=x_2
\quad\Longrightarrow\quad
(x_2,x_1)\in\mathcal C.
\]
Thus, Assumption~\ref{ass:cyclic_invariance} holds. The stationary feasible action set is
\[
\mathcal Z
=
\{z\in[-1,1]:(z,z)\in\mathcal C\}
=
[-1,1].
\]
Let \(\theta^\star\in\{-1,1\}\) and consider the noiseless linear reward model
\(
Y_t=\theta^\star x_t.
\)
This is a special case of the conditionally sub-Gaussian model with \(\eta_t=0\). For any admissible trajectory,
\[
(x_{t-1},x_t)\in\mathcal C,
\qquad
t=2,\ldots,T,
\]
which implies
\[
x_t=x_{t-1},
\qquad
t=2,\ldots,T.
\]
Hence
\[
x_t=x_1,
\qquad
t=1,\ldots,T.
\]
Thus, once the learner chooses its first action, feasibility forces it to repeat that action for the entire horizon. For \(\theta^\star=1\), the optimal stationary feasible action is \(z^\star=1\), and therefore
\[
\operatorname{Reg}^{\mathcal Z}_{\pi}(T;1)
=
T\left(1-\mathbb E_\pi[x_1]\right).
\]
For \(\theta^\star=-1\), the optimal stationary feasible action is \(z^\star=-1\), giving
\[
\operatorname{Reg}^{\mathcal Z}_{\pi}(T;-1)
=
T\left(1+\mathbb E_\pi[x_1]\right).
\]
Averaging the two regrets yields
\[
\frac{1}{2}
\operatorname{Reg}^{\mathcal Z}_{\pi}(T;1)
+
\frac{1}{2}
\operatorname{Reg}^{\mathcal Z}_{\pi}(T;-1)
=
T.
\]
Therefore,
\[
\sup_{\theta^\star\in\{-1,1\}}
\operatorname{Reg}^{\mathcal Z}_{\pi}(T;\theta^\star)
\geq T.
\]
Thus, every admissible policy suffers linear regret for at least one value of \(\theta^\star\).
\end{proof}

%%%%%%%%%%%%%%%%%%%%%%%%%%%%%%%%%%%%%%%%%%%%%%%%%%%%%%%%%%%%%%%%%%%%%%%%%%%%%%%

\subsection{Proof of Theorem~\ref{thm:expected-offline-regret}}
\label{proof_thm:expected-offline-regret}

We first establish several auxiliary lemmas controlling the confidence sets, covariance growth within an episode, the elliptical potential, the number of switches, and feasibility. We then derive a high-probability stationary-regret bound and convert it to expected regret against the offline-optimal feasible trajectory.

%%%%%%%%%%%%%%%%%%%%%%%%%%%%%%%%%%%%%%%%%%%%%%%%%%%%%%%%%%%%%%%%%%%%%%%%%%%%%%%

\begin{lemma}
\label{lem:uniform-confidence}
For \(t\geq1\), define
\[
\mathcal E_t(\delta)
:=
\left\{
\theta\in\mathbb R^d:
\|\theta-\widehat\theta_t\|_{V_t}
\leq
\beta_t(\delta)
\right\},
\]
where
\[
V_t
=
\lambda I_d
+
\sum_{s\in\mathcal O_t}x_sx_s^\top,
\qquad
\widehat\theta_t
=
V_t^{-1}
\sum_{s\in\mathcal O_t}x_sY_s,
\]
and
\[
\beta_t(\delta)
=
R
\sqrt{
2\log\!\left(
\frac{\det(V_t)^{1/2}}
{\det(\lambda I_d)^{1/2}\delta}
\right)
}
+
\sqrt{\lambda}S.
\]
Then, with probability at least \(1-\delta\),
\[
\theta^\star\in\mathcal E_t(\delta),
\qquad
t=1,\ldots,T+1.
\]
Moreover,
\[
\beta_t(\delta)
\leq
\overline\beta_T(\delta),
\qquad
t\leq T+1,
\]
where
\[
\overline\beta_T(\delta)
:=
R
\sqrt{
d\log\!\left(
1+\frac{TL^2}{\lambda d}
\right)
+
2\log\!\left(\frac{1}{\delta}\right)
}
+
\sqrt{\lambda}S.
\]
\end{lemma}

\begin{proof}
For \(s\geq1\), define
\(
\widetilde x_s:=I_sx_s.
\)
Since \(I_s\) and \(x_s\) are \(\mathcal F_{s-1}\)-measurable,
\(\widetilde x_s\) is predictable. Moreover, because \(I_s\in\{0,1\}\),
\[
V_t
=
\lambda I_d
+
\sum_{s=1}^{t-1}
\widetilde x_s\widetilde x_s^\top.
\]

By the conditional \(R\)-sub-Gaussian noise assumption and the
self-normalized concentration inequality for vector-valued martingales
\citep[Theorem~1]{abbasi2011improved}, with probability at least
\(1-\delta\), simultaneously for all \(t\leq T+1\),
\[
\left\|
\sum_{s=1}^{t-1}\widetilde x_s\eta_s
\right\|_{V_t^{-1}}
\leq
R\sqrt{
2\log\!\left(
\frac{\det(V_t)^{1/2}}
{\det(\lambda I_d)^{1/2}\delta}
\right)
}.
\]
Since
\[
\sum_{s=1}^{t-1}\widetilde x_s\eta_s
=
\sum_{s\in\mathcal O_t}x_s\eta_s,
\]
the same bound holds for the retained observations. Using
\(
Y_s=\langle\theta^\star,x_s\rangle+\eta_s,
\)
we obtain
\[
\begin{aligned}
V_t\widehat\theta_t
&=
\sum_{s\in\mathcal O_t}x_sY_s
=
\sum_{s\in\mathcal O_t}
x_sx_s^\top\theta^\star
+
\sum_{s\in\mathcal O_t}x_s\eta_s.
\end{aligned}
\]
Since
\[
V_t\theta^\star
=
\lambda\theta^\star
+
\sum_{s\in\mathcal O_t}
x_sx_s^\top\theta^\star,
\]
it follows that
\[
V_t(\widehat\theta_t-\theta^\star)
=
\sum_{s\in\mathcal O_t}x_s\eta_s
-
\lambda\theta^\star.
\]
Hence
\[
\begin{aligned}
\|\widehat\theta_t-\theta^\star\|_{V_t}
&=
\left\|
\sum_{s\in\mathcal O_t}x_s\eta_s
-
\lambda\theta^\star
\right\|_{V_t^{-1}}
\leq
\left\|
\sum_{s\in\mathcal O_t}x_s\eta_s
\right\|_{V_t^{-1}}
+
\lambda\|\theta^\star\|_{V_t^{-1}}.
\end{aligned}
\]
Because \(V_t\succeq \lambda I_d\), and inversion reverses the Loewner order for positive definite matrices,
\[
V_t^{-1}\preceq (\lambda I_d)^{-1}=\lambda^{-1}I_d.
\]
we have
\[
\lambda\|\theta^\star\|_{V_t^{-1}}
\leq
\sqrt{\lambda}\|\theta^\star\|_2
\leq
\sqrt{\lambda}S.
\]
Therefore, simultaneously for all \(t\leq T+1\),
\[
\|\widehat\theta_t-\theta^\star\|_{V_t}
\leq
R\sqrt{
2\log\!\left(
\frac{\det(V_t)^{1/2}}
{\det(\lambda I_d)^{1/2}\delta}
\right)
}
+
\sqrt{\lambda}S
=
\beta_t(\delta).
\]
Thus, with probability at least \(1-\delta\),
\[
\theta^\star\in\mathcal E_t(\delta),
\qquad
\forall\,t\leq T+1.
\]

It remains to upper bound \(\beta_t(\delta)\). Let
\(
N_t:=|\mathcal O_t|.
\)
Since \(N_t\leq T\) and \(\|x_s\|_2\leq L\),
\[
\operatorname{tr}(V_t)
=
\lambda d
+
\sum_{s\in\mathcal O_t}\|x_s\|_2^2
\leq
\lambda d+N_tL^2.
\]
Let \(\lambda_1,\ldots,\lambda_d\) denote the eigenvalues of
\(V_t\). By the arithmetic--geometric mean inequality,
\[
\begin{aligned}
\det(V_t)
&=
\prod_{i=1}^d\lambda_i
\leq
\left(
\frac{\operatorname{tr}(V_t)}{d}
\right)^d
\leq
\left(
\lambda+\frac{N_tL^2}{d}
\right)^d.
\end{aligned}
\]
Since \( \det(\lambda I_d)=\lambda^d, \)
we obtain
\[
\frac{\det(V_t)}{\det(\lambda I_d)}
\leq
\left(
1+\frac{N_tL^2}{\lambda d}
\right)^d
\leq
\left(
1+\frac{TL^2}{\lambda d}
\right)^d.
\]
Substituting this into the definition of \(\beta_t(\delta)\) yields
\[
\beta_t(\delta)
\leq
R\sqrt{
d\log\!\left(
1+\frac{TL^2}{\lambda d}
\right)
+
2\log\!\left(\frac1\delta\right)
}
+
\sqrt{\lambda}S
=
\overline\beta_T(\delta).
\]
This proves both claims.
\end{proof}

%%%%%%%%%%%%%%%%%%%%%%%%%%%%%%%%%%%%%%%%%%%%%%%%%%%%%%%%%%%%%%%%%%%%%%%%%%%%%%%

\begin{lemma}
\label{lem:determinant-comparison}
Suppose
\(
0\prec A\preceq B
\)
and
\[
\det(B)\leq c\det(A),
\qquad
c\geq1.
\]
Then
\[
B\preceq cA,
\qquad
A^{-1}\preceq cB^{-1}.
\]
Consequently, for every \(x\in\mathbb R^d\),
\[
\|x\|_{A^{-1}}
\leq
\sqrt c\,\|x\|_{B^{-1}}.
\]
\end{lemma}

\begin{proof}
Define
\(
M:=A^{-1/2}BA^{-1/2}.
\)
Since \(A\preceq B\),
\[
M\succeq I_d.
\]
Hence every eigenvalue \(\mu_i(M)\) satisfies \(\mu_i(M)\geq1\). Moreover,
\[
\prod_{i=1}^d\mu_i(M)
=
\det(M)
=
\frac{\det(B)}{\det(A)}
\leq c.
\]
Since all eigenvalues are at least one, each satisfies \(\mu_i(M)\leq c\). Thus
\(
M\preceq cI_d,
\)
which implies
\(
B\preceq cA.
\)
Taking inverses gives
\(
A^{-1}\preceq cB^{-1}.
\)
Therefore,
\[
\|x\|_{A^{-1}}^2
\leq
c\|x\|_{B^{-1}}^2,
\]
and the result follows by taking square roots.
\end{proof}

%%%%%%%%%%%%%%%%%%%%%%%%%%%%%%%%%%%%%%%%%%%%%%%%%%%%%%%%%%%%%%%%%%%%%%%%%%%%%%%

\begin{lemma}
\label{lem:elliptical-potential}
Let
\[
V_1=\lambda I_d,
\qquad
V_{n+1}=V_n+u_nu_n^\top,
\]
where
\[
\|u_n\|_2\leq L,
\qquad
\lambda\geq L^2.
\]
Then
\[
\sum_{n=1}^N
\|u_n\|_{V_n^{-1}}
\leq
\sqrt{
2Nd
\log\!\left(
1+\frac{NL^2}{\lambda d}
\right)
}.
\]
\end{lemma}

\begin{proof}
Since \(V_n\succeq \lambda I_d\), and inversion reverses the Loewner order for positive definite matrices, we obtain
\[
V_n^{-1}\preceq (\lambda I_d)^{-1}=\lambda^{-1}I_d.
\]
Therefore,
\[
\|u_n\|_{V_n^{-1}}^2
\leq
\frac{\|u_n\|_2^2}{\lambda}
\leq
1.
\]
Define
\(
q_n:=\|u_n\|_{V_n^{-1}}^2.
\)
By the matrix determinant lemma,
\[
\det(V_{n+1})
=
\det(V_n)(1+q_n).
\]
Since \(q_n\in[0,1]\),
\[
q_n\leq2\log(1+q_n).
\]
Therefore,
\[
\begin{aligned}
\sum_{n=1}^N q_n
&\leq
2\sum_{n=1}^N\log(1+q_n)
=
2\log
\frac{\det(V_{N+1})}{\det(V_1)}.
\end{aligned}
\]
As in Lemma~\ref{lem:uniform-confidence},
\[
\frac{\det(V_{N+1})}{\det(V_1)}
\leq
\left(
1+\frac{NL^2}{\lambda d}
\right)^d.
\]
Hence
\[
\sum_{n=1}^N
\|u_n\|_{V_n^{-1}}^2
\leq
2d
\log\!\left(
1+\frac{NL^2}{\lambda d}
\right).
\]
Finally, Cauchy--Schwarz gives
\[
\begin{aligned}
\sum_{n=1}^N
\|u_n\|_{V_n^{-1}}
&\leq
\sqrt{
N
\sum_{n=1}^N
\|u_n\|_{V_n^{-1}}^2
}
\leq
\sqrt{
2Nd
\log\!\left(
1+\frac{NL^2}{\lambda d}
\right)
}.
\end{aligned}
\]
\end{proof}

%%%%%%%%%%%%%%%%%%%%%%%%%%%%%%%%%%%%%%%%%%%%%%%%%%%%%%%%%%%%%%%%%%%%%%%%%%%%%%%

\begin{lemma}
\label{lem:number-switches}
Let \(K_T\) denote the number of determinant-doubling episodes and \(J_T\) the number of actual stationary-action switches. Then
\[
K_T
\leq
1+
\left\lceil
d\log_2\!\left(
1+\frac{TL^2}{\lambda d}
\right)
\right\rceil
\qquad\text{and}\qquad
J_T
\leq
\left\lceil
d\log_2\!\left(
1+\frac{TL^2}{\lambda d}
\right)
\right\rceil.
\]
\end{lemma}

\begin{proof}
Let \(W_k\) denote the design matrix at the beginning of episode \(k\).
For every \(k=1,\ldots,K_T-1\), episode \(k+1\) begins only after the
determinant has more than doubled. Hence,
\(
\det(W_{k+1})
>
2\det(W_k).
\)
Iterating this inequality gives
\[
\det(W_{K_T})
>
2^{K_T-1}\det(W_1).
\]
Since
\(
W_1=\lambda I_d
\)
and
\(
W_{K_T}\preceq V_{T+1},
\)
we have
\[
2^{K_T-1}
<
\frac{\det(V_{T+1})}{\det(\lambda I_d)}.
\]
By the determinant bound from
Lemma~\ref{lem:uniform-confidence},
\[
\frac{\det(V_{T+1})}{\det(\lambda I_d)}
\leq
\left(
1+\frac{TL^2}{\lambda d}
\right)^d.
\]
Therefore,
\[
K_T-1
<
d\log_2\!\left(
1+\frac{TL^2}{\lambda d}
\right),
\]
and hence
\[
K_T
\leq
1+
\left\lceil
d\log_2\!\left(
1+\frac{TL^2}{\lambda d}
\right)
\right\rceil.
\]

Finally,
\[
J_T
=
\sum_{k=1}^{K_T-1}
\mathbf 1\{z_{k+1}\neq z_k\}
\leq
K_T-1,
\]
so
\[
J_T
\leq
\left\lceil
d\log_2\!\left(
1+\frac{TL^2}{\lambda d}
\right)
\right\rceil.
\]
\end{proof}

%%%%%%%%%%%%%%%%%%%%%%%%%%%%%%%%%%%%%%%%%%%%%%%%%%%%%%%%%%%%%%%%%%%%%%%%%%%%%%%

\begin{lemma}
\label{lem:algorithm-feasibility}
Under Assumption~\ref{ass:bounded-transitions}, the rare-switching OFUL policy is feasible at every round.
\end{lemma}

\begin{proof}
When \(w=1\), feasibility is pointwise and every selected stationary action belongs to \(\mathcal Z\), so the claim is immediate. Hence suppose \(w\geq2\). The first selected action satisfies \(z_1\in\mathcal Z\), so
\(
(z_1,\ldots,z_1)\in\mathcal C.
\)
The learner repeats \(z_1\) for at least \(w-1\) rounds before attempting any transition. Thus, whenever a first transition is attempted, the current history is
\(
h(z_1)
=
(z_1,\ldots,z_1).
\)
Suppose inductively that immediately before a change from episode \(k\) to episode \(k+1\), the current history is
\(
h(z_k).
\)
If \(z_{k+1}=z_k\), no transition is needed, and continuing to play \(z_k\) produces only the feasible stationary window
\(
(z_k,\ldots,z_k)\in\mathcal C.
\)
If \(z_{k+1}\neq z_k\), Assumption~\ref{ass:bounded-transitions} and the transition oracle provide actions
\[
(a_1,\ldots,a_m),
\qquad
m\leq\tau,
\]
such that, starting from \(h_0=h(z_k)\),
\[
h_j=F(h_{j-1},a_j),
\qquad
(h_{j-1},a_j)\in\mathcal C,
\qquad
j=1,\ldots,m,
\]
and
\(
h_m=h(z_{k+1}).
\)
Thus every window encountered during the transition is feasible. After the transition, repeatedly playing \(z_{k+1}\) also preserves feasibility. The result follows by induction over the episodes.
\end{proof}

%%%%%%%%%%%%%%%%%%%%%%%%%%%%%%%%%%%%%%%%%%%%%%%%%%%%%%%%%%%%%%%%%%%%%%%%%%%%%%%

\begin{lemma}
\label{lem:high-prob-stationary-regret}
With probability at least \(1-\delta\),
\[
\begin{aligned}
\mathfrak R_T^{\mathcal Z}
:=
\sum_{t=1}^T
\langle\theta^\star,z^\star-x_t\rangle
\leq\;&
4\overline\beta_T(\delta)
\sqrt{
Td
\log\!\left(
1+\frac{TL^2}{\lambda d}
\right)
}
\\
&+
2LS(w-1)
\\
&+
2LS\tau
\left\lceil
d\log_2\!\left(
1+\frac{TL^2}{\lambda d}
\right)
\right\rceil.
\end{aligned}
\]
\end{lemma}

\begin{proof}
Define the uniform confidence event
\[
\mathcal G_\delta
:=
\left\{
\theta^\star\in\mathcal E_t(\delta),
\quad
t=1,\ldots,T+1
\right\}.
\]
By Lemma~\ref{lem:uniform-confidence},
\(
\mathbb P(\mathcal G_\delta)\geq1-\delta.
\)
We condition on \(\mathcal G_\delta\). Partition the rounds into retained stationary rounds \(\mathcal O\) and transition rounds \(\mathcal Q\). Then
\[
\mathfrak R_T^{\mathcal Z}
=
\mathfrak R_T^{\mathrm{stat}}
+
\mathfrak R_T^{\mathrm{trans}}.
\]

\paragraph{Stationary rounds.}
Consider \(t\in\mathcal O\) belonging to episode \(k\). Then \(x_t=z_k\). Let
\(
W_k:=V_{t_k}.
\)
By optimism,
\[
z_k
\in
\arg\max_{z\in\mathcal Z}
\left\{
\langle\widehat\theta_{t_k},z\rangle
+
\beta_{t_k}(\delta)
\|z\|_{W_k^{-1}}
\right\}.
\]
Since \(z^\star\in\mathcal Z\) and \(\theta^\star\in\mathcal E_{t_k}(\delta)\),
\[
\langle\theta^\star,z^\star\rangle
\leq
\langle\widehat\theta_{t_k},z_k\rangle
+
\beta_{t_k}(\delta)
\|z_k\|_{W_k^{-1}}.
\]
Therefore,
\[
\begin{aligned}
\langle\theta^\star,z^\star-z_k\rangle
&\leq
\langle\widehat\theta_{t_k}-\theta^\star,z_k\rangle
+
\beta_{t_k}(\delta)
\|z_k\|_{W_k^{-1}}
\leq
2\beta_{t_k}(\delta)
\|z_k\|_{W_k^{-1}}
\leq
2\overline\beta_T(\delta)
\|z_k\|_{W_k^{-1}}.
\end{aligned}
\]
During the initial warm-up, the algorithm may continue playing \(z_1\) after the determinant-doubling threshold has already been crossed in order to establish the stationary history. Let
\[
\mathcal H
:=
\left\{
t\in\mathcal O:
t\text{ belongs to episode }1
\text{ and }
\det(V_t)>2\det(W_1)
\right\}.
\]
The forced warm-up lasts for at most \(w-1\) rounds, so
\(
|\mathcal H|\leq w-1.
\)
For every \(t\in\mathcal H\),
\[
\langle\theta^\star,z^\star-x_t\rangle
\leq
\|\theta^\star\|_2
\|z^\star-x_t\|_2
\leq
2LS.
\]
Hence
\[
\sum_{t\in\mathcal H}
\langle\theta^\star,z^\star-x_t\rangle
\leq
2LS(w-1).
\]
Now consider \(t\in\mathcal O\setminus\mathcal H\) in episode \(k\). Before the determinant-doubling threshold is crossed,
\[
W_k\preceq V_t,
\qquad
\det(V_t)\leq2\det(W_k).
\]
Lemma~\ref{lem:determinant-comparison} therefore gives
\(
W_k^{-1}\preceq2V_t^{-1},
\)
and consequently
\(
\|z_k\|_{W_k^{-1}}
\leq
\sqrt2\,\|z_k\|_{V_t^{-1}}.
\)
Thus,
\[
\langle\theta^\star,z^\star-z_k\rangle
\leq
2\sqrt2\,
\overline\beta_T(\delta)
\|z_k\|_{V_t^{-1}}.
\]
Let \(N:=|\mathcal O|\leq T\) and enumerate the stationary rounds as
\(
t_1<\cdots<t_N.
\)
Set
\[
u_n:=x_{t_n},
\qquad
\widetilde V_n
:=
\lambda I_d
+
\sum_{j=1}^{n-1}u_ju_j^\top.
\]
Since transition observations are excluded from the estimator,
\[
\widetilde V_n=V_{t_n}.
\]
Lemma~\ref{lem:elliptical-potential} gives
\[
\sum_{t\in\mathcal O}
\|x_t\|_{V_t^{-1}}
\leq
\sqrt{
2Td
\log\!\left(
1+\frac{TL^2}{\lambda d}
\right)
}.
\]
Therefore,
\[
\begin{aligned}
\mathfrak R_T^{\mathrm{stat}}
&\leq
2\sqrt2\,
\overline\beta_T(\delta)
\sum_{t\in\mathcal O}
\|x_t\|_{V_t^{-1}}
+
2LS(w-1)
\leq
4\overline\beta_T(\delta)
\sqrt{
Td
\log\!\left(
1+\frac{TL^2}{\lambda d}
\right)
}
+
2LS(w-1).
\end{aligned}
\]

\paragraph{Transition rounds.}
For every \(t\in\mathcal Q\),
\[
\langle\theta^\star,z^\star-x_t\rangle
\leq
2LS.
\]
Each actual switch requires at most \(\tau\) transition rounds, so
\(
|\mathcal Q|
\leq
\tau J_T.
\)
By Lemma~\ref{lem:number-switches},
\[
J_T
\leq
\left\lceil
d\log_2\!\left(
1+\frac{TL^2}{\lambda d}
\right)
\right\rceil.
\]
Hence
\[
\mathfrak R_T^{\mathrm{trans}}
\leq
2LS\tau
\left\lceil
d\log_2\!\left(
1+\frac{TL^2}{\lambda d}
\right)
\right\rceil.
\]
Combining the stationary and transition contributions proves the claim.
\end{proof}

%%%%%%%%%%%%%%%%%%%%%%%%%%%%%%%%%%%%%%%%%%%%%%%%%%%%%%%%%%%%%%%%%%%%%%%%%%%%%%%

\begin{proof}[Proof of Theorem~\ref{thm:expected-offline-regret}]
By Lemma~\ref{lem:algorithm-feasibility}, the rare-switching OFUL policy is feasible at every round and is therefore admissible. Choose
\(
\delta=\frac{1}{T}.
\)
By Lemma~\ref{lem:high-prob-stationary-regret}, on an event of probability at least \(1-1/T\),
\[
\mathfrak R_T^{\mathcal Z}
=
\widetilde{\mathcal O}
\left(
d\sqrt T+\tau d+w
\right).
\]
Indeed,
\[
\overline\beta_T(1/T)
=
\widetilde{\mathcal O}(\sqrt d),
\]
so the statistical term is \(\widetilde{\mathcal O}(d\sqrt T)\), the transition term is \(\widetilde{\mathcal O}(\tau d)\), and the warm-up term is \(\mathcal{O}(w)\). On the complement of the confidence event,
\[
\mathfrak R_T^{\mathcal Z}
=
\sum_{t=1}^T
\langle\theta^\star,z^\star-x_t\rangle
\leq
2LST.
\]
Since the complement has probability at most \(1/T\), its contribution to the expectation is at most \(2LS\). Therefore,
\[
\operatorname{Reg}^{\mathcal Z}_{\pi}(T;\theta^\star)
=
\mathbb E_\pi
\left[
\mathfrak R_T^{\mathcal Z}
\right]
=
\widetilde{\mathcal O}
\left(
d\sqrt T+\tau d+w
\right).
\]
It remains to compare the stationary benchmark with the offline-optimal feasible trajectory. Writing
\[
T=qw+r,
\qquad
0\leq r<w,
\]
Theorem~\ref{thm:stationary_near_optimality} applied to
\[
f_{\theta^\star}(x)
=
\langle\theta^\star,x\rangle
\]
gives
\[
0
\leq
\Delta_T
:=
V_T^{\mathrm{off}}(\theta^\star)
-
T\langle\theta^\star,z^\star\rangle
\leq
r
\left(
\max_{x\in\mathcal X}
\langle\theta^\star,x\rangle
-
\langle\theta^\star,z^\star\rangle
\right).
\]
Since
\[
\|x\|_2\leq L,
\qquad
\|z^\star\|_2\leq L,
\qquad
\|\theta^\star\|_2\leq S,
\]
we have
\[
\max_{x\in\mathcal X}
\langle\theta^\star,x\rangle
-
\langle\theta^\star,z^\star\rangle
\leq
2LS.
\]
Hence
\[
0\leq\Delta_T\leq2LSr\leq2LS(w-1).
\]
Finally,
\[
\operatorname{Reg}^{\mathrm{off}}_{\pi}(T;\theta^\star)
=
\operatorname{Reg}^{\mathcal Z}_{\pi}(T;\theta^\star)
+
\Delta_T.
\]
Combining the two bounds yields
\[
\operatorname{Reg}^{\mathrm{off}}_{\pi}(T;\theta^\star)
=
\widetilde{\mathcal O}
\left(
d\sqrt T+\tau d+w
\right).
\]
This completes the proof.
\end{proof}

\subsection{A Lower Bound for the Transition Diameter}
\label{app:tau-lower-bound}

We show that the transition diameter can have a direct impact on regret. Even in a one-dimensional setting, and even if the learner is told the true reward parameter immediately after its first action, the sliding-window constraint can still force an additional regret of order $\tau$.

\begin{proposition}
\label{prop:tau-lower-bound}
For every integer $\tau \ge 2$ and every horizon $T \ge \tau$, there exists a one-dimensional linear-bandit instance satisfying Assumptions~\ref{ass:cyclic_invariance}, \ref{ass:concave_reward}, and~\ref{ass:bounded-transitions}, with transition diameter exactly $\tau$, such that every admissible policy $\pi$ satisfies
\[
    \sup_{\theta^\star \in \{-1,+1\}}
    \operatorname{Reg}^{\mathrm{off}}_{\pi}(T;\theta^\star)
    \ge \frac{\tau}{4}.
\]
\end{proposition}

\begin{proof}
Fix an integer $\tau \ge 2$. Let
\[
    \mathcal{X}=[-1,1],
    \qquad
    w=2,
    \qquad
    \rho:=\frac{2}{\tau},
\]
and define the window-feasibility set
\[
    \mathcal{C}_{\rho}
    :=
    \left\{
        (x_1,x_2)\in[-1,1]^2:
        |x_2-x_1|\le \rho
    \right\}.
\]
Thus, every admissible trajectory satisfies
\[
    |x_t-x_{t-1}|\le \rho,
    \qquad
    t=2,\ldots,T.
\]

The action space $\mathcal{X}$ is nonempty, compact, and convex. The set $\mathcal{C}_{\rho}$ is nonempty, closed, and convex. It is also invariant under the cyclic shift $(x_1,x_2)\mapsto(x_2,x_1)$, since
\(
    |x_2-x_1|=|x_1-x_2|.
\)
Thus, the required convexity and cyclic-invariance assumptions are satisfied. The stationary feasible set is
\(
    Z
    :=
    \left\{
        z\in[-1,1]:(z,z)\in\mathcal{C}_{\rho}
    \right\}.
\)
Since $|z-z|=0$ for every $z\in[-1,1]$,
\(
    Z=[-1,1].
\)
Since $w=2$, a stationary history consists of a single action, so $h(z)=z$. For any $z,z'\in Z$, each feasible transition can change the action by at most $\rho$. Therefore, any transition from $z$ to $z'$ using $m$ actions must satisfy
\(
|z'-z|\le m\rho,
\)
which implies
\[
m\ge \left\lceil \frac{|z'-z|}{\rho}\right\rceil.
\]
This bound is also achievable by moving from $z$ to $z'$ in equally spaced steps of size at most $\rho$. Hence,
\[
m(z,z')=\left\lceil \frac{|z'-z|}{\rho}\right\rceil.
\]
Since $Z=[-1,1]$ has diameter $2$ and $\rho=2/\tau$,
\[
\max_{z,z'\in Z} m(z,z')
=
\left\lceil \frac{2}{\rho}\right\rceil
=
\left\lceil \frac{2}{2/\tau}\right\rceil
=
\tau.
\]
Thus Assumption~\ref{ass:bounded-transitions} holds, and the transition diameter is exactly $\tau$. Let $\theta^\star\in\{-1,+1\}$ and consider the noiseless linear reward model
\(
Y_t=\theta^\star x_t.
\)
This is a special case of the conditionally sub-Gaussian noise model. If $\theta^\star=+1$, the optimal stationary action is $z^\star=1$, while if $\theta^\star=-1$, it is $z^\star=-1$. Thus, in both cases,
\[
z^\star=\theta^\star
\qquad\text{and}\qquad
\theta^\star z^\star=1.
\]
The constant trajectory $x_t=\theta^\star$ is feasible for all $t$ and achieves total reward $T$. Since the per-round reward cannot exceed $1$, we have
\(
V_T(\theta^\star)=T.
\)
Therefore, the stationary and offline benchmarks coincide on this instance:
\[
\operatorname{Reg}^{\mathrm{off}}_{\pi}(T;\theta^\star)
=
\operatorname{Reg}^{Z}_{\pi}(T;\theta^\star).
\]
To isolate the transition cost, suppose the learner is told the true value of $\theta^\star$ immediately after choosing its first action $x_1$. Hence, from round $2$ onward, there is no statistical uncertainty. Fix $x_1=a\in[-1,1]$ and define $u:=\theta^\star a$. Since $\theta^\star\in\{-1,+1\}$ and $|x_{t+1}-x_t|\le\rho$,
\[
\theta^\star(x_{t+1}-x_t)\le |x_{t+1}-x_t|\le\rho.
\]
Therefore, after $k$ steps,
\[
\theta^\star x_{1+k}
=\theta^\star x_1+\sum_{j=1}^{k}\theta^\star(x_{j+1}-x_j)
\le u+k\rho.
\]
Since $\theta^\star x_t\le1$, we obtain
\[
\theta^\star x_{1+k}\le\min\{1,u+k\rho\}.
\]
Thus, the learner's reward can improve by at most $\rho$ per round until it reaches the optimal value $1$. Consequently,
\[
1-\theta^\star x_{1+k}
\ge [1-u-k\rho]_+,
\]
where $[y]_+:=\max\{y,0\}$. Summing over the horizon gives
\[
\begin{aligned}
R_T(\theta^\star)
&:=\sum_{t=1}^T\theta^\star(z^\star-x_t)=\sum_{t=1}^T(1-\theta^\star x_t)\ge\sum_{k=0}^{T-1}[1-u-k\rho]_+.
\end{aligned}
\]
Since $u\ge -1$ and $\rho=2/\tau$, we have
\[
\frac{1-u}{\rho}\le \frac{2}{\rho}=\tau.
\]
Because $T\ge\tau$, the regret bound gives
\[
R_T(\theta^\star)\ge \sum_{k=0}^{\tau-1}[1-u-k\rho]_+.
\]
The function $s\mapsto[1-u-\rho s]_+$ is nonnegative and nonincreasing, so the sum is lower bounded by its integral:
\[
\begin{aligned}
R_T(\theta^\star)
&\ge \int_0^\tau [1-u-\rho s]_+\,ds=\int_0^{(1-u)/\rho}(1-u-\rho s)\,ds=\frac{(1-u)^2}{2\rho}.
\end{aligned}
\]
Hence,
\[
R_T(\theta^\star)\ge \frac{(1-u)^2}{2\rho}.
\]
If $\theta^\star=+1$, then $u=a$, while if $\theta^\star=-1$, then $u=-a$.
\[
R_T(+1)\ge \frac{(1-a)^2}{2\rho},
\qquad
R_T(-1)\ge \frac{(1+a)^2}{2\rho}.
\]
Averaging the two bounds gives
\[
\begin{aligned}
\frac{R_T(+1)+R_T(-1)}{2}
&\ge
\frac{(1-a)^2+(1+a)^2}{4\rho}
=
\frac{1+a^2}{2\rho}
\ge
\frac{1}{2\rho}.
\end{aligned}
\]
Since $\rho=2/\tau$,
\[
\frac{1}{2\rho}=\frac{\tau}{4}.
\]
Therefore,
\[
\frac{R_T(+1)+R_T(-1)}{2}\ge \frac{\tau}{4},
\]
and hence at least one of the two reward parameters must satisfy
\[
\max_{\theta^\star\in\{-1,+1\}}R_T(\theta^\star)\ge \frac{\tau}{4}.
\]
Finally, suppose the policy randomizes its first action $X_1$. Since $X_1$ is chosen before any reward is observed, its distribution is independent of $\theta^\star$. Conditioning on any realization $X_1=a$ gives
\[
\frac{1}{2}\mathbb{E}\!\left[R_T(+1)\mid X_1=a\right]
+
\frac{1}{2}\mathbb{E}\!\left[R_T(-1)\mid X_1=a\right]
\ge \frac{\tau}{4}.
\]
Taking expectation over $X_1$,
\[
\frac{1}{2}\mathbb{E}[R_T(+1)]
+
\frac{1}{2}\mathbb{E}[R_T(-1)]
\ge \frac{\tau}{4}.
\]
Hence at least one of the two reward parameters satisfies
\[
\sup_{\theta^\star\in\{-1,+1\}}
\operatorname{Reg}^{Z}_{\pi}(T;\theta^\star)
\ge \frac{\tau}{4}.
\]
Since the stationary and offline benchmarks coincide on this instance,
\[
\sup_{\theta^\star\in\{-1,+1\}}
\operatorname{Reg}^{\mathrm{off}}_{\pi}(T;\theta^\star)
\ge \frac{\tau}{4}.
\]
This completes the proof.
\end{proof}

\begin{remark}[Statistical, transition, and finite-horizon costs]
\label{rem:statistical-transition-lower-bound}
Proposition~\ref{prop:tau-lower-bound} isolates the cost of feasible transitions from the statistical difficulty of learning: the $\Omega(\tau)$ lower bound holds even when the learner is told the true reward parameter immediately after its first action. Standard stochastic linear-bandit lower bounds give a statistical cost of order $d\sqrt{T}$; see \citet[Theorem~24.2]{lattimore2020bandit}, while Proposition~\ref{prop:tightness-offline-gap} shows that the finite-horizon gap can be $\Omega(w)$. Taken together, these results identify the lower-bound scales
\(
\Omega\!\left(d\sqrt{T}+\tau+w\right).
\)
Here, the $\Omega(w)$ term comes from a separate finite-horizon construction, so this expression summarizes the three individual lower bounds rather than a single joint minimax result. The remaining gap is in the transition term: our upper bound is $\widetilde{\mathcal{O}}(\tau d)$, whereas the lower bound is $\Omega(\tau)$. Whether the additional factor of $d$ is necessary remains open.
\end{remark}

\section{Proofs for Section~\ref{sec:online_learning_without_cyclic}: Planning}
\label{app:planning_proofs}

\subsection{Proof of Theorem~\ref{thm:bellman_optimality}}
\label{app:thm_bellman_optimality}

\begin{lemma}
\label{lem:finite_memory_equivalence}
Let
\(
s_t=(x_{t-w+1},\ldots,x_{t-1}),
\)
Then a sequence \(x_{1:T}\) is feasible under the original sliding-window
constraints if and only if
\[
x_t\in A(s_t),
\qquad
t=w,\ldots,T.
\]
\end{lemma}

\begin{proof}
Fix \(t\in\{w,\ldots,T\}\). By definition of the state-dependent feasible
set,
\[
x_t\in A(s_t)
\quad\Longleftrightarrow\quad
(s_t,x_t)\in\mathcal C.
\]
Since
\(
s_t=(x_{t-w+1},\ldots,x_{t-1}),
\)
we have
\(
(s_t,x_t)
=
(x_{t-w+1},\ldots,x_t).
\)
Hence
\[
x_t\in A(s_t)
\quad\Longleftrightarrow\quad
(x_{t-w+1},\ldots,x_t)\in\mathcal C.
\]
Because this equivalence holds for every \(t=w,\ldots,T\), the state-dependent feasibility conditions are equivalent to the original sliding-window constraints.
\end{proof}

\begin{proof}[Proof of Theorem~\ref{thm:bellman_optimality}]
We prove the result by induction on the remaining horizon \(n\).
For \(n=0\), no further actions remain to be selected. Hence, the cumulative reward over the remaining horizon is zero, so \(J_0^\star(s)=0\).
Now assume the claim holds for horizon \(n-1\), and consider an \(n\)-step viable history \(s\in\mathcal S_n\). Any feasible \(n\)-step continuation from \(s\) must first choose an action \(x\) satisfying
\(
x\in A(s).
\)
After taking this action, the new history is
\(
s'=F(s,x),
\)
and feasibility of the remaining \(n-1\) steps requires
\(
s'\in\mathcal S_{n-1}.
\)
By the induction hypothesis, the maximum reward achievable from \(s'\) over the remaining \(n-1\) steps is
\(
J_{n-1}^\star(s').
\)
Therefore, any feasible \(n\)-step continuation from \(s\) has reward at most
\[
\langle\theta^\star,x\rangle+J_{n-1}^\star(F(s,x)).
\]
Taking the maximum over all feasible first actions gives
\[
\text{optimal reward from }s
\leq
\max_{\substack{x\in A(s)\\F(s,x)\in\mathcal S_{n-1}}}
\left\{
\langle\theta^\star,x\rangle+J_{n-1}^\star(F(s,x))
\right\}
=
J_n^\star(s).
\]
It remains to show that this upper bound is achievable. Let \(x^\star\) be an action attaining the maximum in the Bellman recursion. Since
\(
F(s,x^\star)\in\mathcal S_{n-1},
\)
the induction hypothesis guarantees a feasible optimal continuation from \(F(s,x^\star)\) with reward
\(
J_{n-1}^\star(F(s,x^\star)).
\)
Combining this continuation with the first action \(x^\star\) gives a feasible \(n\)-step trajectory from \(s\) with total reward
\[
\langle\theta^\star,x^\star\rangle+J_{n-1}^\star(F(s,x^\star))
=
J_n^\star(s).
\]
Hence \(J_n^\star(s)\) equals the optimal reward over all feasible \(n\)-step continuations from \(s\).
Finally, by Lemma~\ref{lem:finite_memory_equivalence}, finite-memory feasibility is equivalent to the original sliding-window constraints. Therefore, the optimal value of the original problem is
\[
V_T^{\mathrm{off}}(\theta^\star)
=
\max_{s=(x_1,\ldots,x_{w-1})\in\mathcal S_{T-w+1}}
\left\{
\sum_{i=1}^{w-1}\langle\theta^\star,x_i\rangle
+
J_{T-w+1}^\star(s)
\right\},
\]
which completes the proof.
\end{proof}

%%%%%%%%%%%%%%%%%%%%%%%%%%%%%%%%%%%%%%%%%%%%%%%%%%%%%%%%%%%%%%%%%%%%%%%%%%%%%%%%%%%%%%%%%%%%%%%%%%%%%%%%%%%%%%%%%%%%%%%%%%%%%%%%%%%%%%%%%%%%%%%%%%%%%%%%%%%%%%

\subsection{Proof of Lemma~\ref{lem:value_span_diameter}}
\label{app:lem_value_span_diameter}

\begin{proof}
Fix \(n\geq1\) and \(s,s'\in\mathcal S_n\), and let \( R:=\rng(g). \) Subtracting \(g_{\min}\) from \(g\) shifts every \(n\)-step value by the same amount \(ng_{\min}\), so it does not change differences between continuation values. Thus, without loss of generality, assume
\[
0\leq g(x)\leq R
\qquad\text{for all }x\in\mathcal X.
\]
By Assumption~\ref{assump:history_diameter}, there is a feasible transition from \(s\) to \(s'\) of length \(m\leq D\). Since \(s'\in\mathcal S_n\), this transition can be followed by a feasible \(n\)-step continuation from \(s'\). Hence \(s\) admits a feasible continuation of length \(n+m\). For any \(\varepsilon>0\), choose a feasible \(n\)-step continuation from \(s'\) whose reward is at least
\(
J_n^g(s')-\varepsilon.
\)
Prepending the transition from \(s\) to \(s'\) gives a feasible \((n+m)\)-step continuation from \(s\). Since all shifted rewards are nonnegative,
\[
J_{n+m}^g(s)\geq J_n^g(s')-\varepsilon.
\]
Now consider any feasible \((n+m)\)-step continuation from \(s\). Its first \(n\) actions form a feasible \(n\)-step continuation from \(s\), and therefore contribute at most \(J_n^g(s)\). The remaining \(m\) actions contribute at most \(mR\). Thus,
\[
J_{n+m}^g(s)\leq J_n^g(s)+mR.
\]
Combining the two bounds gives
\[
J_n^g(s')-\varepsilon
\leq
J_n^g(s)+mR.
\]
Letting \(\varepsilon\downarrow0\) and using \(m\leq D\),
\[
J_n^g(s')-J_n^g(s)\leq DR.
\]
Applying the same argument to a feasible transition from \(s'\) to \(s\), whose length is also at most \(D\), yields
\[
J_n^g(s)-J_n^g(s')\leq DR.
\]
Therefore,
\[
\left|J_n^g(s)-J_n^g(s')\right|
\leq
D\,R
=
D\,\rng(g).
\]
Since \(s,s'\in\mathcal S_n\) were arbitrary,
\[
\Span_{\mathcal S_n}(J_n^g)
\leq
D\,\rng(g),
\]
which proves the claim.
\end{proof}

%%%%%%%%%%%%%%%%%%%%%%%%%%%%%%%%%%%%%%%%%%%%%%%%%%%%%%%%%%%%%%%%%%%%%%%%%%%%%%%%%%%%%%%%%%%%%%%%%%%%%%%%%%%%%%%%%%%%%%%%%%%%%%%%%%%%%%%%%%%%%%%%%%%%%%%%%%%%%%%%%%%%%%%%%%%%%%%%%%%%%%%%%%%%%%%%%%%%%%%%%%%%%%%%%%%%%%%%%%%%%%%%%%%%%%%%%%%%%%%%%%%%%%%%%%%%%%%%%%%%%%%%%%%%%%%%%%%%%%%%%%%%%%%%%%%%%%%%%%%%

\subsection{Proof of Theorem~\ref{thm:noncyclic_regret}}
\label{app:proof_thm_noncyclic_regret}

For episode \(k\), let
\[
\mathcal I_k:=\{t_k,\ldots,t_{k+1}-1\},
\qquad
\ell_k:=|\mathcal I_k|,
\qquad
n_k:=T-t_k+1,
\]
with \(t_{K+1}:=T+1\). Then
\(
n_{k+1}=n_k-\ell_k.
\)

\begin{lemma}
\label{lem:episodewise_diameter_bound}
Under Assumption~\ref{assump:history_diameter} and the uniform confidence
event, every nonterminal episode \(k\) satisfies
\[
J_{n_k}^\star(s_{t_k})
\leq
\sum_{t\in\mathcal I_k}\overline U_k(x_t)
+
J_{n_{k+1}}^\star(s_{t_{k+1}})
+
4LSD.
\]
For the terminal episode,
\[
J_{n_K}^\star(s_{t_K})
\leq
\sum_{t\in\mathcal I_K}\overline U_K(x_t).
\]
\end{lemma}

\begin{proof}
\label{app:lem_episodewise_diameter_bound}
Fix a nonterminal episode \(k\) and write
\(
m:=n_{k+1}=n_k-\ell_k.
\)
Let \(s_k^\star\in\mathcal S_m\) be the history reached after \(\ell_k\) steps along an optimal \(n_k\)-step continuation from \(s_{t_k}\). On the uniform confidence event,
\(
\overline U_k(x)\geq \langle\theta^\star,x\rangle,
\)
and therefore
\[
J_{n_k}^\star(s_{t_k})
\leq
\sum_{j=0}^{\ell_k-1}
\overline U_k(x_{t_k+j}^\star)
+
J_m^\star(s_k^\star).
\]
Since the learner follows an exact Bellman-optimal policy for
\(\overline U_k\),
\[
J_{k,n_k}(s_{t_k})
=
\sum_{t\in\mathcal I_k}\overline U_k(x_t)
+
J_{k,m}(s_{t_{k+1}}),
\]
and Bellman optimality also gives
\[
\sum_{j=0}^{\ell_k-1}
\overline U_k(x_{t_k+j}^\star)
+
J_{k,m}(s_k^\star)
\leq
J_{k,n_k}(s_{t_k}).
\]
Combining these inequalities yields
\[
\begin{aligned}
J_{n_k}^\star(s_{t_k})
\leq\;&
\sum_{t\in\mathcal I_k}\overline U_k(x_t)
+
J_m^\star(s_{t_{k+1}})
+
\bigl(J_m^\star-J_{k,m}\bigr)(s_k^\star)
-
\bigl(J_m^\star-J_{k,m}\bigr)(s_{t_{k+1}}).
\end{aligned}
\]
Both histories belong to \(\mathcal S_m\). By
Lemma~\ref{lem:value_span_diameter},
\(
\Span_{\mathcal S_m}(J_m^\star)\leq 2LSD.
\)
Since \(-LS\leq\overline U_k(x)\leq LS\), the same lemma gives
\(
\Span_{\mathcal S_m}(J_{k,m})\leq 2LSD.
\)
Hence
\[
\Span_{\mathcal S_m}
\bigl(J_m^\star-J_{k,m}\bigr)
\leq 4LSD,
\]
which proves
\[
J_{n_k}^\star(s_{t_k})
\leq
\sum_{t\in\mathcal I_k}\overline U_k(x_t)
+
J_{n_{k+1}}^\star(s_{t_{k+1}})
+
4LSD.
\]
For the terminal episode, \(n_{K+1}=0\). Since the optimistic policy is
followed to the end of the horizon,
\[
J_{K,n_K}(s_{t_K})
=
\sum_{t\in\mathcal I_K}\overline U_K(x_t).
\]
Optimism gives
\(
J_{n_K}^\star(s_{t_K})
\leq
J_{K,n_K}(s_{t_K}),
\)
which proves the terminal bound.
\end{proof}

\begin{proof}[Proof of Theorem~\ref{thm:noncyclic_regret}]
Let
\(
r^\star(x):=\langle\theta^\star,x\rangle
\)
and define the pathwise pseudo-regret
\[
\mathfrak R_T^{\mathrm{off}}
:=
V_T^{\mathrm{off}}(\theta^\star)
-
\sum_{t=1}^T r^\star(x_t).
\]
For any \(\delta\in(0,1)\), define the uniform confidence event
\[
\mathcal G_\delta
:=
\left\{
\|\widehat\theta_t-\theta^\star\|_{V_t}
\leq
\beta_t(\delta)
\ \text{for all } t=1,\ldots,T+1
\right\}.
\]
By the confidence-ellipsoid result~\citep[Theorem~2]{abbasi2011improved}, which is based on a
self-normalized concentration inequality for vector-valued martingales,
\[
\mathbb P(\mathcal G_\delta)\geq 1-\delta.
\]
We first work on \(\mathcal G_\delta\). For \(w>1\), let
\(
n_1=T-w+1.
\)
If \(s_w^\star\) is the complete history obtained from the first \(w-1\) actions of an offline-optimal trajectory, then
\[
V_T^{\mathrm{off}}(\theta^\star)
=
\sum_{t=1}^{w-1}r^\star(x_t^\star)
+
J_{n_1}^\star(s_w^\star).
\]
Since \(|r^\star(x)|\leq LS\),
\[
\sum_{t=1}^{w-1}
\bigl(r^\star(x_t^\star)-r^\star(x_t)\bigr)
\leq
2LS(w-1).
\]
Moreover, \(s_w^\star,s_w\in\mathcal S_{n_1}\), so
Lemma~\ref{lem:value_span_diameter} gives
\[
J_{n_1}^\star(s_w^\star)-J_{n_1}^\star(s_w)
\leq2LSD.
\]
Therefore,
\[
\mathfrak R_T^{\mathrm{off}}
\leq
2LS(w-1)+2LSD
+
J_{n_1}^\star(s_w)
-
\sum_{t=w}^T r^\star(x_t).
\]
For \(w=1\), the same inequality holds without the initialization terms. Applying Lemma~\ref{lem:episodewise_diameter_bound} over all episodes
and telescoping the continuation values gives
\[
J_{n_1}^\star(s_{t_1})
\leq
\sum_{k=1}^K
\sum_{t\in\mathcal I_k}
\overline U_k(x_t)
+
4LSD(K-1).
\]
Hence
\[
\mathfrak R_T^{\mathrm{off}}
\leq
2LS(w-1)
+
4LSDK
+
\sum_{k=1}^K
\sum_{t\in\mathcal I_k}
\bigl(
\overline U_k(x_t)-r^\star(x_t)
\bigr).
\]
On \(\mathcal G_\delta\),
\(
\left|
\langle\widehat\theta_k-\theta^\star,x_t\rangle
\right|
\leq
\beta_k\|x_t\|_{W_k^{-1}},
\)
so clipping preserves optimism and
\[
0
\leq
\overline U_k(x_t)-r^\star(x_t)
\leq
2\beta_k\|x_t\|_{W_k^{-1}}.
\]
Moreover, throughout episode \(k\),
\[
W_k\preceq V_t,
\qquad
\det(V_t)\leq 2\det(W_k).
\]
Applying Lemma~\ref{lem:determinant-comparison} with
\(A=W_k\), \(B=V_t\), and \(c=2\), we obtain
\[
W_k^{-1}\preceq 2V_t^{-1}.
\]
Consequently,
\[
\|x_t\|_{W_k^{-1}}
\leq
\sqrt{2}\,\|x_t\|_{V_t^{-1}}.
\]
Using \(\beta_k\leq\beta_T(\delta)\),
\[
\sum_{k=1}^K
\sum_{t\in\mathcal I_k}
\bigl(
\overline U_k(x_t)-r^\star(x_t)
\bigr)
\leq
2\sqrt2\,\beta_T(\delta)
\sum_{t=1}^T
\|x_t\|_{V_t^{-1}}.
\]
Since \(\lambda\geq L^2\), applying
Lemma~\ref{lem:elliptical-potential} with \(u_t=x_t\) and \(N=T\) gives
\[
\sum_{t=1}^T
\|x_t\|_{V_t^{-1}}
\leq
\sqrt{
2Td
\log\!\left(
1+\frac{TL^2}{\lambda d}
\right)
}.
\]
Thus,
\[
\sum_{k=1}^K
\sum_{t\in\mathcal I_k}
\bigl(
\overline U_k(x_t)-r^\star(x_t)
\bigr)
\leq
4\beta_T(\delta)
\sqrt{
Td
\log\!\left(
1+\frac{TL^2}{\lambda d}
\right)
}.
\]
It remains to bound the number of episodes. Since \(K\) is the number of
determinant-doubling episodes, Lemma~\ref{lem:number-switches} gives
\[
K
\leq
1+
\left\lceil
d\log_2\!\left(
1+\frac{TL^2}{\lambda d}
\right)
\right\rceil.
\]
Combining the preceding bounds, on \(\mathcal G_\delta\),
\[
\begin{aligned}
\mathfrak R_T^{\mathrm{off}}
\leq\;&
4\beta_T(\delta)
\sqrt{
Td
\log\!\left(
1+\frac{TL^2}{\lambda d}
\right)
}
+
2LS(w-1)
+
4LSD
\left[
1+
\left\lceil
d\log_2\!\left(
1+\frac{TL^2}{\lambda d}
\right)
\right\rceil
\right].
\end{aligned}
\]
Thus,
\[
\mathfrak R_T^{\mathrm{off}}
=
\widetilde{\mathcal O}
\left(
d\sqrt T+dD+w
\right)
\]
with probability at least \(1-\delta\). Finally, choose \(\delta=1/T\). Since both the offline and realized cumulative rewards lie in \([-LST,LST]\),
\(
\mathfrak R_T^{\mathrm{off}}\leq2LST
\)
on every sample path. Therefore,
\[
\begin{aligned}
\operatorname{Reg}^{\mathrm{off}}_{\pi}(T;\theta^\star)
&=
\mathbb E\!\left[\mathfrak R_T^{\mathrm{off}}\right] \leq
\widetilde{\mathcal O}
\left(
d\sqrt T+dD+w
\right)
+
2LST\,\mathbb P(\mathcal G_{1/T}^c)
=
\widetilde{\mathcal O}
\left(
d\sqrt T+dD+w
\right),
\end{aligned}
\]
which proves the theorem.
\end{proof}

\subsection{Extension to Approximate Remaining-Horizon Planning}
\label{app:approximate_planning}

The analysis in Section~\ref{sec:online_learning_without_cyclic} assumes access to an exact Bellman-optimal remaining-horizon planner. We show that this assumption can be relaxed to additive approximate planning, at the cost of an explicit cumulative planning-error term.

\begin{assumption}
\label{assump:approximate_optimistic_planner}
At each planning time \(t_k\), the learner can compute a feasible remaining-horizon policy \(\pi_k\) satisfying
\[
J_{k,n_k}^{\pi_k}(s_{t_k})
\geq
J_{k,n_k}(s_{t_k})-\alpha_k,
\qquad
0\leq\alpha_k\leq\alpha_{\max}.
\]
When \(w>1\), the learner also has access to a feasible initialization.
\end{assumption}

Since $J_{k,n_k}(s_{t_k})$ is Bellman optimal, setting $\alpha_k=0$ recovers exact planning.

\begin{lemma}
\label{lem:approx_episodewise_diameter_bound}
Under Assumption~\ref{assump:history_diameter}, the uniform confidence event, and Assumption~\ref{assump:approximate_optimistic_planner}, every nonterminal episode \(k\) satisfies
\[
J_{n_k}^\star(s_{t_k})
\leq
\sum_{t\in\mathcal I_k}\overline U_k(x_t)
+
J_{n_{k+1}}^\star(s_{t_{k+1}})
+
4LSD
+
\alpha_k.
\]
For the terminal episode,
\[
J_{n_K}^\star(s_{t_K})
\leq
\sum_{t\in\mathcal I_K}\overline U_K(x_t)
+
\alpha_K.
\]
\end{lemma}

\begin{proof}
The proof follows that of Lemma~\ref{lem:episodewise_diameter_bound}, with the exact Bellman identity replaced by the additive approximation guarantee. Fix a nonterminal episode \(k\) and let
\(
m:=n_{k+1}=n_k-\ell_k.
\)
By Assumption~\ref{assump:approximate_optimistic_planner},
\[
J_{k,n_k}(s_{t_k})
\leq
J_{k,n_k}^{\pi_k}(s_{t_k})+\alpha_k.
\]
Since the learner follows \(\pi_k\) during episode \(k\),
\[
J_{k,n_k}^{\pi_k}(s_{t_k})
=
\sum_{t\in\mathcal I_k}\overline U_k(x_t)
+
J_{k,m}^{\pi_k}(s_{t_{k+1}}),
\]
and optimality of \(J_{k,m}\) gives
\[
J_{k,m}^{\pi_k}(s_{t_{k+1}})
\leq
J_{k,m}(s_{t_{k+1}}).
\]
Hence
\[
J_{k,n_k}(s_{t_k})
\leq
\sum_{t\in\mathcal I_k}\overline U_k(x_t)
+
J_{k,m}(s_{t_{k+1}})
+
\alpha_k.
\]
Substituting this inequality into the same optimism and Bellman comparison used in the proof of Lemma~\ref{lem:episodewise_diameter_bound} yields
\[
\begin{aligned}
J_{n_k}^\star(s_{t_k})
\leq\;&
\sum_{t\in\mathcal I_k}\overline U_k(x_t)
+
J_m^\star(s_{t_{k+1}})
+
\bigl(J_m^\star-J_{k,m}\bigr)(s_k^\star)
-
\bigl(J_m^\star-J_{k,m}\bigr)(s_{t_{k+1}})
+
\alpha_k.
\end{aligned}
\]
As before, \(s_k^\star,s_{t_{k+1}}\in\mathcal S_m\), and Lemma~\ref{lem:value_span_diameter} gives
\[
\Span_{\mathcal S_m}(J_m^\star)\leq2LSD,
\qquad
\Span_{\mathcal S_m}(J_{k,m})\leq2LSD.
\]
Therefore,
\[
\Span_{\mathcal S_m}
\bigl(J_m^\star-J_{k,m}\bigr)
\leq4LSD,
\]
and thus
\[
J_{n_k}^\star(s_{t_k})
\leq
\sum_{t\in\mathcal I_k}\overline U_k(x_t)
+
J_{n_{k+1}}^\star(s_{t_{k+1}})
+
4LSD+\alpha_k.
\]
For the terminal episode,
\[
J_{n_K}^\star(s_{t_K})
\leq
J_{K,n_K}(s_{t_K})
\leq
J_{K,n_K}^{\pi_K}(s_{t_K})+\alpha_K
=
\sum_{t\in\mathcal I_K}\overline U_K(x_t)+\alpha_K,
\]
which proves the claim.
\end{proof}

\begin{theorem}
\label{thm:noncyclic_regret_approximate}
Suppose the online linear-bandit assumptions of Section~\ref{sec:problem_formulation} hold, with \(\|x\|_2\leq L\), \(\|\theta^\star\|_2\leq S\), and \(\lambda\geq L^2\). Under Assumptions~\ref{assump:history_diameter} and~\ref{assump:approximate_optimistic_planner},
the rare-update optimistic planning policy satisfies
\[
\operatorname{Reg}^{\mathrm{off}}_{\pi}(T;\theta^\star)
=
\widetilde{\mathcal O}\!\left(
d\sqrt{T}+dD+w+d\alpha_{\max}
\right).
\]
\end{theorem}

\begin{proof}
Applying Lemma~\ref{lem:approx_episodewise_diameter_bound} over all episodes and telescoping the continuation values yields
\[
J_{n_1}^\star(s_{t_1})
\leq
\sum_{k=1}^{K}\sum_{t\in\mathcal I_k}\overline U_k(x_t)
+
4LSD(K-1)
+
\sum_{k=1}^{K}\alpha_k.
\]
Combining this with the initialization bound from the proof of Theorem~\ref{thm:noncyclic_regret}, on the uniform confidence event,
\[
\mathfrak R_T^{\mathrm{off}}
\leq
2LS(w-1)
+
4LSDK
+
\sum_{k=1}^{K}\sum_{t\in\mathcal I_k}
\bigl(\overline U_k(x_t)-r^\star(x_t)\bigr)
+
\sum_{k=1}^{K}\alpha_k.
\]
The confidence-width and elliptical-potential bounds give
\[
\sum_{k=1}^{K}\sum_{t\in\mathcal I_k}
\bigl(\overline U_k(x_t)-r^\star(x_t)\bigr)
=
\widetilde{\mathcal O}(d\sqrt T).
\]
Moreover, Lemma~\ref{lem:number-switches} gives \(K=\widetilde{\mathcal O}(d)\). Hence
\[
4LSDK=\widetilde{\mathcal O}(dD),
\qquad
\sum_{k=1}^{K}\alpha_k
\leq
K\alpha_{\max}
=
\widetilde{\mathcal O}(d\alpha_{\max}).
\]
Therefore,
\[
\mathfrak R_T^{\mathrm{off}}
=
\widetilde{\mathcal O}
\left(
d\sqrt T+dD+w+d\alpha_{\max}
\right)
\]
with probability at least \(1-\delta\). Choosing \(\delta=1/T\) and proceeding as in the final step of Theorem~\ref{thm:noncyclic_regret} gives
\[
\operatorname{Reg}^{\mathrm{off}}_{\pi}(T;\theta^\star)
=
\widetilde{\mathcal O}
\left(
d\sqrt T+dD+w+d\alpha_{\max}
\right).
\]
\end{proof}

\begin{remark}
The effect of approximate planning enters only through the cumulative planning error. Since \(K=\widetilde{\mathcal O}(d)\),
\[
\sum_{k=1}^{K}\alpha_k
\leq
K\alpha_{\max}
=
\widetilde{\mathcal O}(d\alpha_{\max}).
\]
Thus, bounded per-episode planning error degrades the regret gracefully. In particular, if \(\alpha_{\max}=\mathcal O(1)\), the term \(d\alpha_{\max}\) is absorbed by \(d\sqrt T\), so the approximate planner achieves the same asymptotic regret rate as the exact planner. Setting \(\alpha_{\max}=0\) recovers exact planning.
\end{remark}

\section{Synthetic Experiments}
\label{app:synthetic_experiments}

In this section, we conduct three experiments to evaluate our theoretical results. Specifically, we ask: \textbf{(Q1)} Does the stationary reduction recover the exact offline optimum when the horizon aligns with the window length, and is the resulting \(\mathcal O(w)\) gap tight when the horizon is misaligned? \textbf{(Q2)} Does rare switching reduce the feasibility cost of moving between stationary actions when direct switches are not possible? \textbf{(Q3)} Beyond cyclic symmetry, can finite-memory optimistic planning effectively exploit recent history while preserving exact sliding-window feasibility? Additional ablation and sensitivity studies are provided in Appendix~\ref{app:synthetic_ablation_studies}.

\subsection{Offline Geometry and Horizon Misalignment}
\label{app:exp-offline-geometry}

\begin{figure}[t!]
    \centering
    \includegraphics[width=0.90\linewidth]{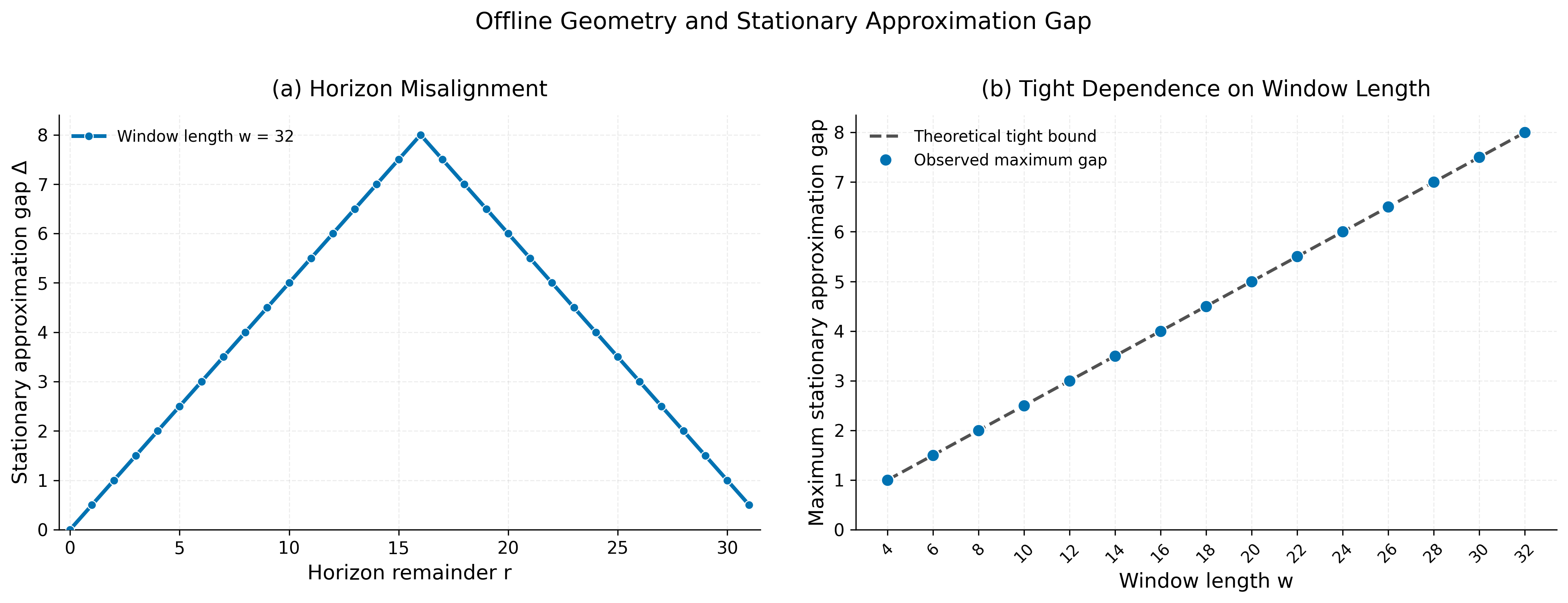}
    \caption{Offline geometry and horizon misalignment. \textbf{Left:} stationary approximation gap versus horizon remainder for $w=32$. \textbf{Right:} maximum gap versus window length; markers show the LP solutions and the dashed line shows the tight value $w/4$.}
    \label{fig:offline-geometry}
\end{figure}

We address Q1 by testing the stationary-reduction result of Theorem~\ref{thm:stationary_near_optimality}. We consider
\[
\mathcal X=[0,1],\qquad
\mathcal C_w=
\left\{
(x_1,\ldots,x_w)\in[0,1]^w:
\sum_{i=1}^w x_i\le \frac{w}{2}
\right\},
\]
with linear reward $f(x)=x$. The stationary feasible set is $\mathcal Z=[0,1/2]$, so the optimal stationary action is $z^\star=1/2$. We use even window lengths $w\in\{4,6,8,\ldots,32\}$ and horizons
\[
T=20w+r,\qquad r\in\{0,\ldots,w-1\}.
\]
For each $(w,r)$, we compute the exact offline optimum $V_T$ by linear programming and measure the stationary approximation gap
\(
\Delta_T=V_T-Tf(z^\star).
\)
Figure~\ref{fig:offline-geometry} summarizes the results. In the left panel, we fix $w=32$ and vary the remainder $r=T\bmod w$. The gap is zero when $r=0$, increases as the horizon becomes misaligned with the window, reaches its maximum at $r=w/2$, and then decreases symmetrically as $r$ approaches $w$. The right panel shows the maximum gap over $r$ for different window lengths. For every tested even $w$, the LP solution gives
\(
\max_r\Delta_T=\frac{w}{4},
\)
with the maximum attained at $r=w/2$. This exactly matches the tightness construction in Appendix~\ref{app:tightness-offline-gap}, where
\(
V_T-Tf(z^\star)=\frac{w}{4}
\)
for horizons of the form $T=qw+w/2$.

\subsection{Rare Switching under Cyclic Ramping Constraints}
\label{app:exp-rare-switching}

\begin{figure}[t!]
    \centering
    \includegraphics[width=0.85\linewidth]{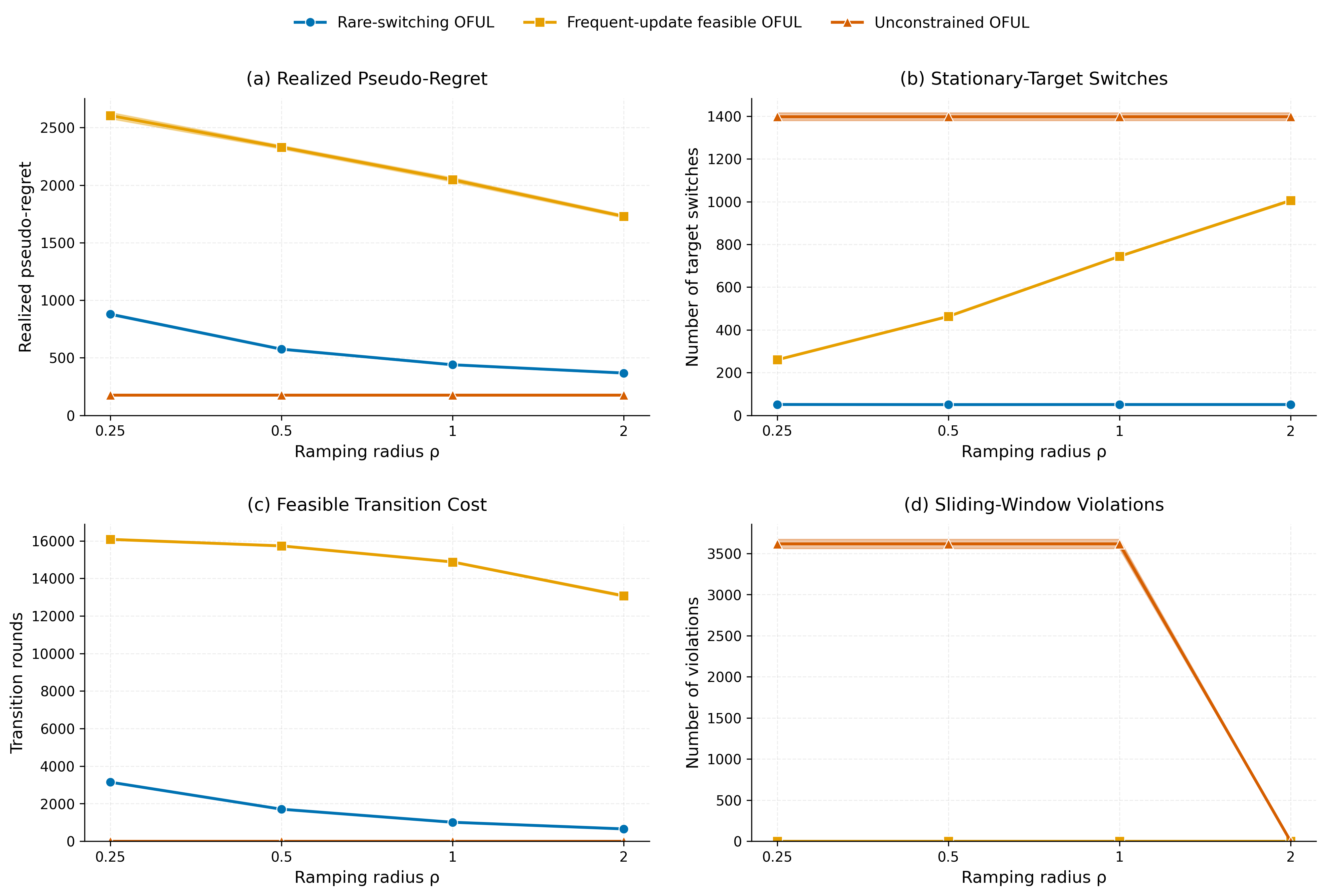}
    \caption{Rare switching under cyclic ramping constraints. \textbf{(a)} Pseudo-regret, \textbf{(b)} stationary-target switches, \textbf{(c)} transition rounds, and \textbf{(d)} constraint violations for Rare-switching OFUL, Frequent-update feasible OFUL, and Unconstrained OFUL across $\rho$. Curves show means over 30 runs with 95\% confidence intervals.}
    \label{fig:rare-switching}
\end{figure}

We address Q2 by examining the cost of maintaining feasibility when transitions between stationary actions cannot be made instantaneously. We consider the simplex action space $\mathcal X=\Delta_d$ with $d=8$ and cyclic ramping constraints
\[
\mathcal C_\rho=
\left\{
(x_1,\ldots,x_w)\in\mathcal X^w:
\max_{1\le i\le w}\|x_{i+1}-x_i\|_1\le\rho,\;
x_{w+1}=x_1
\right\},
\]
with window length \(w=8\). Every stationary action is feasible, so $\mathcal Z=\Delta_d$. We use
\(
\theta^\star=(0.55,0.60,0.65,0.70,0.75,0.80,0.90,1.00)^\top
\)
and observations
\[
Y_t=\langle\theta^\star,x_t\rangle+\eta_t,\qquad
\eta_t\sim\mathcal N(0,0.1^2).
\]
We fix $T=16384$ and vary the ramping radius over
\(
\rho\in\{0.25,0.5,1,2\}.
\)
Since $w\mid T$, the stationary and offline benchmarks coincide. We compare three policies. \emph{Rare-switching OFUL} follows Algorithm~\ref{alg:rare-switching-oful} and updates its optimistic stationary target only when the determinant-based episode condition is triggered. \emph{Frequent-update feasible OFUL} uses the same confidence sets, optimistic-action oracle, and transition mechanism, but recomputes its target after every retained stationary observation. \emph{Unconstrained OFUL}~\citep{abbasi2011improved} ignores the ramping constraint and switches directly to each newly selected optimistic action. For each policy, we report the realized pseudo-regret
\[
\widehat R_T=
\sum_{t=1}^{T}
\langle\theta^\star,z^\star-x_t\rangle,
\]
the number of stationary-target switches $J_T$, the total number of transition rounds, and the number of sliding-window violations
\[
N_{\rm viol}=
\sum_{t=w}^{T}
\mathbf 1\left\{
(x_{t-w+1},\ldots,x_t)\notin\mathcal C_\rho
\right\}.
\]

Figure~\ref{fig:rare-switching}(a) shows that Rare-switching OFUL achieves substantially lower pseudo-regret than Frequent-update feasible OFUL across all tested values of $\rho$. The reason is reflected in Figure~\ref{fig:rare-switching}(b): Rare-switching OFUL performs far fewer stationary-target switches, which in turn leads to many fewer transition rounds, as shown in Figure~\ref{fig:rare-switching}(c). This benefit is especially pronounced for small $\rho$, where moving between targets requires longer feasible transition paths. Figure~\ref{fig:rare-switching}(d) shows that both feasible methods maintain zero sliding-window violations for all tested values of $\rho$, whereas Unconstrained OFUL violates the constraint whenever direct switching is infeasible. At $\rho=2$, Unconstrained OFUL also satisfies the constraint because the $\ell_1$ diameter of the simplex is at most $2$.

%%%%%%%%%%%%%%%%%%%%%%%%%%%%%%%%%%%%%%%%%%%%%%%%%%%%%%%%%%%%%%%%%%%%%%%%%%%%%%%%%%%%%%%%%%%%%%%%%%%% 
%%%%%%%%%%%%%%%%%%%%%%%%%%%%%%%%%%%%%%%%%%%%%%%%%%%%%%%%%%%%%%%%%

\subsection{Planning Beyond Cyclic Symmetry}
\label{app:exp-planning}

\begin{figure}[t!]
    \centering
    \includegraphics[width=0.85\linewidth]{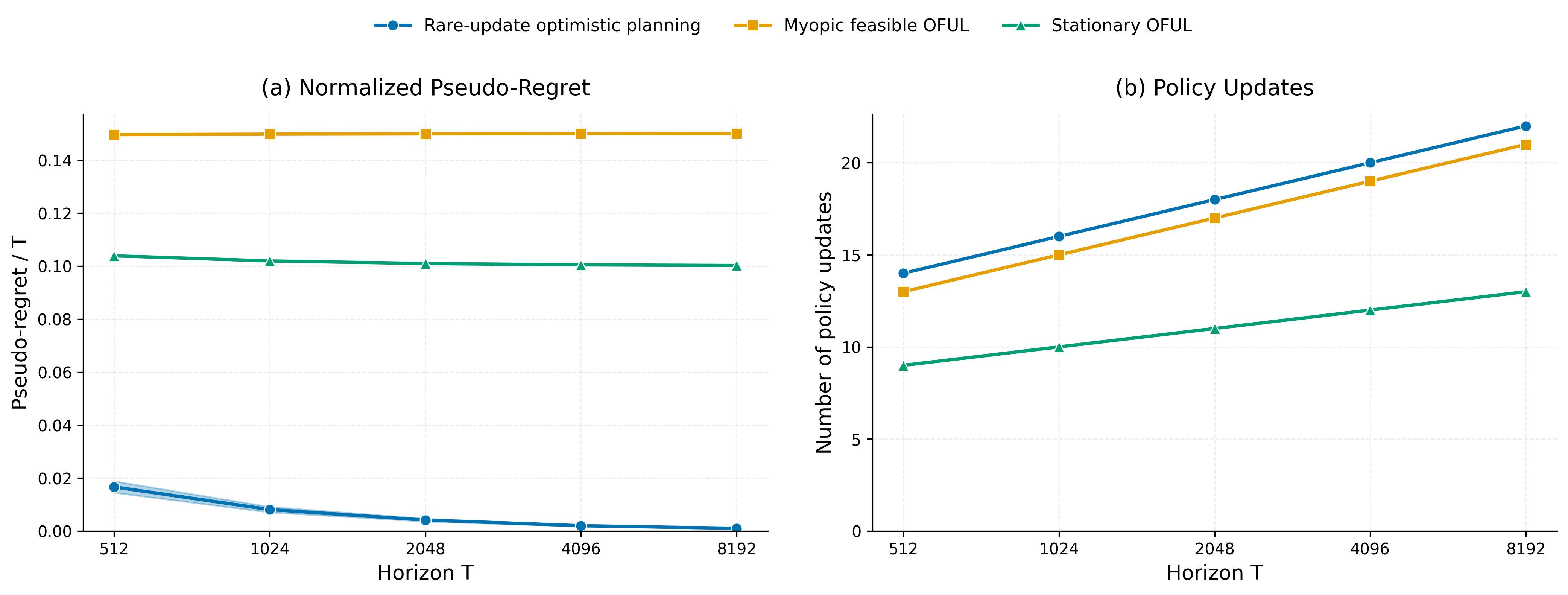}
    \caption{Planning beyond cyclic symmetry. \textbf{Left:} normalized pseudo-regret relative to the exact offline optimum. \textbf{Right:} number of policy updates. Curves show means over 30 runs with 95\% confidence intervals.}
    \label{fig:noncyclic-planning}
\end{figure}

We address Q3 with a non-cyclic constraint for which current actions affect future feasible choices and stationary behavior is suboptimal. Let
\[
\mathcal X=\{0,e_1,e_2\}\subset\mathbb R^2,
\qquad
c=(1,3)^\top,
\qquad
w=4,
\]
and define
\[
\mathcal C
=
\left\{
(x_1,x_2,x_3,x_4)\in\mathcal X^4:
4c^\top x_1+2c^\top x_2+c^\top x_3+c^\top x_4\leq14
\right\}.
\]
Because the coefficients depend on position within the window, \(\mathcal C\) is not invariant under cyclic shifts. Hence feasibility at round \(t\) depends on the history state \(s_t=(x_{t-3},x_{t-2},x_{t-1})\). We use
\[
\theta^\star=(0.6,1.6)^\top,
\qquad
Y_t=\langle\theta^\star,x_t\rangle+\eta_t,
\qquad
\eta_t\sim\mathcal N(0,0.1^2).
\]
The stationary feasible set is \(\mathcal Z=\{0,e_1\}\), so the best stationary reward is \(0.6\) per round. In contrast, the periodic trajectory
\( (e_1,e_2,0,e_1,e_1,e_2,0,e_1,\ldots) \)
is feasible and has average reward
\[
\frac{0.6+1.6+0+0.6}{4}=0.7.
\]
Thus, exploiting non-stationary feasible behavior requires planning over the history state. We compare three feasible policies. \emph{Rare-Update Optimistic Planning} follows Algorithm~\ref{alg:noncyclic-rare-switching} and optimizes the frozen optimistic reward over the full remaining horizon. \emph{Myopic Feasible OFUL} uses the same estimator and confidence sets, but selects the current feasible action without accounting for future rewards. \emph{Stationary OFUL} restricts its actions to the stationary feasible set \(\mathcal Z\). Figure~\ref{fig:noncyclic-planning} shows a clear benefit from remaining-horizon planning. The normalized pseudo-regret of rare-update optimistic planning decreases rapidly with \(T\) and approaches zero, indicating that the learner approaches the offline-optimal non-stationary behavior. In contrast, Stationary OFUL approaches a normalized regret of approximately \(0.1\), matching the gap between the optimal long-run reward \(0.7\) and the best stationary reward \(0.6\). Myopic feasible OFUL performs even worse, with a persistent normalized-regret gap of approximately \(0.15\), showing that immediate feasibility and optimism alone do not account for the effect of current actions on future feasible rewards. The right panel shows that these gains are achieved with few policy updates: the number of replanning events grows slowly with the horizon, consistent with the determinant-doubling update rule. Further details of this experiment are provided in Appendix~\ref{app:implementation_details}.
\section{Ablation and Sensitivity Studies}
\label{app:ablation}

We conduct ablation and sensitivity studies on both synthetic and real-world tasks to analyze the robustness of our method and understand the impact of key design choices.

\subsection{Real-World Ablation Studies}
\label{app:real-world-ablation}

\paragraph{Sensitivity to Budget and Horizon.}
We evaluate the routing policies across sliding-window budgets
\(B \in \{0.15, 0.25, 0.35\}\)
and horizons
\(T \in \{500, 2500, 5000\}\). Table~\ref{tab:llm_budget_sensitivity} shows that Rare-Update Optimistic Planning maintains competitive reward with zero violations across all settings while requiring substantially fewer policy updates. As the budget becomes tighter, constraint violations increase for Unconstrained OFUL and Primal-Dual OFUL.

\begin{table*}[t]
\centering
\caption{LLM-routing performance across different sliding-window budgets. Average reward and normalized pseudo-regret are reported as mean $\pm$ 95\% confidence interval over 10 independent simulations for each configuration; violations and updates are reported as means.}
\label{tab:llm_budget_sensitivity}
\small
\setlength{\tabcolsep}{5pt}
\renewcommand{\arraystretch}{1.1}
\begin{tabular}{cclcccc}
\toprule
Budget & $T$ & Method & Avg. Reward & Norm. Pseudo-Regret & Violations & Updates \\
\midrule

$0.15$ & $500$ & Rare-Update Optimistic Planning
& $0.7618 \pm 0.0116$
& $0.1566 \pm 0.0132$
& 0.0
& 249.8 \\

& & Myopic Feasible OFUL
& $0.7583 \pm 0.0120$
& $0.1601 \pm 0.0108$
& 0.0
& 500.0 \\

& & Unconstrained OFUL
& $0.7713 \pm 0.0167$
& $0.1471 \pm 0.0156$
& 104.5
& 500.0 \\

& & Primal-Dual OFUL
& $0.7686 \pm 0.0161$
& $0.1498 \pm 0.0136$
& 112.7
& 500.0 \\

\cmidrule(lr){2-7}

& $2500$ & Rare-Update Optimistic Planning
& $0.7699 \pm 0.0039$
& $0.1509 \pm 0.0041$
& 0.0
& 789.1 \\

& & Myopic Feasible OFUL
& $0.7682 \pm 0.0048$
& $0.1525 \pm 0.0049$
& 0.0
& 2500.0 \\

& & Unconstrained OFUL
& $0.7780 \pm 0.0059$
& $0.1427 \pm 0.0064$
& 527.2
& 2500.0 \\

& & Primal-Dual OFUL
& $0.7760 \pm 0.0065$
& $0.1447 \pm 0.0062$
& 529.1
& 2500.0 \\

\cmidrule(lr){2-7}

& $5000$ & Rare-Update Optimistic Planning
& $0.7691 \pm 0.0035$
& $0.1518 \pm 0.0034$
& 0.0
& 1105.3 \\

& & Myopic Feasible OFUL
& $0.7706 \pm 0.0030$
& $0.1503 \pm 0.0030$
& 0.0
& 5000.0 \\

& & Unconstrained OFUL
& $0.7769 \pm 0.0036$
& $0.1441 \pm 0.0038$
& 1088.6
& 5000.0 \\

& & Primal-Dual OFUL
& $0.7786 \pm 0.0040$
& $0.1424 \pm 0.0042$
& 1115.7
& 5000.0 \\

\midrule

$0.25$ & $500$ & Rare-Update Optimistic Planning
& $0.7670 \pm 0.0122$
& $0.1595 \pm 0.0103$
& 0.0
& 250.1 \\

& & Myopic Feasible OFUL
& $0.7722 \pm 0.0150$
& $0.1543 \pm 0.0117$
& 0.0
& 500.0 \\

& & Unconstrained OFUL
& $0.7713 \pm 0.0167$
& $0.1552 \pm 0.0153$
& 34.3
& 500.0 \\

& & Primal-Dual OFUL
& $0.7724 \pm 0.0165$
& $0.1541 \pm 0.0141$
& 35.0
& 500.0 \\

\cmidrule(lr){2-7}

& $2500$ & Rare-Update Optimistic Planning
& $0.7741 \pm 0.0063$
& $0.1546 \pm 0.0054$
& 0.0
& 787.5 \\

& & Myopic Feasible OFUL
& $0.7730 \pm 0.0037$
& $0.1558 \pm 0.0041$
& 0.0
& 2500.0 \\

& & Unconstrained OFUL
& $0.7780 \pm 0.0059$
& $0.1507 \pm 0.0065$
& 189.5
& 2500.0 \\

& & Primal-Dual OFUL
& $0.7785 \pm 0.0058$
& $0.1503 \pm 0.0065$
& 189.9
& 2500.0 \\

\cmidrule(lr){2-7}

& $5000$ & Rare-Update Optimistic Planning
& $0.7766 \pm 0.0022$
& $0.1526 \pm 0.0022$
& 0.0
& 1104.4 \\

& & Myopic Feasible OFUL
& $0.7768 \pm 0.0023$
& $0.1524 \pm 0.0023$
& 0.0
& 5000.0 \\

& & Unconstrained OFUL
& $0.7769 \pm 0.0036$
& $0.1523 \pm 0.0036$
& 410.8
& 5000.0 \\

& & Primal-Dual OFUL
& $0.7781 \pm 0.0038$
& $0.1511 \pm 0.0038$
& 424.4
& 5000.0 \\

\midrule

$0.35$ & $500$ & Rare-Update Optimistic Planning
& $0.7626 \pm 0.0114$
& $0.1670 \pm 0.0116$
& 0.0
& 250.0 \\

& & Myopic Feasible OFUL
& $0.7724 \pm 0.0138$
& $0.1572 \pm 0.0119$
& 0.0
& 500.0 \\

& & Unconstrained OFUL
& $0.7713 \pm 0.0167$
& $0.1583 \pm 0.0160$
& 9.7
& 500.0 \\

& & Primal-Dual OFUL
& $0.7713 \pm 0.0167$
& $0.1583 \pm 0.0160$
& 9.7
& 500.0 \\

\cmidrule(lr){2-7}

& $2500$ & Rare-Update Optimistic Planning
& $0.7786 \pm 0.0051$
& $0.1531 \pm 0.0047$
& 0.0
& 789.8 \\

& & Myopic Feasible OFUL
& $0.7767 \pm 0.0052$
& $0.1550 \pm 0.0046$
& 0.0
& 2500.0 \\

& & Unconstrained OFUL
& $0.7780 \pm 0.0059$
& $0.1537 \pm 0.0061$
& 64.2
& 2500.0 \\

& & Primal-Dual OFUL
& $0.7783 \pm 0.0061$
& $0.1534 \pm 0.0062$
& 63.6
& 2500.0 \\

\cmidrule(lr){2-7}

& $5000$ & Rare-Update Optimistic Planning
& $0.7784 \pm 0.0024$
& $0.1539 \pm 0.0024$
& 0.0
& 1103.8 \\

& & Myopic Feasible OFUL
& $0.7789 \pm 0.0037$
& $0.1533 \pm 0.0037$
& 0.0
& 5000.0 \\

& & Unconstrained OFUL
& $0.7769 \pm 0.0036$
& $0.1554 \pm 0.0036$
& 124.8
& 5000.0 \\

& & Primal-Dual OFUL
& $0.7769 \pm 0.0036$
& $0.1554 \pm 0.0036$
& 124.8
& 5000.0 \\

\bottomrule
\end{tabular}
\end{table*}

\begin{figure}[t!]
    \centering
    \includegraphics[width=0.85\linewidth]{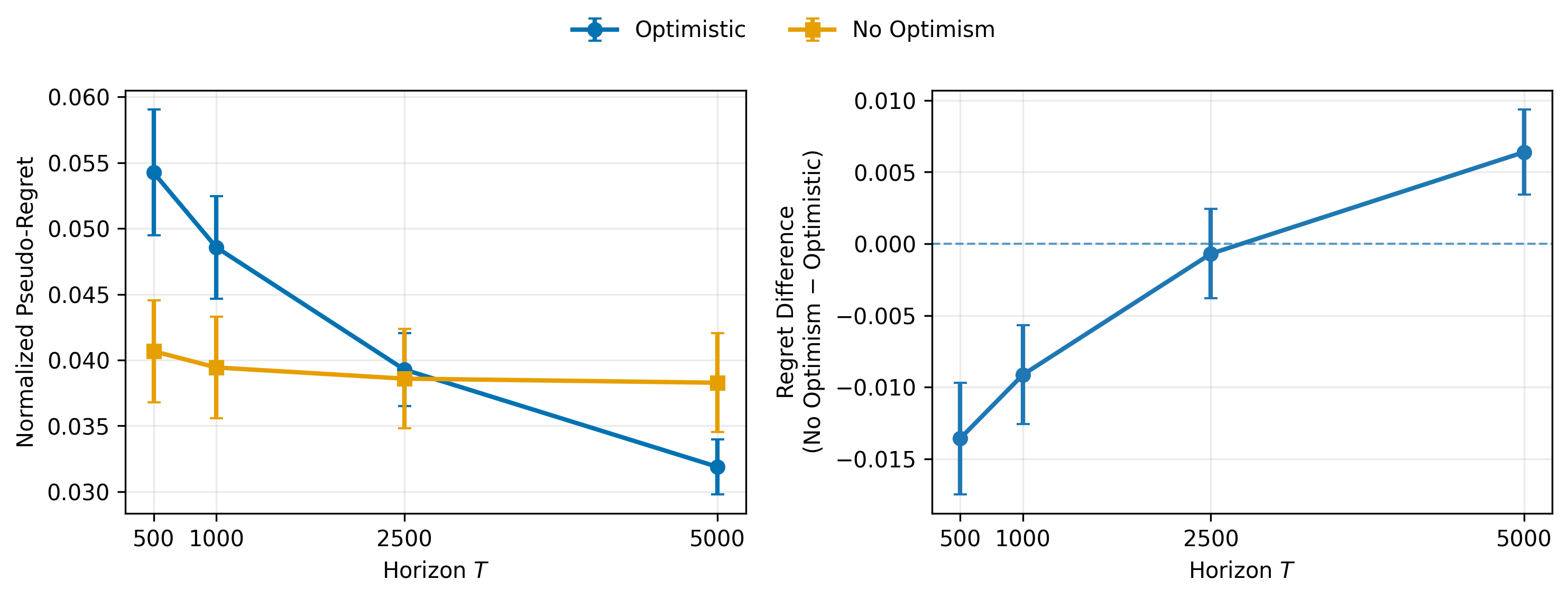}
    \caption{Ablation on optimistic exploration. \textbf{Left}: normalized pseudo-regret for Rare-Update Optimistic Planning and its non-optimistic variant. \textbf{Right:} paired regret difference, defined as \(\mathrm{Regret}_{\mathrm{NoOpt}}-\mathrm{Regret}_{\mathrm{Opt}}\), so positive values favor optimism. Error bars denote 95\% confidence intervals.}
    \label{fig:optimism_ablation}
\end{figure}

%%%%%%%%%%%%%%%%%%%%%%%%%%%%%%%%%%%%%%%%%%%%%%%%%%%%%%%%%%%%%%%%%%%%%%%%%%%%%%%%%%%%%%%%%%%%%%%%%%%%%%%%%%%%%%%%%%%%%%%%%%%%%%%%%%%%%%%%%%%%%%%%%%%%%%%%%%%%%%%%%%%%%%%%%%%%%%%%%%%%%%%%%%%%%%%%%%%%%%%%%%%%%%%%%%%%%%%%%%%%%%%%%%%%%%%%%%%%%%%%%%%%%%%%

\paragraph{Ablation on Optimistic Exploration.}
We isolate the effect of optimism by comparing Rare-Update Optimistic Planning with a non-optimistic variant that uses the same planning horizon, feasibility mechanism, and determinant-based update schedule, but removes the confidence bonus. The two variants use
\[
U_t(a)
=
\langle \widehat{\theta}_t,e_a\rangle
+
\beta_t\|e_a\|_{V_t^{-1}}
\qquad\text{and}\qquad
\widehat{r}_t(a)
=
\langle \widehat{\theta}_t,e_a\rangle,
\]
respectively. We evaluate \(T\in\{500,1000,2500,5000\}\) with window size \(w=20\) and planning horizon \(H=5\), using \(50\) users and \(10\) independent runs per user.

Figure~\ref{fig:optimism_ablation} shows that optimism incurs a short-horizon exploration cost but becomes beneficial as the horizon grows. The non-optimistic variant achieves lower regret at \(T=500\) and \(T=1000\), while the difference is small at \(T=2500\). At \(T=5000\), the optimistic variant achieves lower regret, and the paired difference is positive.

%%%%%%%%%%%%%%%%%%%%%%%%%%%%%%%%%%%%%%%%%%%%%%%%%%%%%%%%%%%%%%%%%%%%%%%%%%%%%%%%%%%%%%%%%%%%%%%%%%%%%%%%%%%%%%%%%%%%%%%%%%%%%%%%%%%%%%%%%%%%%%%%%%%%%%%%%%%%%%%%%%%%%%%%%%%%%%%%%%%%%%%%%%%%%%%%%%%%%%%%%%%%%%%%%%%%%%%%%%%%%%%%%%%%%%%%%%%%%%%%%%%%%%%%

\subsection{Synthetic Ablation Studies}
\label{app:synthetic_ablation_studies}

We conduct additional experiments to examine how the main design choices in feasible rare-switching OFUL affect empirical performance. We use the same setup as in Section~\ref{app:synthetic_experiments}, with window length $w=8$, horizon $T=16384$, Gaussian observation noise with standard deviation $0.1$, and the same reward parameter $\theta^\star$.

%%%%%%%%%%%%%%%%%%%%%%%%%%%%%%%%%%%%%%%%%%%%%%%%%%%%%%%%%%%%%%%%%%%%%%%%%%%%%%%%%%%%%%%%%%%%%%%%%%%%%%%%%%%%%%%%%%%%%%%%%%%%%%%%%%%%%%%%%%%%%%%%%%%%%%%%%%%%%%%%%%%%%%%%%%%%%%%%%%%%%%%%%%%%%%%%%%%%%%%%%%%%%%%%%%%%%%%%%%%%%%%%%%%%%%%%%%%%%%%%%%%%%%%%

\paragraph{Switching Threshold}
\label{app:ablation:threshold}

We first study the determinant threshold that controls the duration of each stationary episode. In the main algorithm, the current optimistic target is kept fixed until
\(
\det(V_t)>2\det(V_{t_k}),
\)
where $s_k$ denotes the start of episode $k$. To examine the sensitivity to this choice, we replace the factor $2$ with a general threshold $\alpha>1$ and consider
\[
\det(V_t)>\alpha\det(V_{s_k}),
\qquad
\alpha\in\{1.25,1.5,2,4,8\}.
\]

All other parameters are kept fixed, including the ramping radius $\rho=1$. Smaller values of $\alpha$ lead to more frequent target updates, while larger values produce longer episodes and fewer switches. Figure~\ref{fig:ablation-threshold} shows that both the number of target switches and the transition cost decrease steadily as $\alpha$ increases, dropping from about $135$ switches and $2700$ transition rounds at $\alpha=1.25$ to roughly $20$ switches and $400$ rounds at $\alpha=8$. The effect on pseudo-regret is non-monotone. Very small thresholds incur high transition cost, while very large thresholds update the target less often and adapt more slowly. The lowest regret is achieved at the default choice $\alpha=2$, which provides a good balance between statistical adaptation and transition cost.

\begin{figure}[t!]
    \centering
    \includegraphics[width=0.9\linewidth]{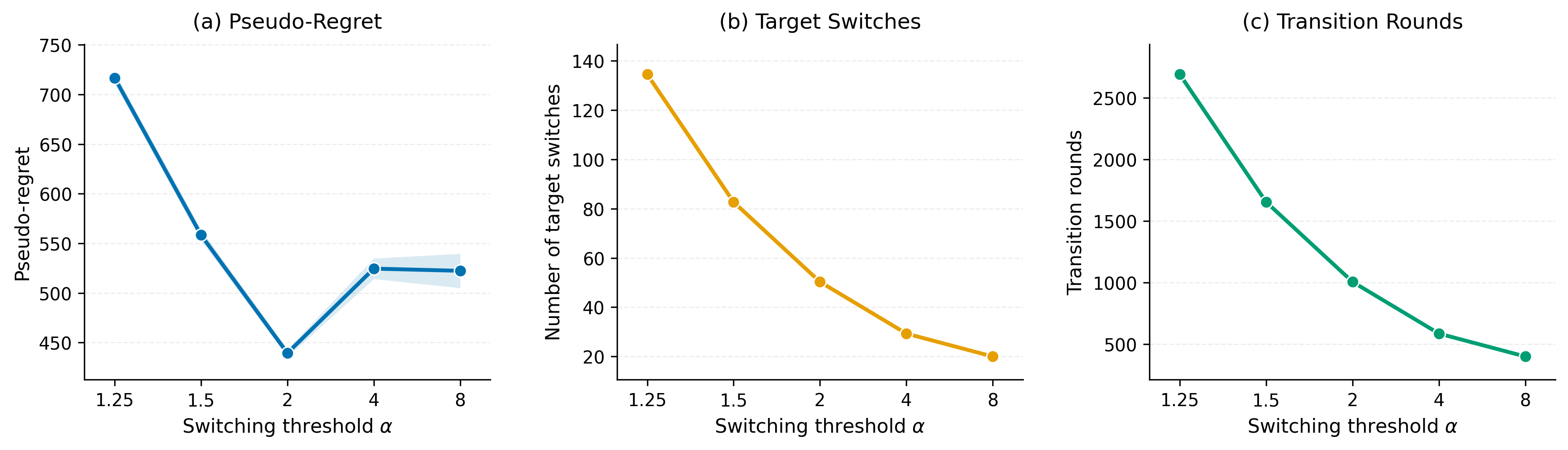}
    \caption{Sensitivity to the determinant switching threshold $\alpha$. \textbf{(a)} Realized pseudo-regret, \textbf{(b)} number of stationary-target switches, and \textbf{(c)} total number of transition rounds. The default choice $\alpha=2$ achieves the lowest pseudo-regret in this experiment. Curves show means over 30 runs with 95\% confidence intervals.}
    \label{fig:ablation-threshold}
\end{figure}

%%%%%%%%%%%%%%%%%%%%%%%%%%%%%%%%%%%%%%%%%%%%%%%%%%%%%%%%%%%%%%%%%%%%%%%%%%%%%%%%%%%%%%%%%%%%%%%%%%%%%%%%%%%%%%%%%%%%%%%%%%%%%%%%%%%%%%%%%%%%%%%%%%%%%%%%%%%%%%%%%%%%%%%%%%%%%%%%%%%%%%%%%%%%%%%%%%%%%%%%%%%%%%%%%%%%%%%%%%%%%%%%%%%%%%%%%%%%%%%%%%%%%%%%

\paragraph{Scaling with the Action Dimension}
\label{app:ablation:dimension}

We next examine how rare-switching OFUL behaves as the dimension of the simplex action space increases. We vary \( d\in\{4,8,16,32\}, \) while keeping $w=8$, $T=16384$, $\rho=1$, and the determinant threshold $\alpha=2$ fixed. To isolate the effect of the action dimension, we use
\[
\theta_d^\star=(0.8,\ldots,0.8,1.0)^\top\in\mathbb R^d.
\]
Thus, the last simplex vertex is optimal and has reward gap $0.2$ relative to every other vertex, independent of $d$. This keeps the statistical separation between the optimal and suboptimal actions fixed as the dimension grows. Within each run, we use the same noise realization for all values of $d$.

Figure~\ref{fig:ablation-dimension} reports pseudo-regret, the number of stationary-target switches, and the total number of transition rounds. As the dimension increases, the learner must explore a larger number of candidate actions while the reward gap remains fixed. This leads to higher learning cost and, in turn, can produce more target changes and feasibility-preserving transitions.

\begin{figure}[t!]
    \centering
    \includegraphics[width=0.9\linewidth]{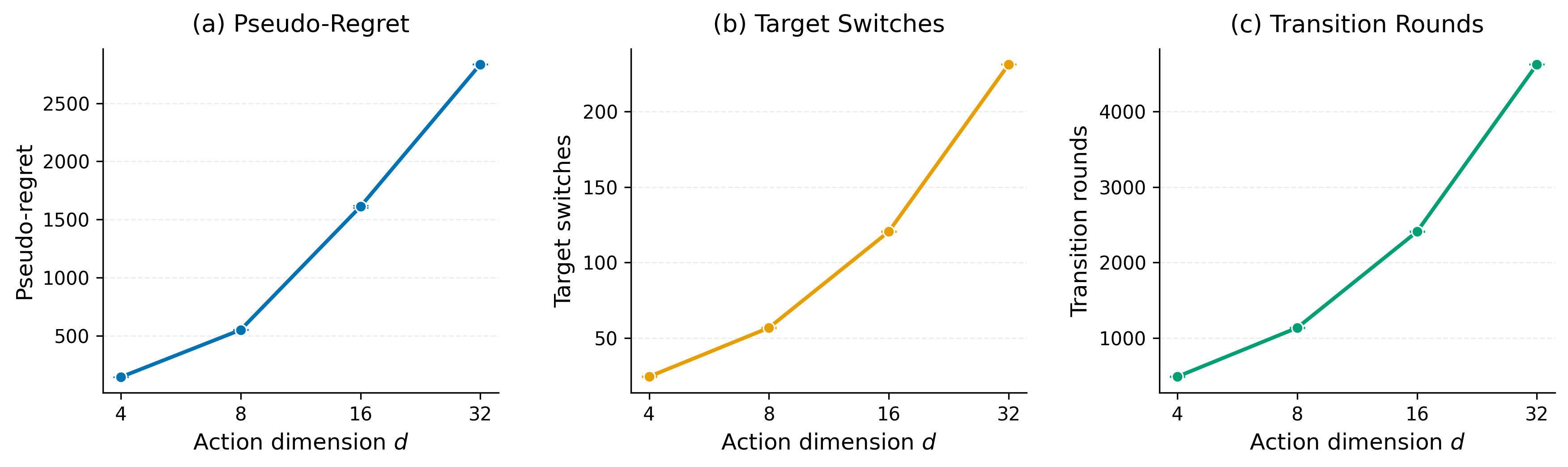}
    \caption{Scaling with the action dimension for $d\in\{4,8,16,32\}$ with a fixed reward gap of $0.2$. \textbf{(a)} Realized pseudo-regret, \textbf{(b)} number of stationary-target switches, and \textbf{(c)} total number of transition rounds. Curves show means over 30 runs with 95\% confidence intervals.}
    \label{fig:ablation-dimension}
\end{figure}

%%%%%%%%%%%%%%%%%%%%%%%%%%%%%%%%%%%%%%%%%%%%%%%%%%%%%%%%%%%%%%%%%%%%%%%%%%%%%%%%%%%%%%%%%%%%%%%%%%%%%%%%%%%%%%%%%%%%%%%%%%%%%%%%%%%%%%%%%%%%%%%%%%%%%%%%%%%%%%%%%%%%%%%%%%%%%%%%%%%%%%%%%%%%%%%%%%%%%%%%%%%%%%%%%%%%%%%%%%%%%%%%%%%%%%%%%%%%%%%%%%%%%%%%

\paragraph{Scaling with the Window Length}
\label{app:ablation:window}

Finally, we examine how the window length affects the cost of maintaining feasibility. We vary \( w\in\{2,4,8,16,32\}, \) and set $T=2048w$, so that $T/w=2048$ is fixed across settings. All other parameters are unchanged. Figure~\ref{fig:ablation-window} shows that the number of stationary-target switches changes very little with $w$, remaining between roughly $48$ and $52$. In contrast, the number of transition rounds grows sharply, from about $100$ at $w=2$ to nearly $4800$ at $w=32$. Figure~\ref{fig:ablation-window}~(d) makes the source of this increase explicit: the number of transition rounds per switch rises from $2$ to about $92$ over the same range. In our simplex setting, distinct stationary targets are simplex vertices and therefore satisfy $\|z'-z\|_1=2$. With $\rho=1$, the transition construction consequently requires
\(
3w-4
\)
rounds per switch, which gives $2,8,20,44,$ and $92$ rounds for the tested window lengths. Thus, the increase in transition cost is driven almost entirely by longer feasible transitions rather than by more frequent switching. Pseudo-regret also increases with $w$, although the horizon grows proportionally with $w$, so this trend should not be interpreted as an isolated effect of the window length.

\begin{figure}[t!]
    \centering
    \includegraphics[width=\linewidth]{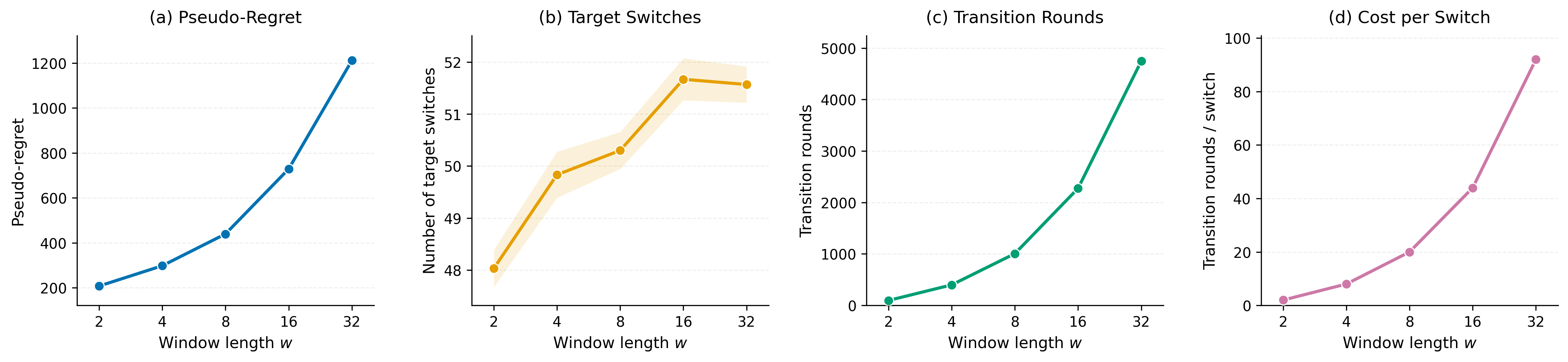}
    \caption{Scaling with the window length for $w\in\{2,4,8,16,32\}$ with $T/w=2048$ fixed. \textbf{(a)} Realized pseudo-regret, \textbf{(b)} number of stationary-target switches, \textbf{(c)} total number of transition rounds, and \textbf{(d)} transition rounds per switch. The switch count remains nearly constant, while the cost of each feasible switch grows with $w$. Curves show means over 30 runs with 95\% confidence intervals.}
    \label{fig:ablation-window}
\end{figure}

%%%%%%%%%%%%%%%%%%%%%%%%%%%%%%%%%%%%%%%%%%%%%%%%%%%%%%%%%%%%%%%%%%%%%%%%%%%%%%%%%%%%%%%%%%%%%%%%%%%%%%%%%%%%%%%%%%%%%%%%%%%%%%%%%%%%%%%%%%%%%%%%%%%%%%%%%%%%%%%%%%%%%%%%%%%%%%%%%%%%%%%%%%%%%%%%%%%%%%%%%%%%%%%%%%%%%%%%%%%%%%%%%%%%%%%%%%%%%%%%%%%%%%%%

\paragraph{Planning Horizon}
\label{app:ablation:planning-horizon}

We examine how much lookahead is needed to benefit from history-state planning. Using the non-cyclic instance from Section~\ref{app:exp-planning}, we replace exact remaining-horizon planning by an \(H\)-step optimistic planner with
\( H\in\{1,2,4,8,16,32,64,128,256,512\}, \) and compare these choices with the full remaining-horizon planner. At round \(t\), the truncated planner selects
\[
x_t
\in
\arg\max_{\substack{x\in A(s_t)\\
F(s_t,x)\in\mathcal S_{T-t}}}
\left\{
\overline U(x)
+
J_{\min\{H-1,T-t\}}^{\overline U}
\bigl(F(s_t,x)\bigr)
\right\}.
\]
The viability condition is imposed over the full remaining horizon, so every policy remains exactly feasible; only the reward lookahead is truncated. Thus, \(H=1\) corresponds to the myopic feasible policy, while ``Full'' denotes exact remaining-horizon planning as in Algorithm~\ref{alg:noncyclic-rare-switching}. We fix \(T=1024\) and keep all statistical parameters and the determinant-doubling update rule unchanged. Figure~\ref{fig:ablation-planning-horizon} shows that longer lookahead generally improves performance, with the largest gains appearing at longer horizons. The normalized pseudo-regret is about \(0.15\) for the myopic policy and drops to roughly \(0.10\) at \(H=2\). It changes only modestly over several short and moderate horizons, before falling more clearly for larger \(H\): to approximately \(0.08\) at \(H=128\), \(0.05\) at \(H=512\), and below \(0.01\) under full remaining-horizon planning. This supports the role of continuation information in evaluating the future consequences of current actions. The number of policy updates remains small across all planning horizons, ranging from roughly \(10\) to \(16\), and increases only gradually with \(H\).

\begin{figure}[t!]
    \centering
    \includegraphics[width=0.85\linewidth]{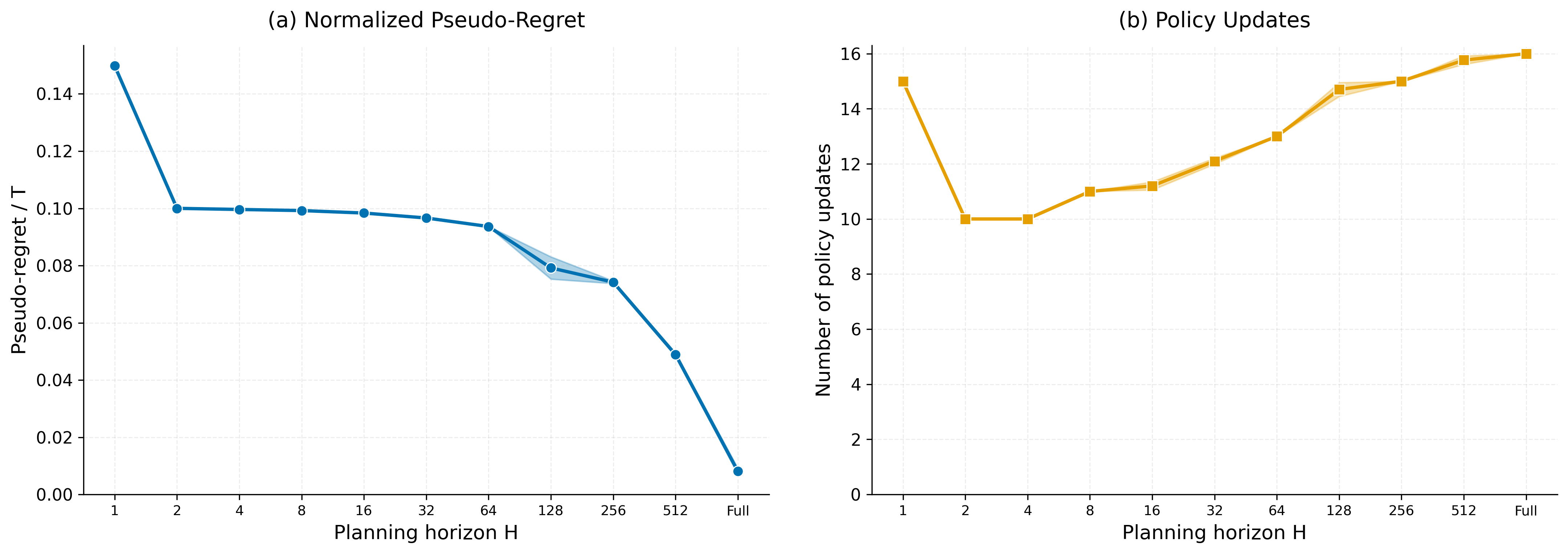}
    \caption{Sensitivity to the planning horizon \(H\) for \(T=1024\). \textbf{Left:} normalized pseudo-regret against the offline-optimal feasible trajectory. \textbf{Right:} number of policy updates. The case \(H=1\) is myopic and ``Full'' denotes exact remaining-horizon planning. Curves show means over 30 runs with 95\% confidence intervals.}
    \label{fig:ablation-planning-horizon}
\end{figure}

%%%%%%%%%%%%%%%%%%%%%%%%%%%%%%%%%%%%%%%%%%%%%%%%%%%%%%%%%%%%%%%%%%%%%%%%%%%%%%%%%%%%%%%%%%%%%%%%%%%%%%%%%%%%%%%%%%%%%%%%%%%%%%%%%%%%%%%%%%%%%%%%%%%%%%%%%%%%%%%%%%%%%%%%%%%%%%%%%%%%%%%%%%%%%%%%%%%%%%%%%%%%%%%%%%%%%%%%%%%%%%%%%%%%%%%%%%%%%%%%%%%%%%%%
\section{Implementation Details}
\label{app:implementation_details}

Complete code, preprocessing scripts, and experimental configurations are publicly available at \url{https://anonymous.4open.science/r/sliding-window-constrained-bandits}.

\subsection{Real-World Experiments}
\label{app:real_world_experiments}

\paragraph{Recommendation Diversity Experiment Setup}
We use the \(K=8\) most frequent first-level content categories in the processed KuaiRand-1K dataset~\citep{gao2022kuairand}. Each category \(j\in\mathcal{A}=\{1,\ldots,K\}\) is treated as an action and represented by the one-hot vector \(e_j\in\mathbb{R}^K\). For each user \(u\), we construct the preference vector
\(
\theta_u^\star
=
\bigl(
\theta_{u,1}^\star,\ldots,\theta_{u,K}^\star
\bigr),
\)
where \(\theta_{u,j}^\star\) is the smoothed empirical click probability for category \(j\) obtained during preprocessing. We retain users with observations from all \(K\) selected categories and at least \(50\) historical interactions in every category, and randomly select \(100\) eligible users. At round \(t\), after selecting category \(a_t\), the environment generates
\[
Y_t
\sim
\operatorname{Bernoulli}
\bigl(
\theta_{u,a_t}^\star
\bigr),
\qquad
\mathbb{E}[Y_t\mid a_t=j]
=
\langle
\theta_u^\star,e_j
\rangle.
\]
The learner observes only the realized binary reward \(Y_t\). The vector \(\theta_u^\star\) remains hidden from the learning algorithms and is used only by the environment and for evaluation. We impose the exact five-in-twenty exposure constraint
\[
\sum_{q=t-19}^{t}
\mathbf{1}\{a_q=j\}
\leq 5,
\qquad
j\in\mathcal{A},
\quad
t=20,\ldots,T.
\]
Hence, with history \(s_t=(a_{t-19},\ldots,a_{t-1})\), the feasible action set is
\[
\mathcal{A}(s_t)
=
\left\{
j\in\mathcal{A}:
\sum_{q=t-19}^{t-1}
\mathbf{1}\{a_q=j\}<5
\right\}.
\]
We evaluate
\(
T\in
\{
500,2500,5000
\}.
\)
For each user and horizon, we perform \(10\) independent simulations. We first average repeated runs within each user and then average the resulting quantities across users. For Rare-Update Optimistic Planning, we use a planning horizon of \(H=5\). For all the OFUL-based methods, we use regularization \(\lambda=1\), confidence level \(\delta=0.05\), sub-Gaussian parameter \(R=1/2\), and the bound \(S=\sqrt{K}\).

\paragraph{LLM Routing Experiment Setup}
We use the LLMRouterBench dataset~\citep{li2026llmrouterbench}, which provides query-level evaluations of multiple large language models, including response quality scores and inference cost information. We construct routing streams from four benchmark tasks: MMLU-Pro~\citep{wang2024mmlu} for general knowledge reasoning, GPQA~\citep{rein2024gpqa} for graduate-level scientific reasoning, LiveCodeBench~\citep{jain2025livecodebench} for coding problems, and ArenaHard~\citep{li2025crowdsourced} for instruction-following and conversational evaluation. These benchmarks represent diverse query distributions with different difficulty levels, where different language models exhibit different quality--cost trade-offs. We select a pool of \(K=6\) language models: GPT-5~\citep{openai2025gpt5}, Claude-Sonnet-4~\citep{anthropic2025claude4}, Gemini-2.5-Pro~\citep{google2025gemini25pro}, DeepSeek-V3~\citep{deepseekai2024deepseek}, Qwen3-235B~\citep{yang2025qwen3}, and Gemini-2.5-Flash~\citep{google2025gemini25flash}. Each model \(j\in\mathcal A=\{1,\ldots,K\}\) is considered an available routing action. For every query \(q_t\), we first extract an \(\ell_2\)-normalized semantic representation using the pretrained Sentence-BERT encoder \texttt{all-MiniLM-L6-v2}~\citep{reimers2019sentence}:
\[
u_t\in\mathbb R^{384},
\qquad
\|u_t\|_2=1.
\]
To reduce the feature dimension, we project the query embeddings onto their leading \(r=64\) principal components. We then \(\ell_2\)-normalize the projected representation, obtaining
\[
v_t\in\mathbb R^{64},
\qquad
\|v_t\|_2=1.
\]
The PCA projection is fitted once on the set of query embeddings before the simulations, without using reward or cost information. For each model \(j\in\mathcal A\), we construct the query--model interaction feature
\[
\phi(q_t,j)=e_j\otimes v_t\in\mathbb R^{64K},
\]
where \(e_j\in\mathbb R^K\) is the one-hot encoding of model \(j\). Since \(K=6\), the resulting feature dimension is \( d=64K=384. \)
After selecting model \(a_t\), the bandit feature vector is
\( x_t:=\phi(q_t,a_t)\in\mathbb R^d. \)
Let \(\mu(q_t,j)\) denote the quality score provided by LLMRouterBench for query \(q_t\) and model \(j\). When a model \(a_t\in\mathcal A\) is selected, the environment returns
\( Y_t=\mu(q_t,a_t). \)
The online algorithms use a linear predictor over \(\phi(q_t,j)\) to estimate model quality. Thus, in this real-world experiment, we do not require the benchmark reward table to be exactly linear in the chosen features. The inference costs \(c(q_t,j)\) are obtained from LLMRouterBench and are used without additional rescaling. The target sliding-window cost constraint is
\[
\sum_{i=t-w+1}^{t}c(q_i,a_i)\le B,
\qquad
t=w,\ldots,T,
\]
with window length \(w=5\). To guarantee pathwise feasibility without knowledge of future queries, the online feasible policies additionally use the zero-padded prefix convention
\[
\sum_{i=\max\{1,t-w+1\}}^{t}c(q_i,a_i)\le B,
\qquad
t=1,\ldots,T.
\]
This condition coincides with the standard length-\(w\) sliding-window constraint for \(t\geq w\) and is stronger only during the first \(w-1\) rounds. We evaluate \(T\in\{500,2500,5000\}\) and \(B\in\{0.15,0.25,0.35,0.40\}\). For each configuration, we perform \(10\) independent simulations using query streams sampled uniformly without replacement. Within each simulation, all methods are evaluated on the same sampled query stream. Since future queries are unavailable at decision time, we use planning horizon \(H=1\). Thus, in this application, Rare-Update Optimistic Planning reduces to myopic feasible optimistic selection while retaining the determinant-based rare-update mechanism.

For all OFUL-based methods, we use ridge regularization parameter \(\lambda_{\mathrm{reg}}=1\), confidence parameter \(\delta=0.05\), and exploration scale \(R=0.5\). The determinant-doubling factor is \(2\). In the LLM-routing experiments, the optimistic score uses the fixed confidence multiplier
\(
\beta_{\mathrm{exp}}
=
R\sqrt{d\log(1/\delta)}.
\)
This experimental confidence multiplier is used uniformly by the OFUL-based methods and differs from the theoretical confidence radius used in the regret analysis. Future queries are not available to the learner before their arrival. Therefore, the LLM-routing experiment does not apply the full remaining-horizon viability recursion of Section~\ref{sec:online_learning_without_cyclic}. After observing the current query \(q_t\), the learner instead restricts its choice to the currently feasible LLM actions
\[
\mathcal A_t(s_t,q_t)
=
\left\{
j\in\mathcal A:
\sum_{i=\max\{1,t-w+1\}}^{t-1}
c(q_i,a_i)
+
c(q_t,j)
\le B
\right\}.
\]
If \(\mathcal A_t(s_t,q_t)=\varnothing\), the learner executes a no-op fallback \(\bot\) satisfying
\[
c(q_t,\bot)=0,
\qquad
\mu(q_t,\bot)=0.
\]
The no-op is used only when no real LLM action is feasible. Because the preceding prefix is maintained within budget, the zero-cost fallback preserves feasibility. Consequently, the feasible policies maintain the sliding-window constraint pathwise without requiring knowledge of future queries. Because future query-dependent costs are unknown, this contextual LLM-routing experiment is not a direct instantiation of Theorem~\ref{thm:noncyclic_regret}; the theorem's regret guarantee assumes known deterministic finite-memory transitions.

\paragraph{Baselines.}
We compare against three baselines. \emph{Myopic Feasible OFUL} selects
\[
a_t
\in
\arg\max_{a\in\mathcal A_t(s_t,q_t)}
U_t(a),
\]
where \(U_t(a)\) denotes the optimistic reward estimate. If \(\mathcal A_t(s_t,q_t)=\varnothing\), it uses the same zero-cost, zero-reward no-op fallback. This baseline enforces feasibility at every round and recomputes its optimistic action at every round, providing a fully updated feasible baseline against which to evaluate the rare-update mechanism.

\paragraph{Primal-Dual OFUL.}
For the recommendation experiment, the sliding-window constraint consists of one exposure constraint for each category. Accordingly, Primal-Dual OFUL maintains a vector of nonnegative dual variables
\[
\lambda_t
=
(\lambda_{t,1},\ldots,\lambda_{t,K})
\in\mathbb R_+^K,
\]
with one multiplier for each category. Let \(U_t(a)\) denote the OFUL optimistic reward estimate for category \(a\). At round \(t\), the policy selects
\[
a_t
\in
\arg\max_{a\in\mathcal A}
\left\{
U_t(a)-\lambda_{t,a}
\right\}.
\]

After observing the reward and updating the OFUL estimator, the dual variables are updated whenever a complete length-\(w\) window is available. For \(t\geq w\), let
\( \ell_t:=t-w+1 \) denote the dual-update index. For each category \(j\in\mathcal A\),
\[
\lambda_{t+1,j}
=
\left[
\lambda_{t,j}
+
\eta_{\ell_t}
\left(
\sum_{i=t-w+1}^{t}\mathbf 1\{a_i=j\}
-
c
\right)
\right]_+,
\qquad
\eta_{\ell_t}
=
\frac{\eta}{\sqrt{\ell_t}},
\]
where \([x]_+:=\max\{x,0\}\). In the recommendation experiments, \(K=8\), \(w=20\), and \(c=5\), so each category has its own dual multiplier corresponding to its five-in-twenty exposure constraint.

For the LLM-routing experiment, there is a single sliding-window cost constraint, so Primal-Dual OFUL uses a scalar nonnegative multiplier \(\lambda_t\). At round \(t\), it selects
\[
a_t
\in
\arg\max_{a\in\mathcal A}
\left\{
U_t(a)-\lambda_t c(q_t,a)
\right\},
\]
and updates the multiplier according to
\[
\lambda_{t+1}
=
\left[
\lambda_t
+
\eta_t
\left(
\sum_{i=\max\{1,t-w+1\}}^{t}
c(q_i,a_i)-B
\right)
\right]_+,
\qquad
\eta_t=\frac{\eta}{\sqrt{t}}.
\]

\emph{Unconstrained OFUL} selects
\[
a_t
\in
\arg\max_{a\in\mathcal A}
U_t(a)
\]
without enforcing the sliding-window constraint. It provides a reference for the reward attainable when exact feasibility is ignored. Because this policy may violate the constraint, its pseudo-regret relative to the constrained offline optimum can be negative and should be interpreted together with its number of constraint violations.

\paragraph{Offline benchmark.}
For each real-world application, we define an offline oracle that has access to the complete reward information and computes the maximum achievable reward under the corresponding application-specific constraint. This oracle is used only as an evaluation reference and is not available to the online learning algorithms. For the recommendation diversity experiment, the offline oracle has access to the user preference vector \(\theta_u^\star\) and solves
\[
V_T^{\mathrm{off}}(\theta_u^\star)
=
\max_{a_1,\ldots,a_T}
\sum_{t=1}^{T}
\theta_{u,a_t}^{\star},
\]
subject to the five-in-twenty exposure constraint
\[
\sum_{q=t-19}^{t}
\mathbf{1}\{a_q=j\}
\leq5,
\qquad
j\in\mathcal A,
\quad
t=20,\ldots,T.
\]
Since all evaluated horizons are divisible by \(20\), the constrained optimum is obtained by selecting the four categories with the largest preference values and repeating a \(20\)-round periodic schedule containing exactly five recommendations from each of these four categories. Every length-\(20\) sliding window therefore contains exactly five occurrences of each selected category. For the LLM-routing experiment, the offline oracle has access to the complete query--model reward and cost tables. Let \(\mu(q_t,j)\) denote the benchmark quality score for query \(q_t\) and model \(j\). The oracle computes
\[
V_T^{\mathrm{off}}
=
\max_{a_1,\ldots,a_T}
\sum_{t=1}^{T}
\mu(q_t,a_t),
\]
subject to the same zero-padded sliding-window constraint
\[
\sum_{i=\max\{1,t-w+1\}}^{t}
c(q_i,a_i)
\le B,
\qquad
t=1,\ldots,T.
\]
The no-op fallback \(\bot\), with
\(c(q_t,\bot)=0\) and \(\mu(q_t,\bot)=0\)
is available to the oracle only when no real LLM action is feasible under the current history. Thus, the online feasible policies and the offline oracle are evaluated under the same action and feasibility conventions.

\paragraph{Evaluation metrics.}
We evaluate all methods using four metrics: average reward, normalized pseudo-regret, constraint violations, and the number of policy updates. For the recommendation diversity experiment, the average expected reward is
\[
\frac{1}{T}
\sum_{t=1}^{T}
\mathbb E[Y_t]
=
\frac{1}{T}
\sum_{t=1}^{T}
\theta_{u,a_t}^{\star}.
\]
For the LLM-routing experiment, the reward is the benchmark quality score, so the average reward is
\[
\frac{1}{T}
\sum_{t=1}^{T}
\mu(q_t,a_t),
\]
where \(\mu(q_t,\bot)=0\) when the fallback action is used. For the recommendation diversity experiment, normalized pseudo-regret is
\[
\frac{R_T}{T},
\qquad
R_T
=
V_T^{\mathrm{off}}(\theta_u^\star)
-
\sum_{t=1}^{T}
\theta_{u,a_t}^{\star}.
\]
For the LLM-routing experiment, normalized pseudo-regret is
\[
\frac{R_T}{T},
\qquad
R_T
=
V_T^{\mathrm{off}}
-
\sum_{t=1}^{T}
\mu(q_t,a_t).
\]
The offline benchmarks have access to the complete reward information and are used only as evaluation references. For the recommendation diversity experiment, the number of constraint violations is
\[
N_{\mathrm{viol}}
=
\sum_{t=20}^{T}
\mathbf{1}
\left\{
\exists j\in\mathcal A:
\sum_{q=t-19}^{t}
\mathbf{1}\{a_q=j\}>5
\right\}.
\]
For the LLM-routing experiment, we report violations of the target length-\(w\) sliding-window inference-cost constraint:
\[
N_{\mathrm{viol}}
=
\sum_{t=w}^{T}
\mathbf{1}
\left\{
\sum_{i=t-w+1}^{t}
c(q_i,a_i)>B
\right\}.
\]
Finally, we report the number of policy updates performed by each algorithm. For Rare-Update Optimistic Planning, an update corresponds to refreshing the frozen optimistic model when the determinant-doubling condition is triggered. Fewer updates therefore indicate fewer policy recomputations.

%%%%%%%%%%%%%%%%%%%%%%%%%%%%%%%%%%%%%%%%%%%%%%%%%%%%%%%%%%%%%%%%%%%%%%%%%%%%%%%%%%%%%%%%%%%%%%%%%%%%%%%%%%%%%%%%%%%%%%%%%%%%%%%%%%%%%%%%%%%%%%%%%%%%%%%%%%%%%%%%%%%%%%%%%%%%%%%%%%%%%%%%%%%%%%%%%%%%%%%%%%%%%%%%%%%%%%%%%%%%%%%%%%%%%%%%%%%%%%%%%%%%%%%%%%%%%%%%%%%%%%%%%%%%%%%%%%%%%%%%%%%%%%%%%%%%%%%%%%%%%%%%%%%%%%%%%%%%%%%%%%%%%%%%%%%%%%%%%%%%%%%%%%%%%%%%%%

\subsection{Synthetic Experiments}

\paragraph{Offline geometry and horizon misalignment.}
For the experiment in Section~\ref{app:exp-offline-geometry}, we solve the finite-horizon offline problem exactly by linear programming for every pair \((w,r)\). We use even window lengths
\(
w\in\{4,6,8,\ldots,32\},
\)
and horizons
\[
T=20w+r,
\qquad
r\in\{0,\ldots,w-1\}.
\]
For each instance, we compute the exact optimum \(V_T\) and report
\(
\Delta_T
=
V_T-Tf(z^\star).
\)
This experiment is deterministic and therefore does not involve repeated stochastic runs or confidence parameters.

\paragraph{Rare switching under cyclic ramping constraints.}
For the experiment in Section~\ref{app:exp-rare-switching}, we use
\[
d=8,
\qquad
w=8,
\qquad
T=16384,
\]
with
\(
\theta^\star
=
(0.55,0.60,0.65,0.70,0.75,0.80,0.90,1.00)^\top
\)
and observations
\[
Y_t
=
\langle\theta^\star,x_t\rangle+\eta_t,
\qquad
\eta_t\sim\mathcal{N}(0,0.1^2).
\]
We vary
\(
\rho\in\{0.25,0.5,1,2\}.
\)
For all OFUL-based policies in this experiment, we use
\[
\lambda=1,
\qquad
\delta=0.05,
\qquad
R=0.1,
\qquad
S=\|\theta^\star\|_2.
\]
We initialize \(V_1=\lambda I_d\) and, for the rare-switching feasible algorithm, update the estimator using retained stationary observations as specified in Algorithm~\ref{alg:rare-switching-oful}. Since \(\mathcal Z=\Delta_d\), the optimistic stationary-action oracle is implemented exactly by evaluating
\[
\widehat{\theta}_{t,i}
+
\beta_t(\delta)
\sqrt{(V_t^{-1})_{ii}},
\qquad
i=1,\ldots,d,
\]
at the vertices \(e_i\) and selecting a maximizer. To preserve feasibility when changing stationary targets from \(z\) to \(z'\), we use the transition
\[
N
=
\left\lceil
\frac{(w-1)\|z'-z\|_1}{\rho}
\right\rceil,
\qquad
y_j
=
z+\frac{j}{N}(z'-z),
\quad
j=1,\ldots,N,
\]
followed by \(z'\) for \(w-2\) additional rounds. All transition rounds are included in the pseudo-regret. This is the transition construction used in the rare-switching experiment.  We compare \emph{Rare-Switching OFUL}, \emph{Frequent-Update Feasible OFUL}, and \emph{Unconstrained OFUL}. The rare-switching method changes its optimistic stationary target only at determinant-based episode boundaries, whereas the frequent-update feasible baseline recomputes the target after every retained stationary observation. The unconstrained baseline ignores the ramping constraint. 

\paragraph{Planning beyond cyclic symmetry.}
For the experiment in Section~\ref{app:exp-planning}, we use
\[
\mathcal X=\{0,e_1,e_2\},
\qquad
w=4,
\qquad
\theta^\star=(0.6,1.6)^\top,
\]
with observations
\[
Y_t
=
\langle\theta^\star,x_t\rangle+\eta_t,
\qquad
\eta_t\sim\mathcal{N}(0,0.1^2).
\]
We evaluate
\(
T\in
\{512,1024,2048,4096,8192\},
\)
using
\[
\lambda=1,
\qquad
\delta=0.05,
\qquad
R=0.1,
\qquad
S=\|\theta^\star\|_2.
\]

Because \(\mathcal X\) is finite and \(w=4\), we compute viable history states explicitly and implement exact remaining-horizon planning by backward dynamic programming. The offline value
\(
V_T^{\mathrm{off}}(\theta^\star)
\)
is computed by the same exact dynamic-programming formulation. We compare three feasible policies. \emph{Rare-Update Optimistic Planning} follows Algorithm~\ref{alg:noncyclic-rare-switching} and optimizes the frozen optimistic reward over the full remaining horizon. \emph{Myopic Feasible OFUL} uses the same statistical estimator and confidence construction but selects
\[
x_t
\in
\arg\max_{\substack{x\in\mathcal A(s_t)\\
F(s_t,x)\in\mathcal S_{T-t}}}
U_k(x),
\]
without including a continuation value. \emph{Stationary OFUL} restricts its actions to the stationary feasible set \(\mathcal Z\). This comparison separates the gain from explicit history-state planning from both statistical optimism and the restriction to stationary behavior. For this experiment, we report
\[
\frac{\widehat R_T}{T},
\qquad
\widehat R_T
=
V_T^{\mathrm{off}}(\theta^\star)
-
\sum_{t=1}^{T}
\langle\theta^\star,x_t\rangle,
\]
together with the number of policy updates. All three policies satisfy the sliding-window constraint exactly.

\paragraph{Computational Infrastructure and Software Environment}
All reported experiments were conducted on a machine equipped with a 10-thread CPU. The experiments were implemented in Python 3.10.20 using standard scientific computing libraries. No GPU acceleration was required.

\paragraph{Dataset and License Information}
We use publicly available datasets for evaluation, including the KuaiRand-1K recommendation dataset and the LLMRouterBench benchmark. The KuaiRand dataset is released under the CC-BY-SA-4.0 license and is available from the official repository:
\url{https://github.com/chongminggao/KuaiRand}.
The LLMRouterBench benchmark and associated resources are available from the official repository:
\url{https://github.com/ynulihao/LLMRouterBench}.
LLMRouterBench aggregates multiple underlying datasets and models; therefore, the licenses of individual components follow their respective original terms of use. We cite the original dataset and benchmark sources and describe the preprocessing procedures used to construct our experimental instances.
\clearpage
\section{Algorithms}
\label{appendix:algorithms}

This section provides the complete pseudocode for the two algorithms developed in the paper. Algorithm~\ref{alg:rare-switching-oful} presents \emph{Feasible Rare-Switching OFUL} for the cyclic setting. Algorithm~\ref{alg:noncyclic-rare-switching} extends the rare-update principle to general sliding-window constraints.

\subsection{Rare-Switching OFUL}
\label{alg:rare_switching_oful}

\begin{algorithm}[H]
\caption{Feasible Rare-Switching OFUL}
\label{alg:rare-switching-oful}
\begin{algorithmic}[1]

\Require Horizon \(T\), regularization \(\lambda\), confidence level \(\delta\),
stationary feasible set \(\mathcal Z\), optimistic-action oracle
\(\mathsf{Opt}\), transition oracle \(\Gamma\)
\Ensure Feasible action sequence \(x_1,\ldots,x_T\)

\State Initialize \(V\gets\lambda I_d,\quad b\gets0_d,\quad t\gets1,\quad k\gets1\)

\While{\(t\leq T\)}
    \State Set \(W_k\gets V,\qquad \widehat\theta_k\gets W_k^{-1}b\)

    \State \(\beta_k(\delta)
    \gets
    R\sqrt{
    2\log\!\left(
    \frac{\det(W_k)^{1/2}}
    {\det(\lambda I_d)^{1/2}\delta}
    \right)}
    +
    \sqrt{\lambda}S\)

    \State Select
    \(z_k
    \gets
    \mathsf{Opt}
    \left(
    \widehat\theta_k,
    W_k,
    \beta_k(\delta)
    \right)\)

    \If{\(k>1\) and \(z_k\neq z_{k-1}\)}
        \State Obtain a feasible transition
        \((a_1,\ldots,a_m)
        \gets
        \Gamma(z_{k-1},z_k)\),
        where \(m\leq\tau\)

        \For{\(j\gets1\) to \(\min\{m,T-t+1\}\)}
            \State Set \(x_t\gets a_j\), execute \(x_t\), and observe \(Y_t\)
            \State \(t\gets t+1\)
            \State \textit{Exclude transition observations from \(V\) and \(b\)}
        \EndFor
    \EndIf

    \If{\(t>T\)}
        \State \textbf{break}
    \EndIf

    \State \(n_k\gets0\)

    \While{\(t\leq T\) and
    \(\left(
    \det(V)\leq2\det(W_k)
    \text{ or }
    [k=1\text{ and }n_k<w-1]
    \right)\)}
        \State Set \(x_t\gets z_k\), execute \(x_t\), and observe \(Y_t\)

        \State Update
        \(V\gets V+x_tx_t^\top,\qquad
        b\gets b+x_tY_t\)

        \State \(t\gets t+1\)
        \State \(n_k\gets n_k+1\)
    \EndWhile

    \State \(k\gets k+1\)
\EndWhile

\end{algorithmic}
\end{algorithm}

\subsection{Rare-Update Optimistic Planning}
\label{alg:rare_planning}

At each planning time \(t_k\), the learner freezes the current confidence set, defines the optimistic reward \(U_k\), and computes an exact
Bellman-optimal remaining-horizon plan from the current history state \(s_{t_k}\). The resulting plan is cached and followed until the
determinant-doubling condition triggers the next planning episode. For \(w=1\), we use the unique empty history \(s_t=()\).

\begin{algorithm}[t!]
\caption{Rare-Update Optimistic Planning}
\label{alg:noncyclic-rare-switching}
\begin{algorithmic}[1]

\Require \(T,w,\lambda,\delta,\mathsf{Init},\mathcal P\)
\Ensure Feasible sequence \(x_1,\ldots,x_T\)

\State \(V\gets\lambda I_d,\quad b\gets0_d\)

\If{\(w>1\)}
    \State \((x_1,\ldots,x_{w-1})\gets\mathsf{Init}(T)\),
    \(s_w\in\mathcal S_{T-w+1}\)
    \For{\(q=1,\ldots,w-1\)}
        \State Execute \(x_q\), observe \(Y_q\)
        \State \(V\gets V+x_qx_q^\top,\quad b\gets b+x_qY_q\)
    \EndFor
    \State \(t\gets w\)
\Else
    \State \(s_t\equiv()\), \(t\gets1\)
\EndIf

\State \(k\gets1\)

\While{\(t\le T\)}
    \State \(t_k\gets t,\quad W_k\gets V,\quad
    \widehat\theta_k\gets W_k^{-1}b\)

    \State
    \[
    \beta_k\gets
    R\sqrt{2\log\!\left(
    \frac{\det(W_k)^{1/2}}
    {\det(\lambda I_d)^{1/2}\delta}
    \right)}
    +\sqrt{\lambda}S
    \]

    \State Define
    \[
    \overline U_k(x)\gets
    \operatorname{clip}_{[-LS,LS]}
    \left(
    \langle\widehat\theta_k,x\rangle+
    \beta_k\|x\|_{W_k^{-1}}
    \right)
    \]

    \State \(n_k\gets T-t_k+1\)

    \State Cache
    \(
    (\widetilde x^{(k)}_1,\ldots,\widetilde x^{(k)}_{n_k})
    \gets
    \mathcal P(s_{t_k},n_k;\overline U_k)
    \)

    \While{\(t\le T\) and \(\det(V)\le2\det(W_k)\)}
        \State \(x_t\gets\widetilde x^{(k)}_{t-t_k+1}\)
        \State Execute \(x_t\), observe \(Y_t\)
        \State \(V\gets V+x_tx_t^\top,\quad b\gets b+x_tY_t\)
        \If{\(w>1\)}
            \State \(s_{t+1}\gets F(s_t,x_t)\)
        \EndIf
        \State \(t\gets t+1\)
    \EndWhile

    \State \(k\gets k+1\)
\EndWhile

\end{algorithmic}
\end{algorithm}

\end{document}